\documentclass{article}

\usepackage[main, final]{neurips_2026}

\usepackage[T1]{fontenc}    %
\usepackage{url}            %
\usepackage{booktabs}       %
\usepackage{amsfonts}       %
\usepackage{nicefrac}       %
\usepackage{microtype}      %
\usepackage[table]{xcolor}  %

\usepackage{microtype}
\usepackage{graphicx}
\usepackage{subcaption}
\usepackage{booktabs} %
\usepackage{multirow}
\usepackage{makecell}
\usepackage{subcaption}
\definecolor{Gray}{gray}{0.93}

\usepackage{algorithm}
\usepackage{algorithmic} %

\definecolor{mydarkred}{rgb}{0.6,0,0}
\definecolor{mydarkgreen}{rgb}{0,0.6,0}
\PassOptionsToPackage{options}{natbib}
\setcitestyle{authoryear,round,citesep={;},aysep={,},yysep={;}}
\usepackage[colorlinks,linkcolor=mydarkred,citecolor=mydarkgreen]{hyperref}

\hypersetup{
  bookmarksnumbered=true,
  bookmarksopen=true,
  bookmarksopenlevel=2,
  pdfstartview=FitH
}

\usepackage{amsmath}
\usepackage{amssymb}
\usepackage{mathtools}
\usepackage{amsthm}

\DeclareMathAlphabet{\mathbbold}{U}{bbold}{m}{n}

\renewcommand{\hat}{\widehat}

\usepackage[capitalize,noabbrev]{cleveref}

\usepackage{tikz}
\usetikzlibrary{positioning,shapes,arrows.meta}

\theoremstyle{plain}
\newtheorem{theorem}{Theorem}[section]
\newtheorem{proposition}[theorem]{Proposition}
\newtheorem{lemma}[theorem]{Lemma}

\theoremstyle{definition}
\newtheorem{definition}[theorem]{Definition}
\newtheorem{assumption}[theorem]{Assumption}
\theoremstyle{remark}
\newtheorem{remark}{Remark}

\makeatletter
\def\th@remark{%
  \thm@headfont{\bfseries\itshape}%
  \normalfont %
  \thm@preskip\topsep \divide\thm@preskip\tw@
  \thm@postskip\thm@preskip
}
\makeatother

\title{PAC Reasoning: Controlling the Performance Loss for Efficient Reasoning}

\author{%
  \textbf{Hao Zeng}$^{1}$\thanks{Equal contribution.} \quad \textbf{Jianguo Huang}$^{2}$\footnotemark[1] \quad \textbf{Bingyi Jing}$^{3,5}$ \\
  \textbf{Hongxin Wei}$^{4}$\thanks{Corresponding author: \texttt{weihx@sustech.edu.cn}.} \quad \textbf{Bo An}$^{2}$ \\[0.5em]
  \normalfont\small $^1$School of Mathematical Sciences, University of Electronic Science and Technology of China, China \\
  \normalfont\small $^2$College of Computing and Data Science, Nanyang Technological University, Singapore \\
  \normalfont\small $^3$School of Artificial Intelligence, The Chinese University of Hong Kong, Shenzhen, China \\
  \normalfont\small $^4$Department of Statistics and Data Science, Southern University of Science and Technology, China \\
  \normalfont\small $^5$Shenzhen Loop Area Institute, China \\
}

\begin{document}

\maketitle

\begin{abstract}
Large reasoning models (LRMs) have achieved remarkable progress in complex problem-solving tasks.
Despite this success, LRMs typically suffer from high computational costs during deployment, highlighting a need for efficient inference.
A practical direction is to switch the LRM between thinking and non-thinking modes dynamically.
However, such approaches often introduce additional reasoning errors and lack statistical guarantees for the performance loss, which are critical for high-stakes applications.
In this work, we propose \textit{Probably Approximately Correct} (PAC) \textit{Reasoning} that controls the performance loss under the user-specified tolerance.
Specifically, we construct an upper confidence bound on the performance loss and determine a threshold for switching to the non-thinking model.
Theoretically, using the threshold to switch between the thinking and non-thinking modes ensures bounded performance loss in a distribution-free manner.
Our experiments on reasoning benchmarks show that the proposed method can save computational budgets and control the user-specified performance loss.
\end{abstract}

\section{Introduction}
Large reasoning models (LRMs) show strong performance in tackling complex problems~\citep {deepseek-ai2025deepseekr1,yang2025qwen3}.
However, the strong performance largely depends on long reasoning chains, which substantially increase the computational cost during inference.
This phenomenon, often referred to as \textit{overthinking}~\citep{yue2025dont}, is evident in mathematical and logic-intensive tasks.
In applications requiring real-time interaction or large-scale processing, such as text generation~\citep{zhang2022opt} and chatbot~\citep{roller2021recipes}, inference efficiency directly determines usability and user experience.
Thus, it is important to improve the efficiency of LRMs.

To address this, some existing works proposed to switch the LRM into a non-thinking mode to avoid overthinking~\citep{chung2025thinker, fang2025thinkless, li2025dynamicmind, liang2025thinkswitcher, ma2025reasoning, paliotta2025thinking, pan2025survey, xiao2025fastslow, yong2025think}.
While effective in reducing computational cost, using a non-thinking model often degrades solution quality or introduces additional errors.
For instance, in theorem-proving tasks, switching techniques may lead to invalid logical inference steps, and in mathematical reasoning, it can result in calculation mistakes or overlooked solution paths.
This motivates studying performance-loss control together with efficient reasoning and the limitation raises a fundamental issue:

\begin{center}%
    \textit{How to improve the efficiency of LRMs, controlling the performance loss provably?}
\end{center}%

In this work, we formalize this challenge by the concept of a PAC efficient model, where an LRM provides probabilistic guarantees that its performance loss relative to a reference thinking model remains within a user-specified tolerance in Definition~\ref{def:pac}.
We address the problem of overthinking, where reasoning models spend unnecessary computation on easy queries.
To achieve this goal, we propose \textbf{PAC Reasoning}, which constructs a composite model $\hat{f}$ by adaptively routing each input between a high-cost thinking model $f$ and a cheaper non-thinking model $\tilde{f}$.
Specifically, PAC Reasoning determines a threshold via a calibration procedure (Algorithm~\ref{alg:compute_bound}). During deployment (Algorithm~\ref{alg:pac_reasoning}), the composite model uses $\tilde{f}$ for inputs whose uncertainty score falls below the calibrated threshold, and defers to $f$ otherwise.
By calibrating the switching rule with respect to an explicit loss tolerance and confidence level, PAC Reasoning provides rigorous statistical guarantees on the resulting performance loss.
Furthermore, we provide a comprehensive discussion of related work in Appendix~\ref{sec:related_work}.

Theoretically, we show that PAC Reasoning provides distribution-free control of performance loss with high-probability guarantees.
We formalize this through a composite model $\hat{f}$, for which the performance loss is non-decreasing with respect to the uncertainty threshold.
We calibrate the threshold by testing candidate thresholds from a pre-specified grid in increasing order, stopping at the first one whose upper confidence bound exceeds a user-specified tolerance.
Under some regular assumptions, we prove that PAC Reasoning controls the loss below the user-specified tolerance with high probability, thereby satisfying PAC efficiency.

We then present experimental results in \Cref{sec:experiments} that evaluate PAC Reasoning across diverse reasoning benchmarks, including MATH-500, ZebraLogic, and Arena-Hard.
The results demonstrate that our approach effectively controls the performance loss and significantly reduces inference cost.
For example, on Arena-Hard with tolerance $\epsilon = 0.08$ for the logits uncertainty score, our method controls the average empirical performance loss at $0.06$ (below the tolerance), and achieves token savings exceeding $40\%$.
We also find that the logits-based uncertainty score provides more stable performance loss control compared to the verbalized-based score.~\footnote{The reproducibility code is available at \url{https://github.com/Jianguo99/PAC-Reasoning}.}

Our contributions are as follows:
\begin{itemize}
    \item
    We are the first to define \textbf{PAC efficient} models through high-probability control of their performance loss relative to a thinking model.
    \item 
    We propose a method, \textbf{PAC Reasoning}, that combines a thinking-mode model with its non-thinking counterpart via an uncertainty-based mechanism to skip unnecessary thinking.
    The method provides \emph{model-free} and \emph{distribution-free} performance guarantees based on statistical theories.
    \item 
    We provide experiments on mathematical reasoning, logical deduction, and text generation, validating that PAC Reasoning achieves significant cost savings in token and expert calls while satisfying the statistical validity of the PAC efficient guarantee.
\end{itemize}

\textbf{Notations.}
Let $f$ be an LRM with a thinking-mode, which is computationally expensive but delivers high performance on its answers.
Given an input prompt $x$, $f$ produces an output $y = f(x)$, which we regard as the ``expert answer''.
The second is the non-thinking LRM $\tilde{f}$, which is computationally cheaper but less accurate, and $\tilde{y} = \tilde{f}(x)$.
And for any input \(x\), we use \(y^{gold}\) as its ``gold reference''.
Let $\mathcal{I}_{cal} = \{1, \ldots, n\}$ and $\mathcal{I}_{test} = \{n+1, \ldots, n+N\}$ denote the indices of the calibration and test sets, respectively.
We define the calibration dataset and the test dataset as:
\[
\mathcal{D}_{cal} = \{(x_i, y_{i})\}_{i \in \mathcal{I}_{cal}},
\quad
\mathcal{D}_{test} = \{(x_i, y_{i})\}_{i \in \mathcal{I}_{test}}.
\]
It is worth noting that the \(y_i\) is not a ground-truth label, but the ``expert answer'' provided by the LRM $f$.
Finally, let $y_{i} = (y_{i,1}, \dots, y_{i,l_{y_i}})$ denote an answer consisting of $l_{y_i}$ tokens.

\section{Probably approximately correct reasoning}\label{sec:pac_reasoning}

\subsection{PAC efficient model}
We aim to build an efficient LRM, denoted by $\hat{f}$, that provides probably approximately correct guarantees for its performance loss while improving efficiency.
Specifically, given an error tolerance $\epsilon$ and a confidence level $1-\alpha$, $\hat{f}$ ensures its performance loss \emph{relative to the only thinking-mode} LRM $f$ does not exceed $\epsilon$ with probability at least $1-\alpha$.
We formulate the PAC guaranteed $\hat{f}$ as follows:

\begin{definition}[($\epsilon, \alpha$)-PAC efficient]\label{def:pac}
An LRM \(\hat{f}\) is called an (\(\epsilon, \alpha\))-probably approximately correct (PAC) efficient model,
if for error tolerance $\epsilon>0$ and confidence level $\alpha \in (0,1)$, it satisfies
\[
\mathbb{P}\left( R(\hat{f}) \le \epsilon \right) \ge 1 - \alpha,
\]
where \(R(\hat{f}) = \mathbb{E}_{x\sim P}[\ell(f(x),\hat{f}(x))]\) is the risk function, with the expert answer as the first argument of $\ell$.
Here $\ell(\cdot,\cdot)$ is a loss function\footnote{For simplicity, we often omit the phrase \emph{with respect to $\ell$}, since most tasks have their conventional loss functions.}, and $x$ is an input prompt drawn from the test distribution $P$.
\end{definition}

\begin{remark}
    The loss function could be a 0-1 loss for verifiable tasks or a semantic loss for generative tasks.
    We sometimes term \((\epsilon, \alpha)\)-PAC efficient simply as PAC efficient.
    Since $\epsilon$ controls the performance loss relative to the thinking model $f$, a meaningful range of $\epsilon$ is $(0, R(\tilde f))$, where $R(\tilde f)$ is the risk of always using the non-thinking model.
\end{remark}

\subsection{PAC Reasoning}
Constructing such a controllable LRM $\hat{f}$ may be straightforward intuitively.
Given an LRM with thinking mode $f$ and a fast LRM without thinking $\tilde{f}$, we create a middle model that selectively uses either the LRM with thinking or not based on certain conditions.
This condition acts like a ``sliding rheostat'' that allows us to tune the performance trade-off by adjusting the ``position'' of the intermediate.
We can obtain a model $\hat{f}$ that achieves the desired error tolerance heuristically.
However, this approach lacks statistical guarantees on the underlying distribution of performance loss.
We calibrate a threshold to control performance loss while reducing computation.

Motivated by this, we present \textbf{PAC Reasoning}, which constructs a composite LRM $\hat{f}$ that improves the efficiency of the thinking model $f$ with a PAC guarantee for the performance loss.
The core idea is to use \textbf{uncertainty scores} to build an \textbf{upper confidence bound}(UCB) for the performance loss.
We could use the upper confidence bound to measure the uncertainty of performance loss for each value of the uncertainty score.
Then we \textbf{calibrate} an uncertainty threshold to switch between the thinking and non-thinking models.
We use the non-thinking LRM $\tilde{f}$ on most inputs and strategically invoke the expensive LRM with thinking $f$ only for inputs whose generation by $\tilde{f}$ has high uncertainty.
Mathematically, we could represent the composite model $\hat{f}$ for a threshold $u$ as follows:
\begin{equation}
    \label{eq:composite_model}
    \hat f(x) \equiv \hat{f}_u(x) = \begin{cases}
        \tilde{f}(x), & U(x) < {u}, \\
        f(x), & U(x) \ge {u}.
    \end{cases}
\end{equation}
Next, we provide the details of the PAC Reasoning.

\paragraph{Uncertainty scores and empirical performance loss}
We assume that for each input prompt $x_i$, the non-thinking LRM $\tilde{f}$ produces an output $\tilde{y}_i$, and that there exists a score $U_i \in [0,1]$ to quantify its uncertainty.
This score should ideally correlate with the likelihood of disagreement with the reference model $f$.
The core idea is to use these uncertainty scores to selectively use the expensive model with thinking $f$.
We aim to find a threshold, $\hat{u}$, and accept the non-thinking LRM's output $\tilde{y}_i$ for the instances such that $U_i < \hat{u}$, while querying the model with thinking $f$ for the cases where $U_i \ge \hat{u}$.
To formalize, we define the empirical performance loss function as:
\begin{equation}
    L(u) = \frac{1}{N} \sum_{i=n+1}^{n+N} \ell(y_i, \tilde{y}_i) \mathbf{1}\{U_i < u\}.
\end{equation}
This function measures the average performance loss over the test set, with nonzero contributions only from points whose uncertainty scores are strictly below $u$.
If we could compute $L(u)$ for all $u$, we would choose the largest threshold $u^*$ such that $L(u^*) \le \epsilon$.
However, computing $L(u)$ requires access to all expert answers $y_i = f(x_i)$ in the test set, which is computationally expensive.
We try an alternative way to build a bound \(\hat{L}_u(\alpha)\) for $\mathbb{E} L(u)$ satisfying the following inequality pointwise w.r.t.~\(\alpha\):
\begin{equation}\label{eq:cum_error}
    \mathbb{P}(\hat{L}_u(\alpha) \ge \mathbb{E} L(u)) \ge 1-\alpha.
\end{equation}
The bound is called the upper confidence bound in the literature.
By construction, $L(u)$ is non-decreasing in $u$ since $\ell$ is non-negative and $\mathbf{1}\{U_i < u\}$ only expands as $u$ increases.
This monotonicity lets us obtain the PAC guarantee in Definition~\ref{def:pac} and apply fixed-sequence single-start without additional corrections; we summarize it in Theorem~\ref{thm:pac_guarantee}.

\begin{algorithm}[ht!]
\caption{CLT-based UCB for $\hat{L}_u(\alpha)$}
\label{alg:compute_bound}
\begin{algorithmic}[1]
\REQUIRE Calibration set $\mathcal{S}_{\mathrm{cal}} = \{(x_i, y_i)\}_{i=1}^n$, models $f$ and $\tilde{f}$, uncertainty scores $\{U_i\}_{i=1}^n$, sampling weights $\{\pi_i\}_{i=1}^n$, sample size $m$, threshold $u$, confidence level $\alpha$
\ENSURE Confidence upper bound $\hat{L}_u(\alpha)$
\STATE Initialize an empty list of samples $\mathcal{Z} \leftarrow []$.
\STATE $\tilde{y}_i \leftarrow \tilde{f}(x_i)$ for all $i = 1, \dots, n$.
\FOR{$j = 1, \dots, m$}
    \STATE Sample an index $i_j \sim \text{Unif}(\{1, \dots, n\})$.
    \STATE Sample a fresh Bernoulli random variable $\xi_j \sim \text{Bern}(\pi_{i_j})$, independently across draws.
    \IF{$\xi_j = 1$}
        \STATE Query the true label $y_{i_j}$ and compute the importance-weighted loss $Z_j \leftarrow \ell(y_{i_j}, \tilde{y}_{i_j}) / \pi_{i_j}$.
    \ELSE
        \STATE $Z_j \leftarrow 0$.
    \ENDIF
    \STATE Append $Z_j$ to $\mathcal{Z}$.
\ENDFOR
\STATE For a threshold $u$, set $Z_j(u) \leftarrow Z_j \cdot \mathbf{1}\{U_{i_j} < u\}$ for $j = 1, \dots, m$.
\STATE $\hat{\mu}_Z(u) \leftarrow \frac{1}{m} \sum_{j=1}^m Z_j(u)$
\STATE $\hat{\sigma}_Z(u) \leftarrow \sqrt{\frac{1}{m-1} \sum_{j=1}^m (Z_j(u) - \hat{\mu}_Z(u))^2}$
\STATE $z_{1-\alpha} \leftarrow (1-\alpha)$-quantile of the standard normal distribution.
\STATE {\bfseries Return} $\hat{L}_u(\alpha) \leftarrow \hat{\mu}_Z(u) + z_{1-\alpha} \frac{\hat{\sigma}_Z(u)}{\sqrt{m}}$.
\end{algorithmic}
\end{algorithm}

\paragraph{Constructing the upper confidence bound (UCB)}
Next we discuss how to construct a valid UCB, via importance sampling.
Given a sampling size $m$, we first collect $m$ indices $\{i_1, \dots, i_m\}$ by sampling uniformly with replacement from $\{1, \dots, n\}$.
Then, for each selected index $i_j$, we draw a fresh Bernoulli variable $\xi_j \sim \text{Bern}(\pi_{i_j})$, independently across draws, to decide whether to query its expert answer $y_{i_j}$.
The weights satisfy $0<\pi_i\le1$ and are fixed conditional on the calibration set.
Given this set, the procedure yields $m$ i.i.d.\ random variables:
\begin{equation}
    Z_j(u) = \ell(y_{i_j}, \tilde{y}_{i_j}) \frac{\xi_j}{\pi_{i_j}} \mathbf{1}\{U_{i_j} < u\}.
\end{equation}
Its conditional mean is the calibration-pool risk $\overline L_n(u)=n^{-1}\sum_{i=1}^n\ell(y_i,\tilde y_i)\mathbf1\{U_i<u\}$, since $\mathbb{E}[\xi_j/\pi_{i_j}\mid i_j,\mathcal D_{cal}]=1$.
Algorithm~\ref{alg:compute_bound} uses a central limit theorem (CLT) approximation for this mean.
The population risk is $R(u)=\mathbb E_{\mathcal D_{cal}}[\overline L_n(u)]=\mathbb E[L(u)]$ under i.i.d.\ sampling from the deployment law.

\begin{remark}\label{remark:importance_sampling}
For a fixed threshold $u$, importance sampling gives an unbiased estimate of $\overline L_n(u)$ conditional on the calibration pool and of $R(u)$ after averaging over the original calibration sample.
The resampling draws are independent only conditional on the shared pool.
Here $n$ is the number of original calibration observations and $m$ is the number of auditing draws; repeated indices do not add new observations from the deployment law.
The CLT bound is asymptotic for the pool mean as $m$ grows, while its use for the population risk in Assumption~\ref{assump:validity} must also account for the randomness of the $n$ original observations.
Appendix~\ref{sec:clt_ucb_validity} gives the conditional CLT and separates these two sources of randomness.
Alternatively, if the importance-weighted losses are bounded, one could use concentration inequalities like Hoeffding's or Bernstein's inequality to construct a valid confidence bound~\citep{bentkus2004hoeffdings, hao2019bootstrapping, hoeffding1994probability, learned-miller2020new, ramdas2022admissible, waudby-smith2021confidence, waudby-smith2024estimating}, which may provide better guarantees for smaller sample sizes.
We give an example based on Hoeffding's inequality~\citep{hoeffding1994probability} in Algorithm~\ref{alg:compute_bound_finite_sample} in Appendix~\ref{app:bounded_loss_alg}.
\end{remark}

\paragraph{Calibration}
Fix a deterministic finite grid $\Lambda=\{u_1<\cdots<u_K\}\subset(0,1]$, with $K\ge1$, before observing the calibration data $\mathcal{D}_{cal}$.
For example, $u_k=k/K$ gives a grid with a size $K$ chosen in advance.
Set $u_0=0$, which gives the expert-only rule in Equation~\eqref{eq:composite_model} and has zero expert-relative loss since $U(x)\ge0$ and $\ell(y,y)=0$.
We test $\hat L_{u_k}(\alpha)\le\epsilon$ in increasing order of $k$, stop at the first failure, and select the preceding threshold:
\begin{equation}
    \label{eq:u^hat}
    \begin{aligned}
        \tau &= \min\!\left(\{k\in\{1,\ldots,K\}:\hat L_{u_k}(\alpha)>\epsilon\}\cup\{K+1\}\right), 
        \hat u &= u_{\tau-1}.
    \end{aligned}
\end{equation}
If the first candidate fails, $\hat u=u_0=0$; if all $K$ candidates pass, $\tau=K+1$ and $\hat u=u_K$.
The fallback $u_0$ is safe by construction and does not require a UCB test.
A UCB equal to $\epsilon$ passes, but no later candidate is tested after a failure, even if its UCB would pass.
At deployment, the fast model is used only when $U(x)<\hat u$; equality is assigned to the expert, including $U(x)=1$ when $\hat u=1$.
Algorithm~\ref{alg:pac_reasoning} implements this fixed-sequence early-stopping rule.

\begin{algorithm}[ht]
\caption{PAC Reasoning}\label{alg:pac_reasoning}
\textbf{Input:} Calibration set $\{(x_i, y_i)\}_{i=1}^n$, test prompts $x_i, i \in \mathcal{I}_{test}$, model without thinking $\tilde{f}$, model with thinking $f$, loss function $\ell$, error tolerance $\epsilon$, confidence level $1-\alpha$, deterministic grid $\Lambda=\{0<u_1<\cdots<u_K\le1\}$ fixed before calibration\\
\textbf{Output:} Composite model $\hat{f}$

\begin{algorithmic}[1]
\STATE Compute confidence bound function $\hat{L}_{u}(\alpha)$ via Algorithm~\ref{alg:compute_bound} on calibration data $\{(x_i, y_i)\}_{i=1}^n$.
\STATE Initialize $\hat u\leftarrow u_0=0$ (expert-only fallback).
\FOR{$k=1,\ldots,K$}
    \IF{$\hat L_{u_k}(\alpha)>\epsilon$}
        \STATE \textbf{break}
    \ENDIF
    \STATE $\hat u\leftarrow u_k$.
\ENDFOR
\STATE Define $\hat f=\hat f_{\hat u}$ by Equation~\eqref{eq:composite_model}, using $\tilde f$ for $U(x)<\hat u$ and $f$ for $U(x)\ge\hat u$.
\STATE \textbf{Return} $\hat{f}$.
\end{algorithmic}
\end{algorithm}

\section{Theoretical analysis}\label{sec:theoretical_analysis}
We show how fixed-sequence selection turns pointwise risk bounds into a guarantee for the selected policy.
Fix the models $f,\tilde f$, the score function $U:\mathcal X\to[0,1]$, and a nonnegative loss $\ell$ with $\ell(y,y)=0$ before calibration.
Let $x\sim P$ and $y=f(x)$, and let $\mathcal D_{cal}$ consist of $n$ i.i.d.\ draws from this law.
Assume $\mathbb E_{x\sim P}[\ell(y,\tilde f(x))]<\infty$.
For the policy $\hat f_u$ in Equation~\eqref{eq:composite_model}, its population risk and independent test-sample risk are
\[
R(u)=\mathbb E_{x\sim P}[\ell(y,\hat f_u(x))],
\qquad
\widehat R(u)=\frac1N\sum_{i\in\mathcal I_{test}}\ell(y_i,\hat f_u(x_i)).
\]
In Assumption~\ref{assump:validity} and Theorem~\ref{thm:pac_guarantee}, probability is over $\mathcal D_{cal}$ and any additional randomness in the UCB procedure, not over the independent test sample.

Fix $\epsilon>0$, $\alpha\in(0,1)$, and the deterministic finite grid $\Lambda=\{u_1<\cdots<u_K\}\subset(0,1]$ before observing $\mathcal D_{cal}$.
Set $u_0=0$, so $R(u_0)=0$ because this policy always uses the expert.
For $0\le u\le v\le1$,
\[
R(v)-R(u)
=
\mathbb E_{x\sim P}\!\left[\ell(y,\tilde f(x))\mathbf1\{u\le U(x)<v\}\right]\ge0.
\]
Thus, the true risk is non-decreasing in the threshold; no monotonicity of the UCB is assumed.

\begin{assumption}[Pointwise UCB validity]\label{assump:validity}
For every fixed $u\in\Lambda$, the UCB computed from $\mathcal D_{cal}$ satisfies
\[
\mathbb P\big(R(u)\le\widehat L_u(\alpha)\big)\ge1-\alpha.
\]
This is a pointwise condition, not an event on which all grid points are covered at once.
\end{assumption}
This high-level condition concerns the population risk $R(u)$.
A bound for a fixed empirical calibration pool alone does not establish it.

\begin{theorem}[PAC guarantee]\label{thm:pac_guarantee}
Under the setup above, let $\hat u$ be selected on the fixed finite grid $\Lambda$ by Equation~\eqref{eq:u^hat} and Algorithm~\ref{alg:pac_reasoning}.
Assume that $R(u)$ is non-decreasing, $R(u_0)=0$, and Assumption~\ref{assump:validity} holds at every fixed grid point.
If selection stops at the first failed test and returns the preceding threshold, with $u_0$ as the fallback, then
\[
\mathbb P\big(R(\hat f_{\hat u})\le\epsilon\big)
=
\mathbb P\big(R(\hat u)\le\epsilon\big)\ge1-\alpha.
\]
\end{theorem}
\begin{proof}[Proof sketch]
If all grid risks are at most $\epsilon$, every possible output, including $u_0$, is safe.
Otherwise, let $k^*=\min\{k:R(u_k)>\epsilon\}$ be the first unsafe index.
It is fixed by $P$ and the pre-specified grid, not by the calibration data.
Selecting an unsafe threshold requires passing the test at $u_{k^*}$ before stopping, so
\[
\{R(\hat u)>\epsilon\}
\subseteq
\{\widehat L_{u_{k^*}}(\alpha)\le\epsilon<R(u_{k^*})\}.
\]
Pointwise validity at this single fixed candidate bounds the probability by $\alpha$.
Appendix~\ref{sec:proof_pac_guarantee} gives the full argument.
\end{proof}
The selection guarantee has the same finite-sample or asymptotic validity as the pointwise population-risk UCBs.
The asymptotic result uses a fixed finite grid; see Appendix~\ref{sec:proof_pac_guarantee}.
For loss bounded in $[a,b]$ with $a<b$, we also give an empirical test-risk bound.
\begin{theorem}[Empirical risk PAC guarantee]\label{thm:empirical_test_pac_appendix}
Assume the setup and conditions of Theorem~\ref{thm:pac_guarantee} hold, and the i.i.d.\ test batch $\mathcal D_{test}$ is independent of the calibration data and the UCB randomization.
For $\epsilon>0$, $\alpha\in(0,1)$, $\ell\in[a,b]$ with $a<b$, and any $t>0$,
\[
\mathbb{P}\!\big(\widehat R(\hat u)\le \epsilon+t\big)\;\ge\;1-\alpha-\exp\!\Big(-\tfrac{2 N t^2}{(b-a)^2}\Big).
\]
Here, the probability also includes the independent test sample.
\end{theorem}

\begin{proof}[Proof sketch]
The result follows from the decomposition of the error event:
$\{\widehat R(\hat u) > \epsilon + t\} \subseteq \{R(\hat u) > \epsilon\} \cup \{\widehat R(\hat u) - R(\hat u) > t\}$.
The first term is bounded by $\alpha$ via Theorem~\ref{thm:pac_guarantee}.
The second term is bounded by $\exp(-2Nt^2/(b-a)^2)$ using the conditional Hoeffding's inequality, as the test set is independent of $\hat u$.
A detailed proof is in Appendix~\ref{sec:proof_of_empirical_test_pac_appendix}.
\end{proof}

A special case is a bounded loss $\ell\in[0,1]$, e.g., $0$-$1$ loss for binary verifiable answers.
Then $b-a=1$ and the bound simplifies to $\mathbb{P}\!\big(\widehat R(\hat u)\le \epsilon+t\big)\ge 1-\alpha-e^{-2 N t^2}$.
The extra tail term accounts for test-sample error around the selected policy's population risk.
Appendix~\ref{sec:transductive_pac_reasoning} gives a separate guarantee for a fixed pool, and Appendix~\ref{sec:multipac} extends fixed-sequence selection to multi-level policies.

\section{Experiments}\label{sec:experiments}
In this section, we conduct comprehensive experiments to answer the following research questions:
\begin{enumerate}
    \item Does PAC Reasoning control the performance loss at the target tolerance with confidence?
    \item What computational savings does PAC Reasoning achieve while bounding the error rate?
\end{enumerate}

\subsection{Setup}
\label{sec:exp_setup}
\paragraph{Large language models.}
In this study, we evaluate the PAC Reasoning based on the Qwen3 series models~\citep{yang2025qwen3} and Llama-3.1-8B–based models.
Specifically, we employ the ``Qwen3-4B-Thinking-2507'' as the thinking model and ``Qwen3-4B-Instruct-2507'' as the non-thinking model.
We also use the ``DeepSeek-R1-Distill-Llama-8B'' as the thinking model and ``Llama-3.1-8B-Instruct'' as the non-thinking model, and show the results in Appendix~\ref{app:exp_other_llm}.

\textbf{Uncertainty score.}
Our theoretical guarantees do not depend on the choice of uncertainty score.
Here, we focus on two common uncertainty scores used in LLMs, the \textbf{logits-based score}, a white-box score from model logits, and the \textbf{verbalized score}, a black-box score from self-reported confidence.
For the logits-based score, we use token-level probabilities from the prediction logits~\citep{kwon2023efficient,zheng2024sglang,zhou2025theoretical}.
Specifically, we define the logits-based score of $y_i$ as its average token probability~\citep{hao2023reasoning, huang2025look}:
$$U_{logits}(y_i) = 1-\frac{1}{l_{y_i}} \sum_{j=1}^{l_{y_i}} \mathbb{P}(y_{i,j} | y_{i,1}, \dots, y_{i,j-1}, x_i),
$$
where $\mathbb{P}(y_{i,j} | y_{i,1}, \dots, y_{i,j-1},  x_i)$ is the conditional probability of token $y_{i,j}$.
For the verbalized uncertainty scores,  the non-thinking model explicitly states its self-reported confidence~\citep{xiong2023can,tian2023just,yang2024verbalized,zhou2025theoretical}.
In this study, we report the average verbalized confidence over 10 trials, and the corresponding prompts are listed in Appendix~\ref{app:experimental_details}.

\textbf{Datasets.} We evaluate PAC Reasoning on datasets spanning reasoning and open-ended generation tasks. Specifically, our evaluation covers a high-level mathematics benchmark, \textbf{MATH-500}~\citep{lightman2023lets}, a text-based logical reasoning task, \textbf{ZebraLogic}~\citep{lin2025zebralogic}, and an alignment-focused open-ended writing benchmark, \textbf{Arena-Hard}~\citep{li2025crowdsourced}.
For each dataset, we partition the original test set into a calibration subset and a held-out test subset randomly.
Table~\ref{table:splitting_settings} in Appendix~\ref{app:experimental_details} provides details of the splitting strategies.

\begingroup
\textbf{Baselines.} We compare PAC Reasoning with Naive threshold, empirical-mean Naive control, Router~\citep{ong2025routellm}, Chain of Draft (CoD)~\citep{xu2025chain}, and NoThinking~\citep{ma2025reasoning}, with details in Appendix~\ref{sec:baseline}.
They cover fixed-threshold routing, empirical calibration, learned routing, and prompting-based efficiency methods.
\endgroup

\textbf{Loss functions.}
PAC Reasoning controls the performance loss incurred by switching from the thinking model $f$ to the non-thinking model $\tilde{f}$.
Therefore, we define loss relative to the thinking model, not relative to the gold answer.
We use two loss functions:
(1) \textbf{Semantic cosine distance}.
Given reference output $y_i=f(x_i)$ and PAC output $\hat{y}_i=\hat{f}(x_i)$, the semantic loss is given by:
\begin{equation}
\label{eq:semantic_loss}
\ell(y_i, \hat{y}_i) = 1 - \frac{v_{y_i} \cdot v_{\hat{y}_i}}{\|v_{y_i}\| \|v_{\hat{y}_i}\|},
\end{equation}
where $v_{y_i}$ and $v_{\hat{y}_i}$ are the embeddings. We use ``Qwen3-Embedding-4B'' as the embedding model~\citep{yang2025qwen3}.
(2) \textbf{Binary 0--1 loss} for verifiable tasks.
It measures the accuracy drop when the expert answer is correct:
\begin{equation}
\label{eq:binary_loss}
\ell(y_i, \hat{y}_i) = \ell(y_i, \hat{y}_i \mid y_i^{gold}) = \mathbbold{1}\{ \hat{y}_i \neq y_i^{gold}\} \mathbbold{1}\{y_i=y_i^{gold}\},
\end{equation}
where $y_i^{gold}$ is the gold answer for $x_i$.
We discuss the choice of loss function in Appendix~\ref{sec:choice_loss_functions}.

\begin{figure}[!t]
    \centering
    \subcaptionbox{MATH-500}{
        \includegraphics[width=0.33\linewidth]{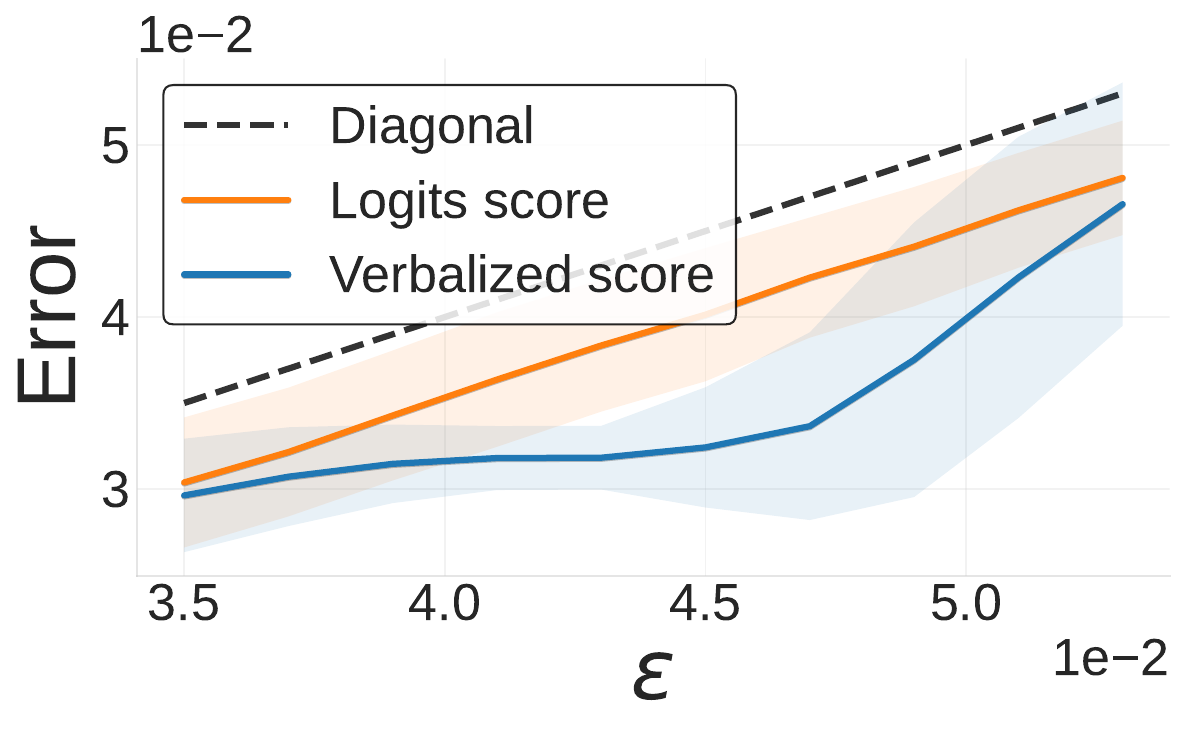}
         \includegraphics[width=0.33\linewidth]{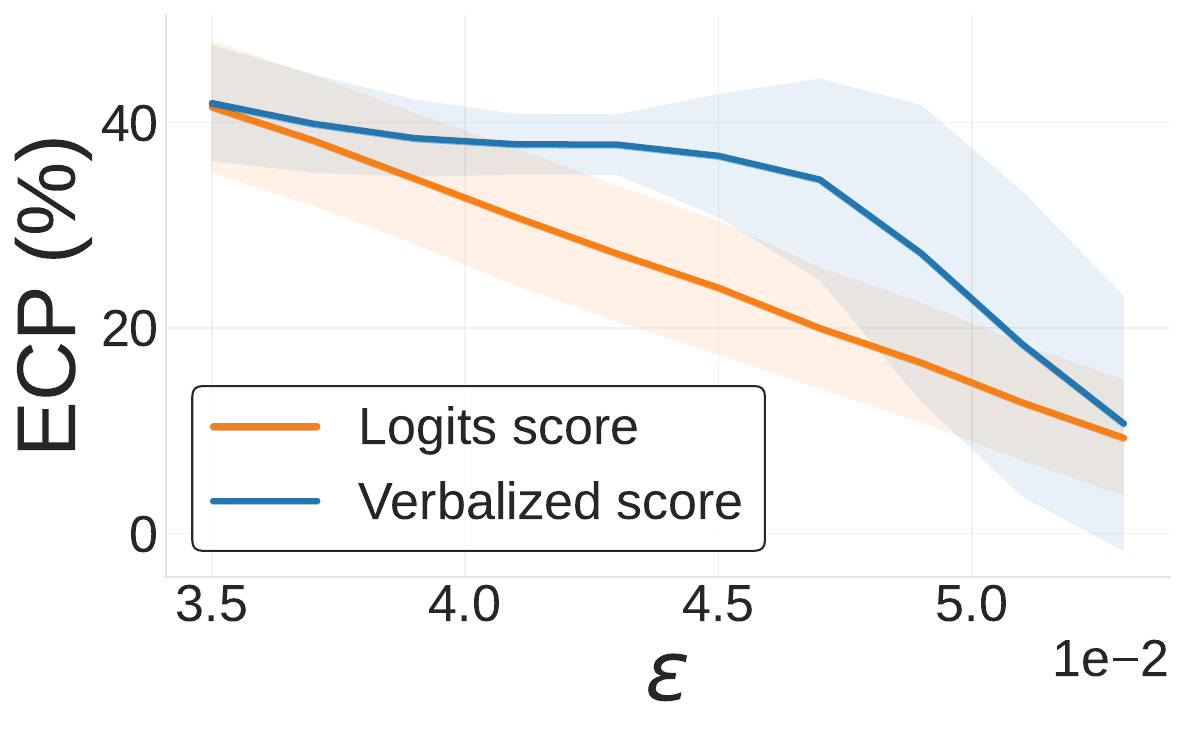}
        \includegraphics[width=0.33\linewidth]{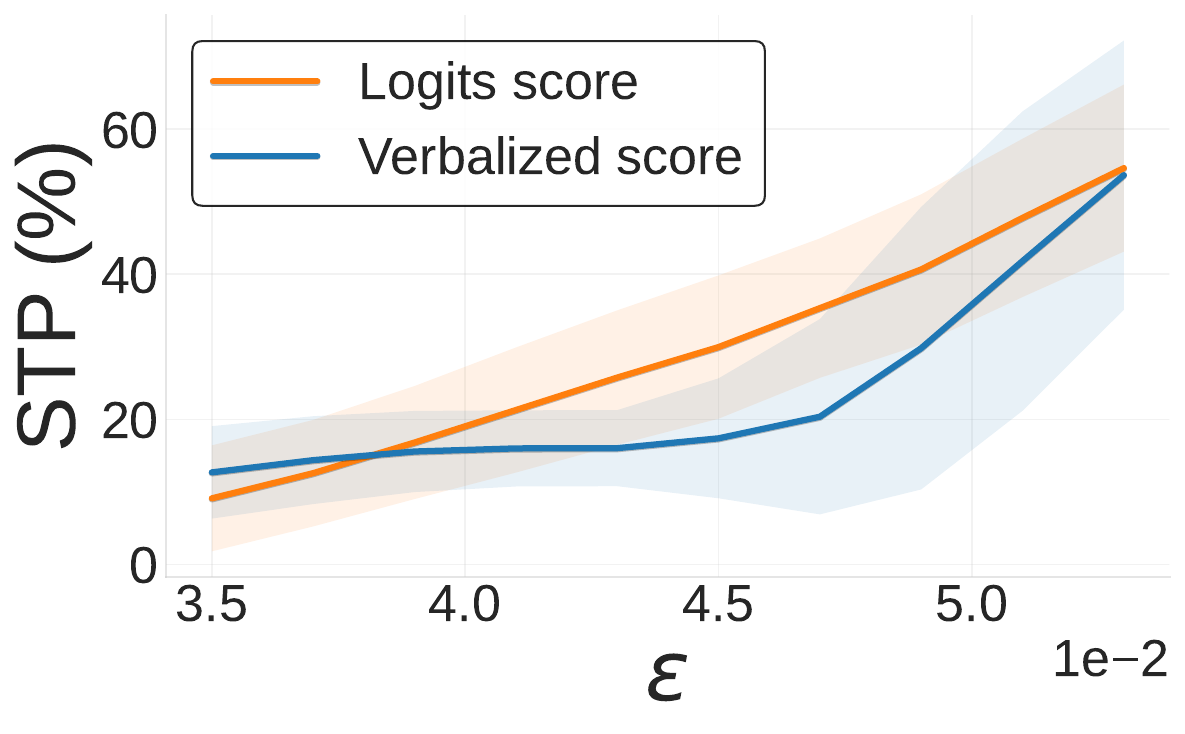}
    }
    \\
    \subcaptionbox{ZebraLogic}{
        \includegraphics[width=0.33\linewidth]{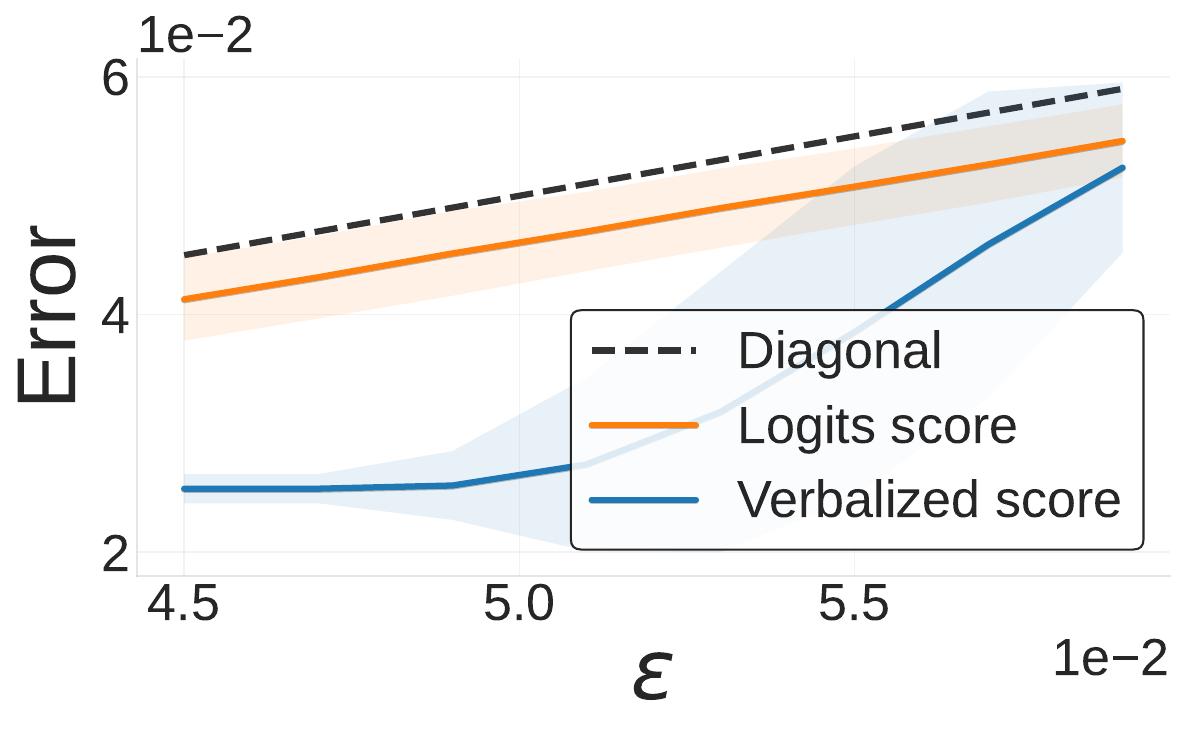}
        \includegraphics[width=0.33\linewidth]{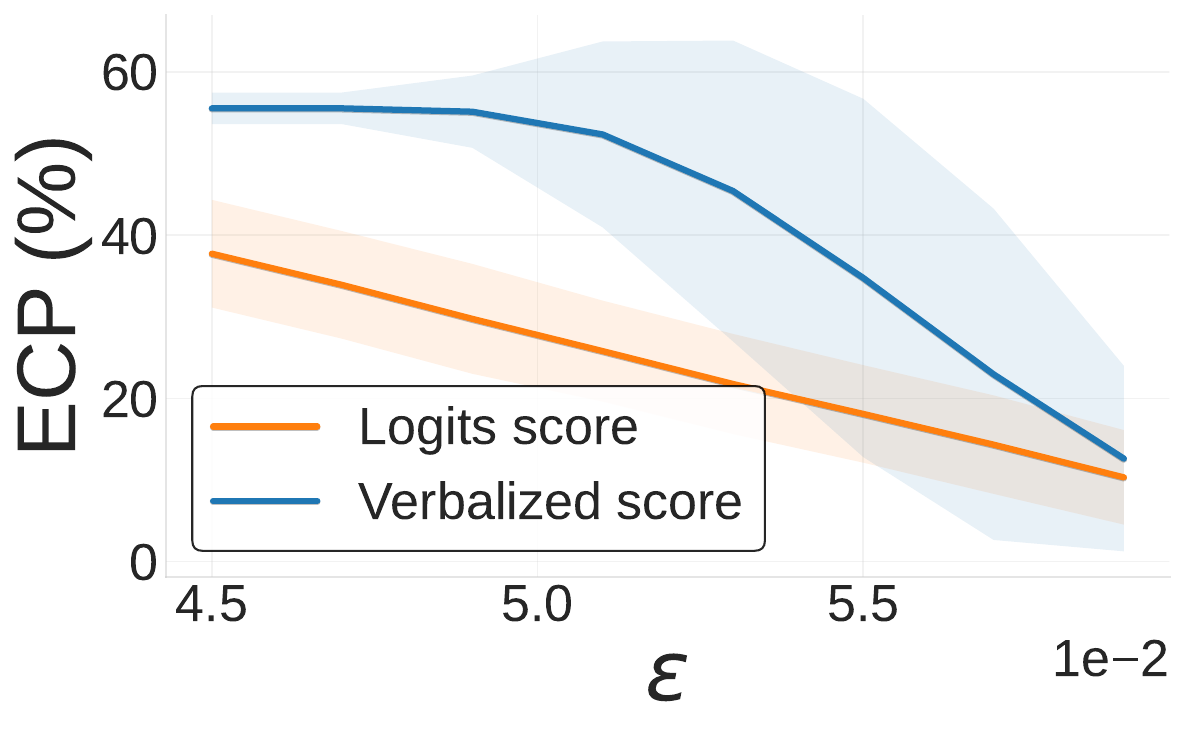}
        \includegraphics[width=0.33\linewidth]{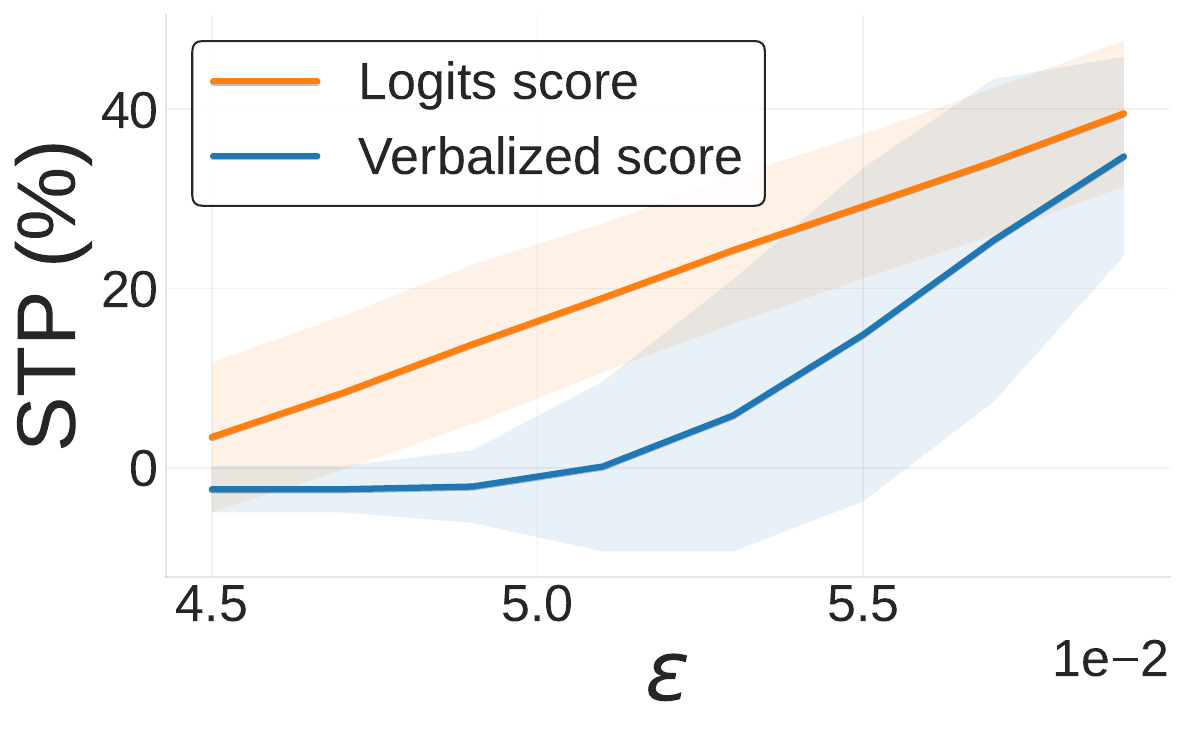}
    }
        \\
    \subcaptionbox{Arena-Hard}{
        \includegraphics[width=0.33\linewidth]{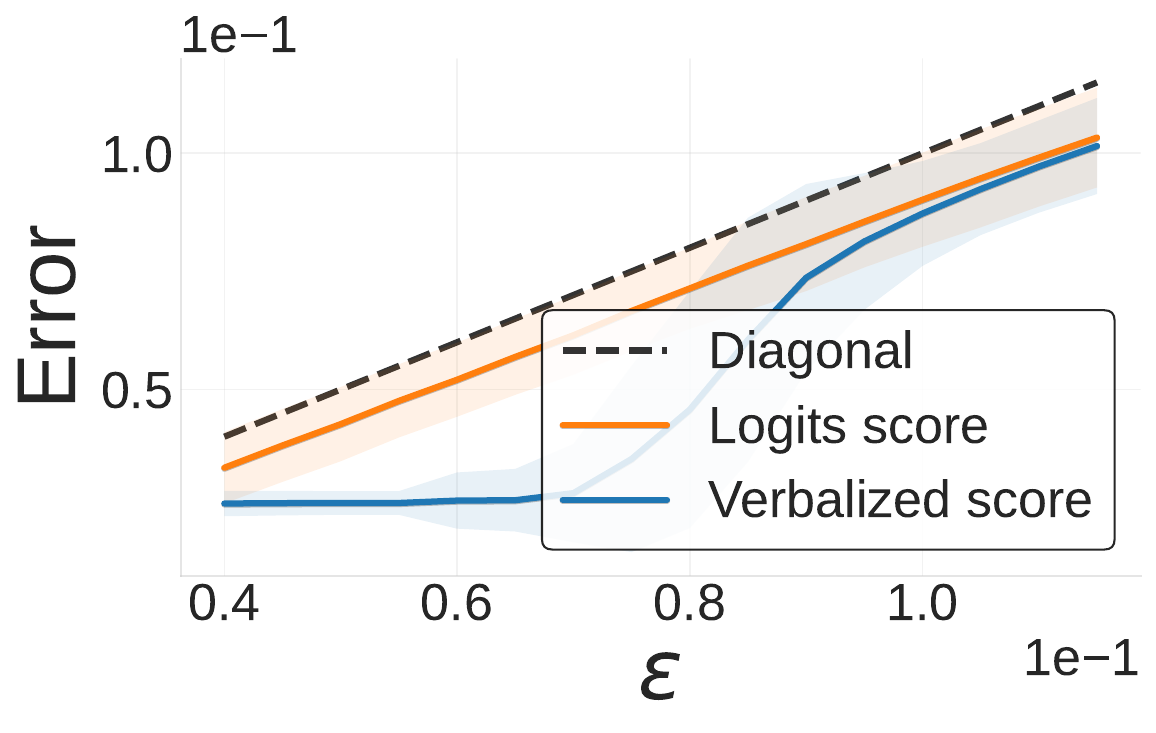}
        \includegraphics[width=0.33\linewidth]{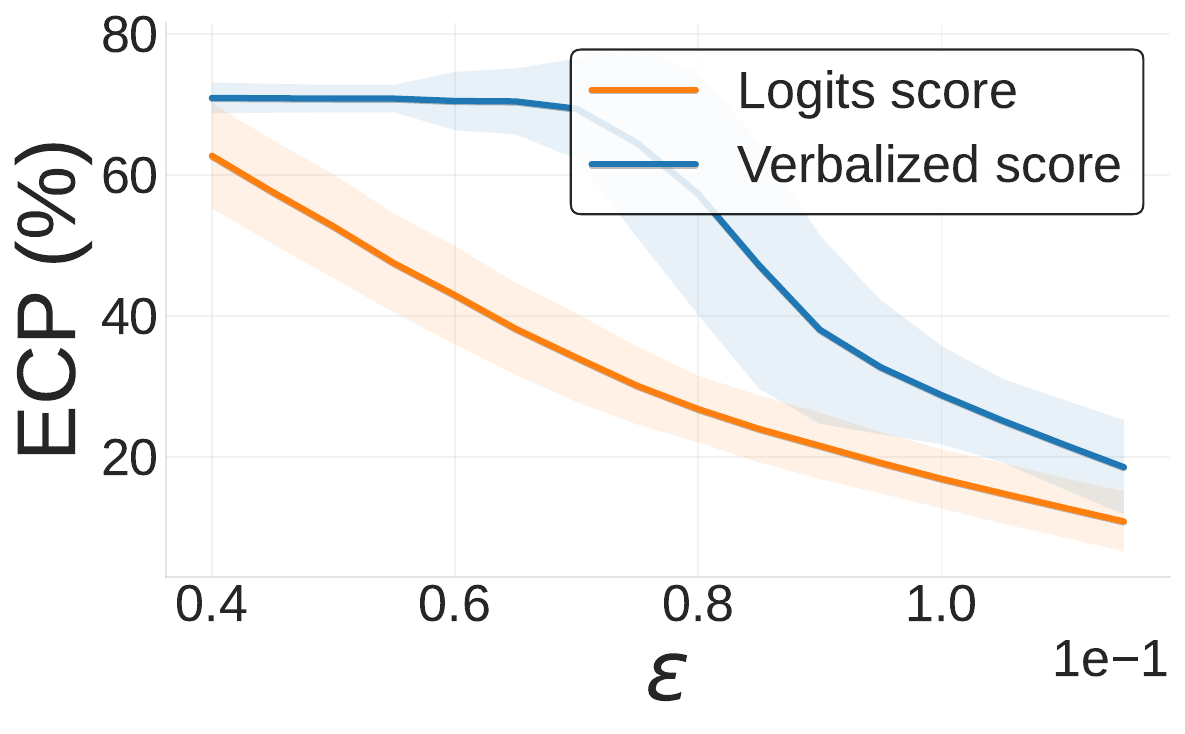}
        \includegraphics[width=0.33\linewidth]{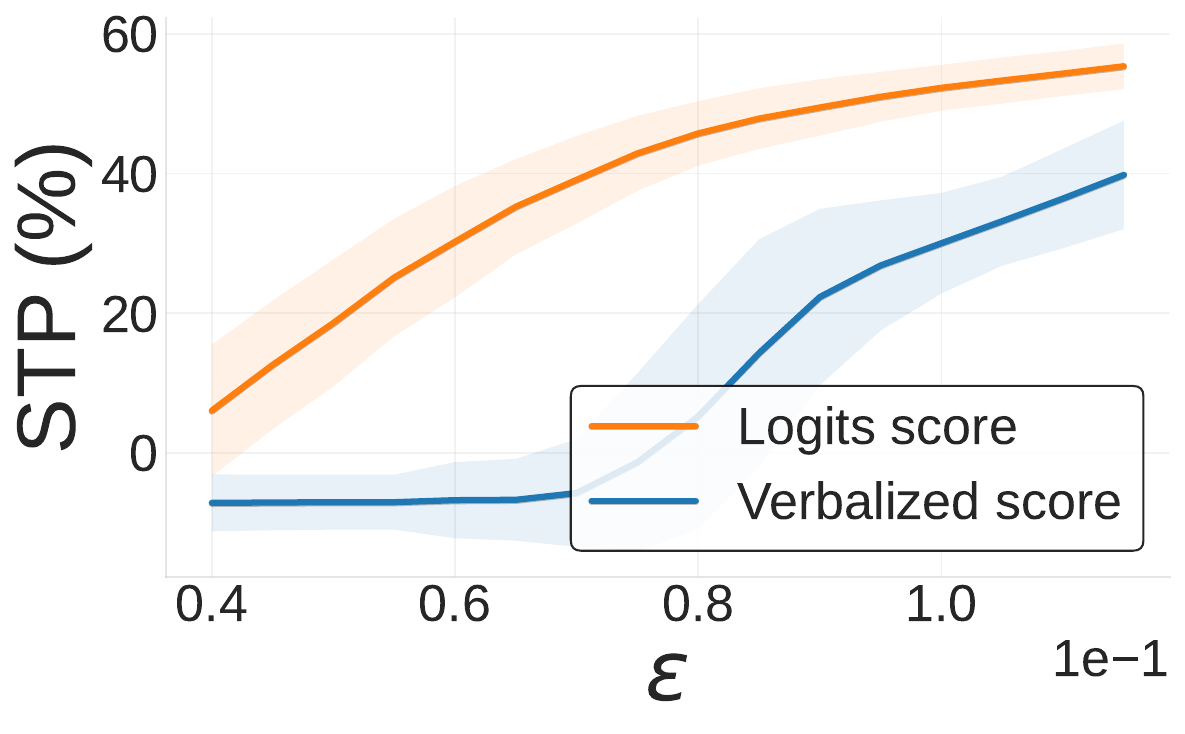}
    }
    \caption{\textbf{Error control, ECP and STP of PAC Reasoning} for semantic loss across three benchmarks at $\alpha = 0.05$.
    Uncertainty score includes the logits-based score and the verbalized score.
    All experiments are repeated 100 times, and the shaded areas represent standard deviations.
    }\label{fig:semantic_results}
    \vspace{-10pt}
\end{figure}

\textbf{Metrics.}
To evaluate the effectiveness of PAC Reasoning in budget usage, we define two metrics: \textbf{Expert Call Percentage} (ECP) and \textbf{Saved Token Percentage} (STP). These metrics are defined by:
\begin{align*}
    \text{ECP} := \frac{|\{i : U_i \geq \hat{u}, \, i \in \mathcal{I}_{\text{test}}\}|}{N}\times100\%,
    \text{STP} := \frac{1}{N} \sum_{i \in \mathcal{I}_{\text{test}}} \left(1-\frac{l_{\tilde{y}_i} + \mathbbold{1}\{U_i \geq \hat{u}\} l_{y_i}}{l_{y_i}}\right)\times100\%,
\end{align*}
where $l_{\tilde{y}_i}$ and $l_{y_i}$ represent the number of tokens in the candidate answer $\tilde{y}_i$ and the reference answer $y_i$, respectively. The ECP measures the proportion of test cases requiring expert intervention. STP quantifies the token efficiency by comparing the token counts of candidate and reference answers, accounting for cases where expert calls are triggered. We also report the accuracy of our proposed method in Appendix~\ref{sec:acc}.

We repeat each experiment 100 times and report the mean and standard deviation of the budget savings. We fix $\alpha = 0.05$ throughout all experiments while varying $\epsilon$, and set the sampling weight $\pi = \pi_i = 0.5$ for each \(i\in \mathcal{I}_{cal}\) and the sample size $m = n/\pi$.

\subsection{Results}\label{sec:exp_results}

\textbf{PAC Reasoning improves efficiency while controlling performance loss.}
Figure~\ref{fig:semantic_results} reports the results under semantic loss.
Across all benchmarks, PAC Reasoning consistently keeps the empirical performance loss below the target tolerance $\epsilon$ while achieving substantial budget savings.
For instance, on Arena-Hard with $\epsilon = 0.08$ and the logits-based score, it incurs only about $0.06$ empirical performance loss, with ECP and STP of roughly $20\%$ and $40\%$, respectively.
These trends are consistent with those obtained using the Llama-based model in Appendix~\ref{app:exp_other_llm}, and also hold on additional benchmarks such as GPQA~\citep{rein2024gpqa} and HumanEval~\citep{chen2021evaluating} in Appendix~\ref{app:more_data}.
Moreover, PAC Reasoning improves Pass@1 over the non-thinking baseline under semantic-loss control, as shown in Appendix~\ref{sec:acc}.
Under the binary loss function, PAC Reasoning also maintains error rates within the target tolerance while saving computational budget (See Table~\ref{table:binary_results_more_methods}).
For example, on MATH-500, PAC Reasoning with the logits-based score calls the thinking model on 21.48\% of test queries and saves 37.61\% of tokens.
{We also report success counts over 100 runs and exact $95\%$ binomial intervals for the empirical test-loss criterion in Appendix~\ref{sec:test_loss_success}.}
Beyond token savings, it also reduces wall-clock latency and dollar cost in real deployments (See Table~\ref{tab:latency_cost} in Appendix~\ref{sec:binary_result}).
Overall, PAC Reasoning bounds performance loss within the target tolerance and substantially improves inference efficiency on benchmarks.

\begin{table*}[t]
\centering
\caption{\textbf{PAC Reasoning enables safe, effective inference-cost reduction over heuristic baselines}. 
\colorbox{green!10}{Green cells indicate valid loss control.} 
Experiments use binary loss with $\epsilon=0.03$ on MATH-500 and $\epsilon=0.08$ on ZebraLogic.}
\label{table:binary_results_more_methods}
\setlength{\tabcolsep}{1.2mm}{
\resizebox{\textwidth}{!}{
\begin{tabular}{llcccccc}
\toprule
\multirow{2}{*}{Method}
& \multirow{2}{*}{Score}
& \multicolumn{3}{c}{MATH-500}
& \multicolumn{3}{c}{ZebraLogic} \\
\cmidrule(r){3-5} \cmidrule(r){6-8}
& 
& Error $\downarrow$ 
& ECP (\%) $\downarrow$ 
& STP (\%) $\uparrow$
& Error $\downarrow$ 
& ECP (\%) $\downarrow$ 
& STP (\%) $\uparrow$ \\
\midrule

\multirow{2}{*}{Naive threshold}
& Logits-based
& \colorbox{green!10}{$0.0179 \pm 0.0068$}
& $14.44 \pm 2.02$
& $43.58 \pm 4.78$
& \colorbox{green!10}{$0.0062 \pm 0.0026$}
& $77.28 \pm 1.36$
& $-34.78 \pm 1.26$ \\

& Verbalized
& $0.0346 \pm 0.0095$
& $2.83 \pm 0.94$
& $66.67 \pm 4.91$
& \colorbox{green!10}{$0.0631 \pm 0.0074$}
& $12.49 \pm 1.07$
& $32.70 \pm 2.11$ \\

\cmidrule(lr){2-8}

\multirow{2}{*}{Naive control}
& Logits-based
& $0.0351 \pm 0.0098$
& $6.04 \pm 2.75$
& $61.52 \pm 6.10$
& $0.0959 \pm 0.0099$
& $5.47 \pm 1.63$
& $46.59 \pm 2.43$ \\

& Verbalized
& $0.0351 \pm 0.0098$
& $ 6.04 \pm 2.75$
& $61.52 \pm 6.10$
& $0.0948 \pm 0.0093$
& $6.56 \pm 0.92$
& $42.75 \pm 1.95$ \\

\cmidrule(lr){2-8}

\multirow{2}{*}{PAC Reasoning}
& Logits-based
& \colorbox{green!10}{$0.0206 \pm 0.0126$}
& $21.48 \pm 17.85$
& $37.61 \pm 23.19$
& \colorbox{green!10}{$0.0618 \pm 0.0181$}
& $22.50 \pm 7.47$
& $23.13 \pm 9.09$ \\

& Verbalized
& \colorbox{green!10}{$0.0209 \pm 0.0141$}
& $24.59 \pm 20.48$
& $36.13 \pm 26.44$
& \colorbox{green!10}{$0.0530 \pm 0.0246$}
& $26.95 \pm 20.68$
& $21.21 \pm 17.20$ \\

\midrule

Router
& --
& $0.0435 \pm 0.0107$
& $1.97 \pm 0.74$
& $74.73 \pm 2.59$
& \colorbox{green!10}{$0.0781 \pm 0.0082$}
& $23.77 \pm1.39$
& $39.81 \pm 1.06$ \\

CoD
& --
& $0.3548 \pm 0.0604$
& $--$
& $-1.33 \pm 0.82$
& \colorbox{green!10}{$0.0492 \pm 0.0071$}
& $--$
& $0.03 \pm 0.62$ \\

NoThinking
& --
& $0.4445 \pm 0.0583$
& $--$
& $-4.51 \pm 1.00$
& \colorbox{green!10}{$0.0392 \pm 0.0066$}
& $--$
& $-0.52 \pm 0.69$ \\

\bottomrule
\end{tabular}}}
\end{table*}

\textbf{PAC Reasoning saves inference cost without sacrificing error control.}
Table~\ref{table:binary_results_more_methods} compares PAC Reasoning with multiple baselines under binary loss, using $\epsilon=0.03$ for MATH-500 and $\epsilon=0.08$ for ZebraLogic.
These tolerances are representative operating points for the two tasks.
In an application, $\epsilon$ should reflect the acceptable loss relative to the thinking model; Appendix~\ref{app:re_calibration_cost} discusses this choice.
PAC Reasoning satisfies the target error tolerance on both datasets while achieving substantial token savings, reflecting its explicit PAC-style control of performance loss.
In contrast, heuristic methods may yield higher token savings but do not reliably control downstream prediction risk.
For example, the router method attains the highest STP on MATH-500, but its empirical error exceeds the target tolerance.
Similarly, CoD and NoThinking reduce reasoning cost by avoiding full reasoning, but can incur severe performance degradation, especially on MATH-500.
Overall, PAC Reasoning provides a principled trade-off between efficiency and reliability, distinguishing it from heuristic approaches that lack explicit risk control.

\textbf{Logits-based uncertainty score is more stable.} From the results in Figure~\ref{fig:semantic_results} and Table~\ref{table:binary_results_more_methods}, the logits-based uncertainty score consistently shows lower variance in both ECP and STP.
In contrast, the verbalized score exhibits larger fluctuations due to its sparse and clustered distribution. For instance, on ZebraLogic (see Table~\ref{table:binary_results_more_methods}) with PAC reasoning, the standard deviations of ECP and STP under the verbalized score are $20.68\%$ and $17.20\%$, considerably higher than the corresponding $7.47\%$ and $9.09\%$ values for the logits-based score.
Similar patterns can also be observed in Figure~\ref{fig:semantic_results}, where the variance of the verbalized score is consistently higher than that of the logits-based score across different settings.
Therefore, although the verbalized score occasionally achieves stronger risk control or higher savings, its calibration is less reliable, leading to less consistent performance.

\begingroup
Appendix~\ref{app:re_accuracy_cost} compares accuracy and loss at four expert-call budgets, and Appendix~\ref{app:re_llama} reports additional operating points for the Llama pair.
\endgroup

\section{Discussion}
In this section, we study how robust PAC Reasoning is to practical design.
We first consider the reliability of the uncertainty score using expected calibration error (ECE).
Then we study the ablation on varying the uncertainty score and the calibration set size.
Across these settings, PAC Reasoning keeps valid error control and shows stable efficiency.

\textbf{Quality of uncertainty score.}
To assess the reliability of uncertainty scores, we compute the ECE, which measures the difference between predicted confidence and actual accuracy.
As shown in Appendix~\ref{sec:ece}, logits-based scores achieve lower ECE than verbalized scores on most benchmarks.
For example, on MATH-500, the logits-based score achieves an ECE of 0.0450, while the verbalized score has an ECE of 0.0634.
Despite the differences in calibration quality, both approaches successfully maintain the PAC guarantee, demonstrating the robustness of our framework to imperfect uncertainty quantification.
It means that perfect calibration is not necessary for PAC Reasoning.
{Appendix~\ref{app:re_score_sensitivity} separates score compression from quantization and tests three UCBs under regular and sparse losses.}

\textbf{Calibration set size.}
Calibration size introduces a practical trade-off between calibration cost and guarantee tightness.
Appendix~\ref{sec:size_cal} studies calibration ratios from 10\% to 90\% and shows that error control remains stable, with average loss staying within the target tolerance even for small calibration sets.
For instance, on MATH-500 with the logits-based score, the average loss changes only from 0.0331 to 0.0344 as the calibration ratio increases from 10\% to 90\%, while STP remains between 14.39\% and 18.52\%.
These results describe the observed trade-off across calibration sizes.
Appendix~\ref{app:re_calibration_cost} gives the cost decomposition and deployment break-even condition.

\textbf{Alternative uncertainty scores.}
PAC Reasoning can use any scoring function correlated with the disagreement rate.
One alternative is reward model-based scoring, where a trained reward model evaluates the non-thinking model's output.
We conduct experiments in Appendix~\ref{sec:reward} using reward scores on MATH-500 and ZebraLogic, demonstrating that PAC Reasoning maintains valid error control even with this alternative scoring method.
For example, on MATH-500 with $\epsilon = 0.045$, reward-based PAC Reasoning achieves an error of 0.0362 with 22.43\% token savings.
This flexibility is particularly valuable in scenarios where uncertainty scores are unreliable or unavailable.

\textbf{Empirical behavior under distribution shift.}
We calibrate on 300 MATH-500 examples and evaluate on 200-example mixed sets containing 20, 40, or 60 AIME problems, repeated 100 times under binary loss.
Across these mixtures, the mean test loss stays below the target tolerance, while the escalation rate increases and token savings decrease as the AIME proportion grows.
Time savings are 61.73\%--69.02\% and monetary savings are 21.17\%--37.24\%.
Appendix~\ref{app:re_shift} reports MATH-to-HumanEval and AIME-only evaluations.

\section{Conclusion}\label{sec:conclusion}
We propose PAC Reasoning, a practical efficiency improvement method with statistical guarantees for reasoning models.
Our approach constructs a composite model that selectively uses either an expert model or a candidate model based on a constructed confidence bound.
PAC Reasoning contributes by delivering statistical assurances for model performance and demonstrating that it reliably achieves the specified error rate with probability.
We validate our method through extensive experiments on real-world reasoning tasks, showing it can reduce computational costs while maintaining user-specified error tolerances with confidence.
Future research will focus on developing better uncertainty estimation methods, exploring tighter theoretical bounds, and broadening the method's applicability to other large language model efficiency-improvement strategies.

\subsection{Limitations}\label{sec:limitations}
The population-risk guarantee assumes i.i.d.\ calibration data from the deployment law and valid pointwise population-risk UCBs, with the models, score, loss, and finite candidate sequence fixed before calibration.
The conditional CLT in Appendix~\ref{sec:clt_ucb_validity} controls resampling error within a fixed calibration pool; population-risk validity also requires control of calibration-sample error.
Calibration requires extra expert queries and may need to be repeated when the models, score, loss, or deployment law change; reported ECP and STP exclude amortized calibration cost.
Efficiency depends on the score and deployment law, and savings can disappear under distribution shift, which is outside the stated guarantee.
The guarantee bounds loss relative to the chosen thinking model under the specified loss function; gold-answer accuracy is reported separately in Appendix~\ref{sec:acc}.

\begin{ack}
This research is supported by Guangdong Basic and Applied Basic Research Foundation (Grant No. 2026A1515011367).
This project is also supported by the Jiangsu Provincial Key Discipline Construction Project (Statistics) and the Open Project of the Joint Lab for Statistics and Finance (Grant No.~2025JLSF101).
We thank Shenzhen Guangming Science City Development and Construction Co., Ltd. and its Material Genome Big-Science Facilities Platform for providing technical support and assistance with data collection and analysis.
\end{ack}

\newpage

\bibliography{pac}
\bibliographystyle{plainnat}

\newpage
\appendix

\section{Related work}\label{sec:related_work}
Our work intersects efficiency improvement for reasoning models and the distribution-free inference for risk control of its performance loss with confidence.

\textbf{Efficiency improvement for reasoning models.}
Large Reasoning Models (LRMs) have recently become a research hotspot due to their outstanding performance in handling complex tasks~\citep{yue2025dont}.
However, the problem of overthinking has emerged~\citep{sui2025stop,chen2025reasoning}, where LRMs tend to engage in unnecessarily long reasoning chains and redundant computational steps.
This increases latency and cost, and may even cause error accumulation through extended reasoning paths.
For example, in mathematical problems, LRMs may explore irrelevant solution branches or perform excessive intermediate calculations that do not contribute to the final answer.
To alleviate this issue, recent studies propose efficient reasoning strategies such as early exit of thinking~\citep{yang2025dynamic, jiang2025flashthink} and adaptive switching between ``fast'' and ``slow'' thinking modes to reduce reasoning tokens and avoid redundant steps~\citep{cheng2025think, chung2025thinker, fang2025thinkless, li2025dynamicmind, liang2025thinkswitcher, ma2025reasoning, paliotta2025thinking, pan2024dynathink, xiao2025fastslow, yao2024hdflow, yong2025think, zhang2025avengers, zhang2025gpt5}.
\begingroup
\paragraph{LLM routing and cascades}
Table~\ref{tab:related_cascades} compares the objectives and statistical analyses of six related methods.
FrugalGPT learns a cascade and stopping thresholds under an average query-cost budget~\citep{chen2023frugalgpt}, while Mixture of Thoughts uses answer consistency across Chain-of-Thought and Program-of-Thought samples without training a separate router~\citep{yue2023large}.
TREACLE learns a budget-aware policy over models, prompts, and repeated queries~\citep{zhang2024efficient}.
Online cascade learning instead updates small models from LLM demonstrations as inputs arrive and gives a no-regret learning algorithm~\citep{nie2024online}.

C3PO minimizes disagreement with the most powerful model subject to a bound on the probability of exceeding an inference-cost budget, and also gives regret generalization bounds~\citep{valkanas2025c3po}.
PAC Reasoning studies the complementary constraint: reducing thinking-model use while bounding expert-relative population loss by a user-chosen tolerance.
Trust or Escalate calibrates selective LLM judges with exact binomial bounds and fixed-sequence testing to control disagreement with human preferences among accepted evaluations~\citep{jung2024trust}.
Its cascade may abstain, whereas our loss averages over all inputs and the thinking model serves as the fallback.

\begin{table}[htbp]
\centering
\small
\caption{{Objectives and statistical analyses of related LLM routing and cascade methods.
The PAC Reasoning row states the conditions used by our selection theorem.}}
\label{tab:related_cascades}
\setlength{\tabcolsep}{3pt}
\renewcommand{\arraystretch}{1.15}
\begin{tabular}{p{0.18\linewidth}p{0.35\linewidth}p{0.40\linewidth}}
\toprule
Method & Objective and controlled quantity & Calibration or analysis \\
\midrule
FrugalGPT & Maximize task accuracy under an average inference-cost budget. & Supervised answer-correctness scores; empirical accuracy--cost evaluation. \\
Mixture of Thoughts & Reduce API cost while maintaining task accuracy. & Training-free answer consistency; empirical accuracy--cost evaluation. \\
TREACLE & Maximize answer accuracy under a total cost budget. & Learned reinforcement-learning policy; theory for oracle ordering and answer-consistency rules. \\
C3PO & Minimize strongest-model disagreement under a probabilistic cost constraint. & Conformal cost calibration under exchangeability; additional regret generalization bounds. \\
Trust or Escalate & Control human-disagreement risk among accepted LLM evaluations. & I.i.d.\ human-labelled calibration, exact binomial bounds, and fixed-sequence testing. \\
Online Cascade & Learn a cascade and deferral policy for input streams. & Online imitation learning from LLM demonstrations with a no-regret algorithm. \\
PAC Reasoning & Reduce thinking-model use subject to an expert-relative population-loss tolerance. & I.i.d.\ calibration, a fixed candidate sequence, and valid pointwise population-risk UCBs. \\
\bottomrule
\end{tabular}
\end{table}
\endgroup

\begingroup
\textbf{PAC-style error rate control.}
Conformal prediction provides finite-sample marginal coverage under exchangeability~\citep{angelopoulos2025theoretical,liu2026conformal}.
Conformal risk control extends this idea to the expected value of a monotone loss~\citep{angelopoulos2023conformal}.
Risk-controlling prediction sets provide distribution-free calibration of predictive risks~\citep{bates2021distributionfree}, and Learn then Test frames risk calibration as multiple hypothesis testing~\citep{angelopoulos2025learn}.
\endgroup
These methods have been applied to various domains, including conformal language modeling for text generation with guaranteed quality~\citep{quach2024conformal}, and automatically adaptive conformal risk control that adjusts to the difficulty of test samples~\citep{blot2025automatically}.
Localized adaptive risk control enables online calibration with conditional guarantees~\citep{zecchin2024localized}.
\begingroup
PAC Reasoning uses this risk-control view for thinking/non-thinking acceleration.
Its expert-relative loss makes the expert branch loss zero, and the fixed-sequence rule converts valid pointwise population-risk bounds into a guarantee for the selected threshold.
Our Naive control baseline uses the empirical calibration loss directly, as specified in Appendix~\ref{sec:baseline}.
\endgroup

\section{Finite-Sample Upper Confidence Bounds}\label{app:bounded_loss_alg}
This section presents an alternative algorithm for computing confidence bounds that provides strict finite-sample guarantees under bounded loss assumptions in Algorithm~\ref{alg:compute_bound_finite_sample}.
This approach leverages concentration inequalities such as the Hoeffding inequality or betting-based confidence intervals~\citep{bentkus2004hoeffdings, hao2019bootstrapping, hoeffding1994probability, learned-miller2020new, ramdas2022admissible, waudby-smith2021confidence, waudby-smith2024estimating} which provides finite-sample validity, typically at the cost of looser bounds than the asymptotic normal approximation.

\begin{algorithm}[ht]
\caption{Compute Confidence Bound $\hat{L}_u(\alpha)$ Based on Hoeffding Inequality}\label{alg:compute_bound_finite_sample}
\textbf{Input:} Calibration set $\{(x_i, y_i)\}_{i=1}^n$, model with thinking $f$, model without thinking $\tilde{f}$, uncertainty scores $\{U_i\}_{i=1}^n$, sampling weights $\{\pi_i\}_{i=1}^n$, sampling size $m$, confidence level $\alpha$, loss upper bound $B > 0$ \\
\textbf{Output:} The finite-sample upper confidence bound $\hat{L}_u(\alpha)$.
\begin{algorithmic}[1]
\STATE Initialize an empty list of samples $\mathcal{Z} = []$.
\STATE Let $\tilde{y}_i = \tilde{f}(x_i)$ for all $i=1, \dots, n$.
\FOR{$j = 1, \dots, m$}
    \STATE Sample an index $i_j \sim \text{Unif}(\{1, \dots, n\})$.
    \STATE Sample a Bernoulli random variable $\xi_{i_j} \sim \text{Bern}(\pi_{i_j})$.
    \IF{$\xi_{i_j} = 1$}
        \STATE Query the true label $y_{i_j}$ and compute $Z_j = \min\left(\frac{\ell(y_{i_j}, \tilde{y}_{i_j})}{\pi_{i_j}}, \frac{B}{\pi_{i_j}}\right)$.
    \ELSE
        \STATE $Z_j = 0$.
    \ENDIF
    \STATE Append $Z_j$ to $\mathcal{Z}$.
\ENDFOR
\STATE For a threshold $u$, define the variables $Z_j(u) = Z_j \cdot \mathbf{1}\{U_{i_j} < u\}$ for $j=1, \dots, m$.
\STATE $\hat{\mu}_Z(u) \leftarrow \frac{1}{m} \sum_{j=1}^m Z_j(u)$
\STATE $R \leftarrow \frac{B}{\min_i \pi_i}$
\STATE $\delta_{\text{HB}}(\alpha) \leftarrow \sqrt{\frac{R^2 \log(2/\alpha)}{2m}}$
\STATE \textbf{Return} $\hat{L}_u(\alpha) \leftarrow \hat{\mu}_Z(u) + \delta_{\text{HB}}(\alpha)$.
\end{algorithmic}
\end{algorithm}

\section{Validity of CLT-based upper confidence bound}\label{sec:clt_ucb_validity}
\begingroup
We separate the $n$ original calibration observations from the $m$ auditing draws in Algorithm~\ref{alg:compute_bound}.
For a fixed threshold $u$, write
\[
g_i(u)=\ell(y_i,\tilde y_i)\mathbf1\{U_i<u\},
\qquad
\overline L_n(u)=\frac1n\sum_{i=1}^n g_i(u).
\]
Conditional on $\mathcal D_{cal}$ and its weights $0<\pi_i\le1$, the independent index and fresh Bernoulli draws give
\[
\mathbb E[Z_j(u)\mid\mathcal D_{cal}]=\overline L_n(u),
\qquad
\sigma_{Z,n}^2(u):=\operatorname{Var}(Z_j(u)\mid\mathcal D_{cal})
=\frac1n\sum_{i=1}^n\frac{g_i(u)^2}{\pi_i}-\overline L_n(u)^2.
\]

\begin{proposition}[Conditional asymptotic validity of the CLT UCB]\label{prop:clt_ucb_validity}
Fix $\alpha\in(0,1)$, the calibration pool, a threshold $u$, and weights for which $0<\sigma_{Z,n}^2(u)<\infty$.
For $m\ge2$, let $\widehat\mu_Z(u)=m^{-1}\sum_{j=1}^m Z_j(u)$ and define the sample variance and sample standard deviation by
\[
\widehat\sigma_Z^2(u)=\frac1{m-1}\sum_{j=1}^m\big(Z_j(u)-\widehat\mu_Z(u)\big)^2,
\qquad
\widehat\sigma_Z(u)=\sqrt{\widehat\sigma_Z^2(u)}.
\]
Define
\[
\widehat L^{\mathrm{CLT}}_u(\alpha)
:=\widehat\mu_Z(u)+z_{1-\alpha}\frac{\widehat\sigma_Z(u)}{\sqrt m},
\]
where $z_{1-\alpha}$ is the $(1-\alpha)$-quantile of the standard normal distribution.
Then, with probability over the auditing draws,
\[
\lim_{m\to\infty}
\mathbb P\!\left(\overline L_n(u)\le\widehat L^{\mathrm{CLT}}_u(\alpha)\mid\mathcal D_{cal}\right)
=1-\alpha.
\]
\end{proposition}

\begin{proof}
Conditional on $\mathcal D_{cal}$, the variables $Z_j(u)$ are i.i.d.\ with finite positive variance, so the i.i.d.\ central limit theorem gives
\[
\frac{\sqrt m\big(\widehat\mu_Z(u)-\overline L_n(u)\big)}{\sigma_{Z,n}(u)}
\xrightarrow{d}\mathcal N(0,1).
\]
The law of large numbers gives $\widehat\sigma_Z^2(u)\xrightarrow{p}\sigma_{Z,n}^2(u)$ and hence $\widehat\sigma_Z(u)\xrightarrow{p}\sigma_{Z,n}(u)$.
All limits in this proof are conditional on the fixed pool.
Slutsky's theorem and the symmetry of the standard normal law imply
\[
\frac{\overline L_n(u)-\widehat\mu_Z(u)}{\widehat\sigma_Z(u)/\sqrt m}
\xrightarrow{d}\mathcal N(0,1).
\]
On the event $\widehat\sigma_Z(u)>0$, this statistic is at most $z_{1-\alpha}$ exactly when $\overline L_n(u)\le\widehat L^{\mathrm{CLT}}_u(\alpha)$.
Since $\mathbb P(\widehat\sigma_Z(u)>0\mid\mathcal D_{cal})\to1$, the claim follows.
\end{proof}

\paragraph{Calibration and resampling randomness}
Under i.i.d.\ calibration sampling from $P$, the population risk is $R(u)=\mathbb E[\overline L_n(u)]$ and
\[
\widehat\mu_Z(u)-R(u)
=\big(\widehat\mu_Z(u)-\overline L_n(u)\big)
+\big(\overline L_n(u)-R(u)\big).
\]
The conditional CLT controls the first term, which comes from the $m$ auditing draws.
The second term comes from the $n$ original calibration observations and remains random when $m$ grows with $n$ fixed.
Thus, increasing $m$ alone does not establish a population-risk UCB.
To use the CLT UCB in Assumption~\ref{assump:validity}, its studentized limit must hold for $R(u)$ under a joint sampling regime that includes both terms, or a separate bound must cover the calibration-sample error.
\endgroup

\section{Proof of Theorem~\ref{thm:pac_guarantee}}\label{sec:proof_pac_guarantee}

\begin{proof}
Write $\hat u=u_{\hat k}$, where $\hat k\in\{0,\ldots,K\}$ is the last index accepted by Algorithm~\ref{alg:pac_reasoning}.
The fallback has risk $R(u_0)=0$.
If $R(u_k)\le\epsilon$ for every $k=1,\ldots,K$, the output is safe for every realization of the calibration data and the UCB randomness.

Otherwise, define the first unsafe index
\[
k^*=\min\{k\in\{1,\ldots,K\}:R(u_k)>\epsilon\}.
\]
The population law, models, score function, tolerance, and candidate grid are fixed before calibration.
Hence $k^*$ and $u_{k^*}$ are fixed, even though they are unknown to the algorithm.
By monotonicity, the safe grid indices are exactly $0,\ldots,k^*-1$.
In particular, the index $k^*$ itself is unsafe.

If the selected policy is unsafe, then $\hat k\ge k^*$.
The algorithm accepts only a prefix of the fixed sequence and stops at the first failed test.
It must therefore have accepted $k^*$, giving
\[
\begin{aligned}
\{R(\hat u)>\epsilon\}
&\subseteq
\{\widehat L_{u_{k^*}}(\alpha)\le\epsilon\}\\
&\subseteq
\{\widehat L_{u_{k^*}}(\alpha)<R(u_{k^*})\}.
\end{aligned}
\]
Assumption~\ref{assump:validity}, applied at this single fixed candidate, yields
\[
\mathbb P\big(R(\hat u)>\epsilon\big)
\le
\mathbb P\big(\widehat L_{u_{k^*}}(\alpha)<R(u_{k^*})\big)
\le\alpha.
\]
Taking complements proves the claim.
This argument requires neither independence between tests nor a simultaneous confidence event over the grid.
It also covers a failure at the first candidate and the case where all tests pass.
\end{proof}

\paragraph{Asymptotic pointwise bounds}
Keep $P$, the models, the score function, $\epsilon$, $\alpha$, and the finite grid fixed, and let $r\to\infty$ index a sampling regime for the UCB.
Suppose that, for each fixed grid point,
\[
\limsup_{r\to\infty}
\mathbb P\big(\widehat L^{(r)}_u(\alpha)<R(u)\big)\le\alpha.
\]
Apply the same selection rule to these bounds and denote its output by $\hat u^{(r)}$.
If an unsafe candidate exists, the same fixed $k^*$ gives the event inclusion above for every $r$.
If no unsafe candidate exists, every output is safe.
Thus,
\[
\liminf_{r\to\infty}
\mathbb P\big(R(\hat u^{(r)})\le\epsilon\big)\ge1-\alpha.
\]
This result requires pointwise population-risk validity under the full calibration and UCB sampling law.

\section{Proof of Theorem~\ref{thm:empirical_test_pac_appendix}}\label{sec:proof_of_empirical_test_pac_appendix}
\subsection{Notation recalling and lemma}
We present PAC guarantees for the empirical test risk under precise Hoeffding conditions, making explicit the roles of calibration and test randomness.
For any threshold $u$, the deployment rule $T_u$ predicts with the expert when $U(x)\ge u$ and otherwise uses the fast model.
The population risk is
\[
R(u)\;=\;\mathbb{E}_{(x,y)\sim P}\big[\ell\big(y,\,T_u(x)\big)\big].
\]
Given an independent test set $\mathcal{D}_\text{test}$ drawn i.i.d.\ from $P$, the empirical test risk is
\[
\widehat R(u)\;=\;\frac{1}{N}\sum_{j=n+1}^{n+N} \ell\big(y_j,\,T_u(x_j)\big).
\]
Our guarantee for the empirical test risk uses the conditions of Theorem~\ref{thm:pac_guarantee} and the following lemma.

\begin{lemma}[Conditional Hoeffding bound]\label{lem:hoeffding}
Let $\hat u$ be determined by $\mathcal D_{cal}$ and the UCB randomization, jointly independent of the i.i.d.\ test sample.
Conditioned on $\hat u$, the test losses $Z_j(\hat u):=\ell(y_j,T_{\hat u}(x_j))$ for $j\in\mathcal I_{test}$ are then i.i.d.\ with mean $R(\hat u)$.
Assume they are bounded in $[a,b]$ with $a<b$.
Then for any $t>0$,
\[
\mathbb{P}\!\big(\widehat R(\hat u)-R(\hat u)>t\,\big|\,\hat u\big)\;\le\;\exp\!\Big(-\tfrac{2 N t^2}{(b-a)^2}\Big).
\]
\end{lemma}
\begin{proof}
Let $Z_j(\hat u) = \ell(y_j, T_{\hat u}(x_j))$ for $j\in \mathcal{I}_{test}$.
The threshold $\hat u$ is fixed when we condition on it.
Since the i.i.d.\ test sample is independent of the calibration data and UCB randomization, the variables $\{Z_j(\hat u):j\in\mathcal I_{test}\}$ are conditionally independent and identically distributed.

By the boundedness of the loss function, each $Z_j(\hat u)$ is bounded in $[a, b]$.
The conditional expectation of each $Z_j(\hat u)$ is
$ \mathbb{E}[Z_j(\hat u) \,|\, \hat u] = \mathbb{E}[\ell(y_j, T_{\hat u}(x_j)) \,|\, \hat u] $.
Since $(x_j, y_j)$ is independent of $\hat u$, this is equal to the unconditional expectation over the data distribution, $ \mathbb{E}_{(x,y) \sim P}[\ell(y, T_{\hat u}(x))] $, which is the definition of the true risk $R(\hat u)$.
The empirical risk is the sample mean: $$\widehat{R}(\hat u)=\frac1N\sum_{j\in\mathcal I_{test}} Z_j(\hat u).$$
Its conditional expectation is $$\mathbb E[\widehat R(\hat u)\mid\hat u]=\frac1N\sum_{j\in\mathcal I_{test}}\mathbb E[Z_j(\hat u)\mid\hat u]=R(\hat u).$$

We can now apply Hoeffding's inequality \citep{hoeffding1994probability} to the conditional i.i.d. bounded variables $Z_j(\hat u)$.
For any $t > 0$, the one-sided version states that
$$ \mathbb{P}\left( \frac{1}{N}\sum_{j=n+1}^{n+N} Z_j(\hat u) - \mathbb{E}\left[\frac{1}{N}\sum_{j=n+1}^{n+N} Z_j(\hat u) \,\middle|\, \hat u\right] > t \,\middle|\, \hat u \right) \le \exp\left(-\frac{2Nt^2}{(b-a)^2}\right). $$
Substituting the empirical risk and true risk, we get
$$ \mathbb{P}\left( \widehat{R}(\hat u) - R(\hat u) > t \,\middle|\, \hat u \right) \le \exp\left(-\frac{2Nt^2}{(b-a)^2}\right). $$
This completes the proof.
\end{proof}

\subsection{Proof of Theorem~\ref{thm:empirical_test_pac_appendix}}
\begin{proof}
The i.i.d.\ test sample is independent of the calibration data and UCB randomization, so Lemma~\ref{lem:hoeffding} applies.
Use the inclusion
\[
\{\widehat R(\hat u)>\epsilon+t\}\subseteq\{R(\hat u)>\epsilon\}\cup\{\widehat R(\hat u)-R(\hat u)>t\}
\]
and take probabilities.
Combine the result of Theorem~\ref{thm:pac_guarantee} (i.e., $\mathbb{P}(R(\hat u)>\epsilon)\le\alpha$) with the law of total probability and Lemma~\ref{lem:hoeffding} to conclude the claim:
\begin{align*}
\mathbb{P}(\widehat R(\hat u)>\epsilon+t) &\le \mathbb{P}(R(\hat u)>\epsilon) + \mathbb{P}(\widehat R(\hat u)-R(\hat u)>t) \\
& = \mathbb{P}(R(\hat u)>\epsilon) + \mathbb{E}\big[\mathbb{P}\big(\widehat R(\hat u)-R(\hat u)>t\,\big|\,\hat u\big)\big] \\
& \le \alpha + \exp\!\Big(-\tfrac{2 N t^2}{(b-a)^2}\Big).
\end{align*}
This is equivalent to the stated guarantee.
\end{proof}

\section{Transductive PAC Reasoning}
\label{sec:transductive_pac_reasoning}

The transductive target is the average loss on a fixed pool, not the population risk on future inputs.
Condition on the finite collection
\[
\mathcal F=\{(x_i,U_i,\tilde y_i,y_i)\}_{i=1}^n,
\qquad
\tilde y_i=\tilde f(x_i),\quad y_i=f(x_i).
\]
The expert answers may be unknown to the algorithm, but they are fixed values for this analysis.
If generation is random, fix the generated outputs or their random seeds as part of $\mathcal F$.
Let $A$ denote the randomness from index sampling, expert-query decisions, and any further UCB randomization.
For the same strict-threshold routing rule as Equation~\eqref{eq:composite_model}, define
\[
L_{\mathcal F}(u)=\frac1n\sum_{i=1}^n
\ell(y_i,\tilde y_i)\mathbf1\{U_i<u\}.
\]
The loss is nonnegative, so $L_{\mathcal F}(u)$ is non-decreasing and $L_{\mathcal F}(0)=0$.

Choose a finite ordered grid $\Lambda_{\mathcal F}=\{0<u_1<\cdots<u_K\le1\}$ and fix its order before the random sampling and expert queries used for calibration.
Unlike the inductive grid, this grid may depend on the known scores of the fixed pool.
What matters is that it is fixed conditional on $\mathcal F$, and is not chosen or reordered from the realized query losses or UCB values.
For each candidate, require the design-conditional pointwise bound
\begin{equation}\label{eq:transductive_pointwise}
\mathbb P_A\!\left(L_{\mathcal F}(u_k)\le\widehat L_{u_k}(\alpha)\mid\mathcal F\right)\ge1-\alpha,
\qquad k=1,\ldots,K.
\end{equation}

For a fixed draw count $m$, use uniform index sampling with replacement and fresh independent Bernoulli query decisions with fixed probabilities $\pi_i>0$ in Algorithms~\ref{alg:compute_bound} and~\ref{alg:compute_bound_finite_sample}.
The weighted draw
\[
Z_j(u)=\frac{\xi_j}{\pi_{i_j}}\ell(y_{i_j},\tilde y_{i_j})\mathbf1\{U_{i_j}<u\}
\]
then has conditional mean $L_{\mathcal F}(u)$.
Cached expert answers may be reused, but the Bernoulli decision is redrawn at every index draw.
If $0\le\ell\le B$, these draws lie in $[0,B/\min_i\pi_i]$.
Writing $C=B/\min_i\pi_i$, Hoeffding's inequality~\citep{hoeffding1994probability} gives
\[
\mathbb P_A\!\left(
L_{\mathcal F}(u)>\frac1m\sum_{j=1}^m Z_j(u)
+C\sqrt{\frac{\log(2/\alpha)}{2m}}
\,\middle|\,\mathcal F\right)\le\frac{\alpha}{2}\le\alpha.
\]
This is the bound used in Algorithm~\ref{alg:compute_bound_finite_sample} and establishes Equation~\eqref{eq:transductive_pointwise} under this sampling design.
A CLT bound gives only an asymptotic version when its conditional variance and normal-approximation conditions hold.
Algorithm~\ref{alg:transductive_pac_labeling} separates this pointwise-bound construction from the fixed-sequence selection rule.

\begin{algorithm}[ht]
\caption{Transductive PAC Reasoning}
\label{alg:transductive_pac_labeling}
\begin{algorithmic}[1]
\REQUIRE Fixed inputs and scores $\{(x_i,U_i,\tilde y_i)\}_{i=1}^n$, expert $f$, loss $\ell$, tolerance $\epsilon$, error level $\alpha$, fixed query design, grid $\Lambda_{\mathcal F}$ fixed before calibration queries, UCB procedure satisfying Equation~\eqref{eq:transductive_pointwise}
\ENSURE Outputs $\{\hat y_i\}_{i=1}^n$ and threshold $\hat u$
\STATE Draw the calibration queries and construct the pointwise UCBs using the fixed query design.
\STATE Initialize $\hat u\leftarrow u_0=0$.
\FOR{$k=1,\ldots,K$}
    \IF{$\widehat L_{u_k}(\alpha)>\epsilon$}
        \STATE \textbf{break}
    \ENDIF
    \STATE $\hat u\leftarrow u_k$.
\ENDFOR
\FOR{$i=1,\ldots,n$}
    \IF{$U_i\ge\hat u$}
        \STATE Obtain $y_i$ from the expert, reusing a cached answer if available, and set $\hat y_i\leftarrow y_i$.
    \ELSE
        \STATE $\hat y_i\leftarrow\tilde y_i$.
    \ENDIF
\ENDFOR
\STATE \textbf{Return} $\{\hat y_i\}_{i=1}^n$ and $\hat u$.
\end{algorithmic}
\end{algorithm}

\begin{theorem}[Transductive PAC guarantee]\label{thm:transductive_pac}
For a fixed pool $\mathcal F$ and a grid fixed before its calibration queries, suppose Equation~\eqref{eq:transductive_pointwise} holds at every candidate.
Then Algorithm~\ref{alg:transductive_pac_labeling}, with its first-failure stopping rule and expert-only fallback, satisfies
\[
\mathbb P_A\!\left(
\frac1n\sum_{i=1}^n\ell(y_i,\hat y_i)
=L_{\mathcal F}(\hat u)\le\epsilon
\,\middle|\,\mathcal F\right)\ge1-\alpha.
\]
\end{theorem}
\begin{proof}
Conditional on $\mathcal F$, the risks and the candidate sequence are fixed.
If every candidate has loss at most $\epsilon$, all possible outputs, including $u_0=0$, are safe.
Otherwise, let $k^*=\min\{k:L_{\mathcal F}(u_k)>\epsilon\}$.
Selecting an unsafe candidate requires reaching an index at least $k^*$ and hence passing the test at $k^*$.
Therefore,
\[
\begin{aligned}
\mathbb P_A\big(L_{\mathcal F}(\hat u)>\epsilon\mid\mathcal F\big)
&\le
\mathbb P_A\big(\widehat L_{u_{k^*}}(\alpha)\le\epsilon\mid\mathcal F\big)\\
&\le
\mathbb P_A\big(\widehat L_{u_{k^*}}(\alpha)<L_{\mathcal F}(u_{k^*})\mid\mathcal F\big)
\le\alpha.
\end{aligned}
\]
The last inequality is Equation~\eqref{eq:transductive_pointwise} at the single fixed first unsafe candidate.
The final expert queries only supply the outputs required by the selected policy; they do not change the fixed pool or threshold.
The loss identity follows from $\ell(y_i,y_i)=0$ and the common strict-threshold rule.
\end{proof}

\begin{remark}[Comparison with the inductive setting]
The inductive guarantee controls the population risk of the selected policy, with probability over calibration data and algorithm randomness.
An independent test sample is used to evaluate that policy, not to define its population-risk guarantee.
The transductive guarantee instead controls loss on the given fixed pool, with probability over the query design.
No i.i.d.\ assumption on that pool is needed for this conditional statement, but it does not imply risk control on a new population.
With only asymptotic conditional pointwise validity on a fixed grid, the same first-unsafe-index inclusion gives an asymptotic transductive guarantee.
\end{remark}

\section{Multi-level PAC Reasoning}\label{sec:multipac}
We extend the selection argument to a fixed sequence of multi-level policies.
Fix an expert $f$, two fast models $\tilde f_1,\tilde f_2$, and a score function $U:\mathcal X\to[0,1]$ before calibration.
For $y=f(x)$ and $0\le u_1\le u_2\le1$, define the policy
\[
T_{u_1,u_2}(x)=
\begin{cases}
\tilde f_1(x), & U(x)<u_1,\\
\tilde f_2(x), & u_1\le U(x)<u_2,\\
f(x), & U(x)\ge u_2.
\end{cases}
\]
The intervals do not overlap, even when $u_1=u_2$ or a score equals a threshold.
In particular, $T_{0,0}=f$ is a safe fallback, including for inputs with score zero.
Let $d_k(x)=\ell(y,\tilde f_k(x))$ for $k=1,2$.
The population risk is
\[
R(u_1,u_2)
=\mathbb E_{x\sim P}\!\left[
d_1(x)\mathbf1\{U(x)<u_1\}
+d_2(x)\mathbf1\{u_1\le U(x)<u_2\}
\right].
\]

Ordering models by cost or average accuracy does not imply that this risk is non-decreasing in every threshold.
For example, at fixed $u_2$ and $0\le a\le b\le u_2$,
\[
R(b,u_2)-R(a,u_2)
=
\mathbb E\!\left[(d_1(x)-d_2(x))\mathbf1\{a\le U(x)<b\}\right],
\]
which can have either sign.
We therefore do not rely on coordinate-wise risk monotonicity or search for a maximum over all pairs that happen to pass.

\paragraph{Fixed policy sequence}
Before observing $\mathcal D_{cal}$, fix a finite ordered list of threshold pairs
$\mathbf u^{(1)},\ldots,\mathbf u^{(J)}$, with $J\ge1$ and $0\le u_1^{(j)}\le u_2^{(j)}\le1$.
Set $T_j=T_{\mathbf u^{(j)}}$, $T_0=f$, and $r_j=\mathbb E_{x\sim P}[\ell(y,T_j(x))]$, so $r_0=0$.
Assume that the expert-relative losses of the fixed fast models are nonnegative and integrable, with $\ell(y,y)=0$.
For each fixed candidate, require a pointwise population-risk bound $B_j(\alpha)$ with
\begin{equation}\label{eq:multilevel_pointwise}
\mathbb P\big(r_j\le B_j(\alpha)\big)\ge1-\alpha,
\qquad j=1,\ldots,J,
\end{equation}
under the full calibration and UCB sampling law.
To estimate a candidate risk, its losses must use exactly the same tier intervals as $T_j$.
The UCB construction must satisfy Equation~\eqref{eq:multilevel_pointwise} for the population risk of each policy.

Test $B_j(\alpha)\le\epsilon$ in this fixed order, stop at the first failure, and return the preceding policy.
In symbols,
\[
\tau=\min\!\left(\{j\in\{1,\ldots,J\}:B_j(\alpha)>\epsilon\}\cup\{J+1\}\right),
\qquad
\hat j=\tau-1,\quad \hat f=T_{\hat j}.
\]
A first-candidate failure returns $T_0=f$, and passing all tests returns $T_J$.

\begin{proposition}[Fixed-sequence multi-level guarantee]\label{prop:multilevel_fixed_sequence}
For the fixed policy list above, assume $\epsilon>0$, $\alpha\in(0,1)$, and Equation~\eqref{eq:multilevel_pointwise}.
Then the first-failure rule satisfies
\[
\mathbb P\big(r_{\hat j}\le\epsilon\big)\ge1-\alpha.
\]
No monotonicity of $r_1,\ldots,r_J$ is required.
\end{proposition}
\begin{proof}
If no candidate is unsafe, the selected policy and the expert fallback are safe.
Otherwise, let $j^*=\min\{j:r_j>\epsilon\}$.
This is a fixed index because the policy list and population risks do not depend on calibration data.
If $r_{\hat j}>\epsilon$, then $\hat j\ge j^*$ by the definition of the first unsafe index, regardless of whether later risks increase or decrease.
Returning $T_{\hat j}$ requires all tests $1,\ldots,\hat j$ to pass, including the test at $j^*$.
Consequently,
\[
\{r_{\hat j}>\epsilon\}
\subseteq
\{B_{j^*}(\alpha)\le\epsilon<r_{j^*}\},
\]
whose probability is at most $\alpha$ by Equation~\eqref{eq:multilevel_pointwise}.
\end{proof}
The guarantee applies to the last accepted policy in the fixed sequence.
For a fixed finite policy list, asymptotic pointwise bounds give the corresponding asymptotic conclusion by the same event inclusion.

\paragraph{Extension to more tiers}
Let $M\ge2$ be the number of models, with fast models $\tilde f_1,\ldots,\tilde f_{M-1}$ and expert $f$.
For thresholds $0\le u_1\le\cdots\le u_{M-1}\le1$, use
\[
T_{\mathbf u}(x)=
\begin{cases}
\tilde f_1(x), & U(x)<u_1,\\
\tilde f_k(x), & u_{k-1}\le U(x)<u_k,\quad k=2,\ldots,M-1,\\
f(x), & U(x)\ge u_{M-1}.
\end{cases}
\]
With $d_k(x)=\ell(y,\tilde f_k(x))$, its risk is
\[
R(\mathbf u)=\mathbb E\!\left[
d_1(x)\mathbf1\{U(x)<u_1\}
+\sum_{k=2}^{M-1}d_k(x)\mathbf1\{u_{k-1}\le U(x)<u_k\}
\right].
\]
An all-zero threshold vector is expert-only.
Fix a finite ordered list of such vectors before calibration and require a pointwise population-risk UCB for each associated policy.
Proposition~\ref{prop:multilevel_fixed_sequence} then gives the risk guarantee for the selected policy.

\section{Experimental details}\label{app:experimental_details}

\paragraph{Hyperparameter settings of LLMs} In this study, we configure the decoding parameters as follows: for Qwen/Qwen3-4B-Instruct-2507, we set \textit{Temperature} = 0.7, \textit{TopP} = 0.8, \textit{TopK} = 20, and \textit{MinP} = 0; for Qwen/Qwen3-4B-Thinking-2507, we set \textit{Temperature} = 0.6, \textit{TopP} = 0.95, \textit{TopK} = 20, and \textit{MinP} = 0. Experiments were run on one NVIDIA RTX A6000 Graphics Card.

\paragraph{The prompt for verbalized uncertainty score} In Table~\ref{tab:verbalized_score_prompt}, we present the prompt used to elicit the verbalized confidence scores. After ten trials, we obtained the average confidence score and defined the verbalized uncertainty score as $1$ minus this average confidence.

\begin{table}
\caption{Prompt for the verbalized confidence scores.}\label{tab:verbalized_score_prompt}
\rowcolors{1}{Gray}{Gray}
\begin{tabular}{!{\vrule width 1.2pt} p{\linewidth}!{\vrule width 1.2pt}}
    \Xhline{1.2pt}
        \textbf{System prompt:} You are a reasoning assistant. For each question and proposed answer, you must estimate how likely the proposed answer is correct.\\

        \textbf{User prompt:}\\
        Question: \{QUESTION\}\\
        Answer: \{ANSWER\}\\
        Provide a probability (between 0.0 and 1.0) that your answer is correct. Only output the probability.\\
    \Xhline{1.2pt}
\end{tabular}
\end{table}

\paragraph{Details of Datasets} Table~\ref{table:splitting_settings} summarizes the datasets employed in our experiments, together with their corresponding splitting strategies. For each dataset, we report its type, overall size, and the partitioning into PAC calibration and PAC test sets.

\begin{table}[ht]
\centering
\caption{The details of datasets and splitting settings for PAC experiments}\label{table:splitting_settings}
\begin{tabular}{l|l|l|l|r}
\hline
\textbf{Dataset}& \textbf{Dataset Type} &\textbf{Dataset Size} & \textbf{Split Setting} & \textbf{Size} \\
\hline
\multirow{2}{*}{MATH-500} & \multirow{2}{*}{Math Reasoning} &\multirow{2}{*}{\centering 500}& PAC Calibration & 300 \\
     & & & PAC Test & 200 \\
\hline
\multirow{2}{*}{ZebraLogic} & \multirow{2}{*}{Text reasoning} &\multirow{2}{*}{\centering 1000}& PAC Calibration & 500 \\
    & & & PAC Test & 500 \\
\hline
\multirow{2}{*}{Arena-Hard} & \multirow{2}{*}{Alignment Task} &\multirow{2}{*}{\centering 750}& PAC Calibration & 450 \\
     & & & PAC Test & 300 \\
\hline
\end{tabular}
\end{table}

\subsection{Baselines}
\label{sec:baseline}
\begingroup
We compare empirical-mean calibration, learned routing~\citep{2025llmrouter}, prompting-based efficient reasoning~\citep{xu2025chain}, and reasoning-free generation~\citep{ma2025reasoning}.
The following definitions specify the implementations used in our experiments.
\endgroup

\paragraph{Naive threshold}
A simple baseline that routes each query to the thinking model whenever its uncertainty score $U_i$ exceeds a fixed threshold, and to the non-thinking model otherwise. We set the threshold to $0.05$ in our experiments.

\paragraph{Naive control}
{This baseline calibrates the routing threshold using the empirical mean loss.}
Specifically, given a target error tolerance $\epsilon$, this baseline selects the largest threshold $u$ such that the empirical loss
\[
L(u)=\frac{1}{n}\sum_{i=1}^{n}\ell(y_i,\tilde{y}_i)\mathbf{1}\{U_i\le u\}
\]
is at most $\epsilon$ on the calibration set.
{It compares the empirical loss directly with $\epsilon$, while PAC Reasoning tests a pointwise UCB in the fixed order of Algorithm~\ref{alg:pac_reasoning}.}

\paragraph{Router}
We train a learned router for the model pair using the open-source library \textsc{LLMRouter}~\citep{2025llmrouter}.
Specifically, we sample 2,000 instances from eight widely used benchmarks, including ARC-Challenge~\citep{clark2018think}, CommonsenseQA~\citep{talmor2019commonsenseqa}, GSM8K~\citep{cobbe2021training}, MATH~\citep{hendrycks2021measuring}, HumanEval~\citep{chen2021evaluating}, MMLU~\citep{hendrycks2021measuringa}, NaturalQA~\citep{kwiatkowski2019natural}, and TriviaQA~\citep{joshi2017triviaqa}.
The router is implemented as a lightweight classifier that takes the input prompt as features and outputs a routing score, representing the predicted probability that invoking the thinking model is necessary.
At inference time, inputs are routed to the thinking model if the predicted probability exceeds a fixed threshold of $0.5$, and otherwise handled by the non-thinking model.
Such a router does not provide explicit guarantees on the resulting performance loss after routing.

\paragraph{Chain of Draft}
Chain of Draft (CoD)~\citep{xu2025chain} is a prompting-based method designed to improve reasoning efficiency by explicitly constraining the verbosity of intermediate reasoning steps.
Instead of generating full chain-of-thoughts, CoD enforces a compact ``draft'' for each reasoning step, typically limited to a few words, thereby reducing token usage while preserving a minimal reasoning structure.
The specific prompt template used in our experiments is shown in Table~\ref{tab:cod_prompt}.
As a prompting strategy, CoD focuses on reducing reasoning length, but does not control the performance loss induced by truncated reasoning.

\begin{table}[htbp]
\caption{Prompt for Chain of Draft.}\label{tab:cod_prompt}
\rowcolors{1}{Gray}{Gray}
\begin{tabular}{!{\vrule width 1.2pt} p{\linewidth}!{\vrule width 1.2pt}}
    \Xhline{1.2pt}
        \textbf{System prompt:} Think step by step, but only keep a minimum draft for each thinking step, with at most five words. Finally, put your final answer within \texttt{\textbackslash boxed\{\}}.

        \textbf{User prompt:}\\
        Question: \{QUESTION\}\\
    \Xhline{1.2pt}
\end{tabular}
\end{table}

\paragraph{NoThinking}
NoThinking~\citep{ma2025reasoning} eliminates explicit chain-of-thought reasoning by removing special reasoning markers (e.g., \texttt{<think>}) and relying on direct answer generation combined with simple aggregation strategies.
The method shows that, for certain reasoning tasks, models can achieve competitive performance without explicitly generating intermediate reasoning steps, leading to reduced inference cost and latency.
However, NoThinking does not incorporate uncertainty-aware decision mechanisms, nor does it provide explicit control or theoretical guarantees over the resulting performance degradation.

\subsection{Choice of loss functions}
\label{sec:choice_loss_functions}
The loss functions, shown as in~\Cref{eq:semantic_loss} and \Cref{eq:binary_loss}, serve distinct purposes in evaluating the PAC Reasoning.
The semantic cosine distance captures the degree of semantic alignment between the PAC Reasoning's prediction and reference outputs.
It is particularly suitable for tasks where nuanced differences in meaning are critical, such as natural language understanding or generation tasks.
By using the ``Qwen/Qwen3-Embedding4B'' model, we ensure that the embeddings capture rich contextual information and compare semantic content in high-dimensional spaces.
In contrast, the binary 0–1 loss is designed for scenarios where the correctness of the generated answer is verifiable,
such as in mathematical problem-solving or multiple-choice question answering.
This loss function is effective for evaluating the framework's ability to produce exact matches to ground-truth answers, emphasizing precision in verifiable tasks.
By testing PAC Reasoning on these two loss functions, we can assess its semantic quality and factual accuracy across diverse tasks.

\section{Additional experiments}
\label{sec:extensive_study}

\subsection{Experimental results of Llama-based LLMs}
\label{app:exp_other_llm}
In this section, we evaluate PAC Reasoning on additional LLM architectures and larger-scale models to further verify the generalizability of our framework.
Specifically, we conduct experiments using Llama-3.1-8B–based models: the ``DeepSeek-R1-Distill-Llama-8B'' as the thinking model and ``Llama-3.1-8B-Instruct'' as the lower-performance non-thinking model.
We configure the decoding parameters as follows: for Llama-3.1-8B-Instruct, we set \textit{Temperature} = 0.6, \textit{TopP} = 0.95, \textit{TopK} = 20, \textit{MinP} = 0, and \textit{max\_tokens} = 4096.
For DeepSeek-R1-Distill-Llama-8B, we set \textit{Temperature} = 0.6, \textit{TopP} = 0.95, \textit{TopK} = 20, and \textit{MinP} = 0.
Experiments were run on one NVIDIA RTX A6000 Graphics Card.
Other experimental details follow Appendix~\ref{app:experimental_details}.

Figure~\ref{fig:llama_results} summarizes the performance of PAC Reasoning on Llama-3.1-8B–based models.
Across all three benchmarks, PAC Reasoning consistently maintains valid performance loss control, with empirical performance loss staying below the diagonal reference line.
For uncertainty estimation, the logits-based score exhibits tighter calibration and lower ECP than the verbalized score, particularly under smaller $\epsilon$ values, while the verbalized score shows slightly higher variance but still adheres to theoretical bounds.
In terms of efficiency, both scores achieve substantial STP, demonstrating that PAC Reasoning can reliably identify confident cases and reduce unnecessary calls to the thinking model.
Overall, the results confirm that PAC Reasoning generalizes well to larger LLM architectures and continues to deliver stable performance loss control and efficiency gains.

\begin{figure}[t]
    \centering
    \subcaptionbox{MATH-500}{
        \includegraphics[width=0.33\linewidth]{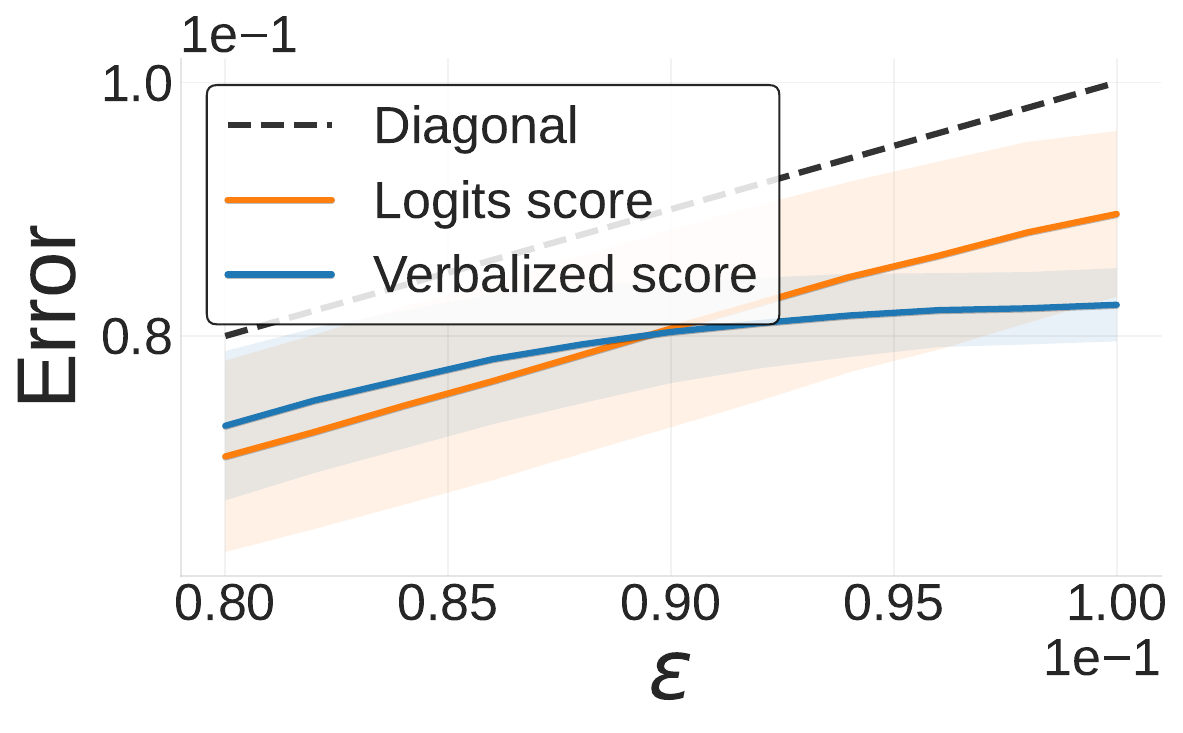}
         \includegraphics[width=0.33\linewidth]{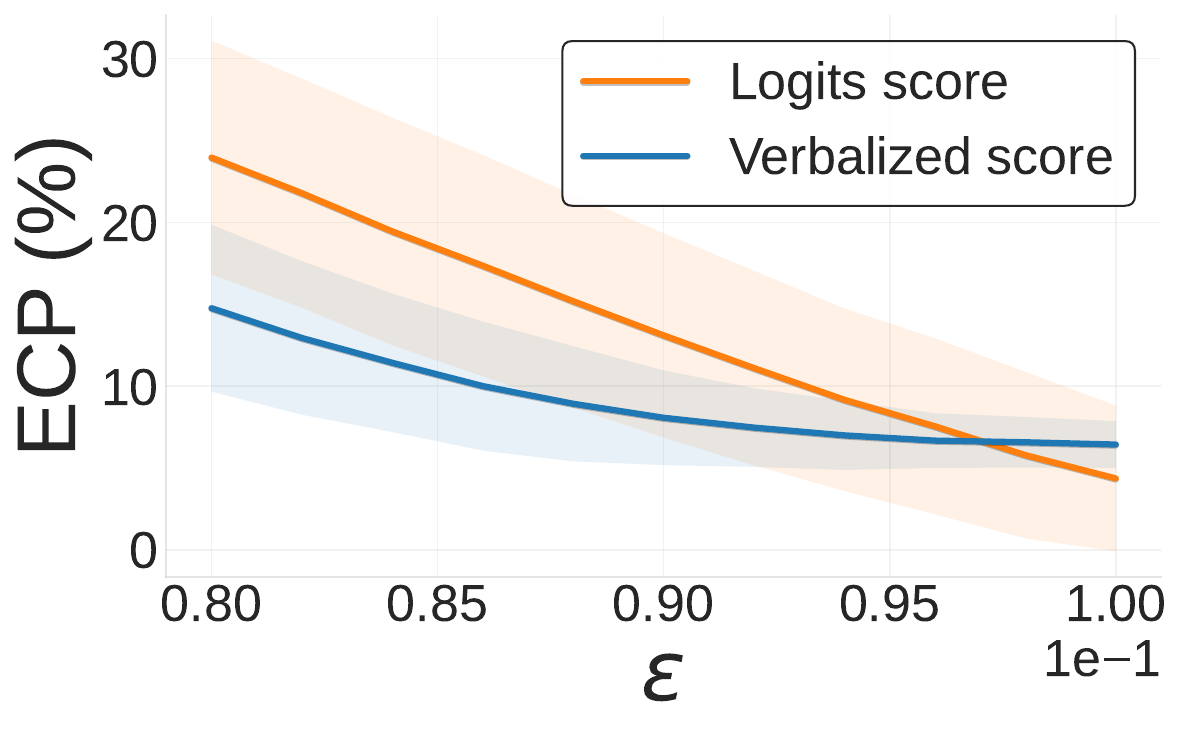}
        \includegraphics[width=0.33\linewidth]{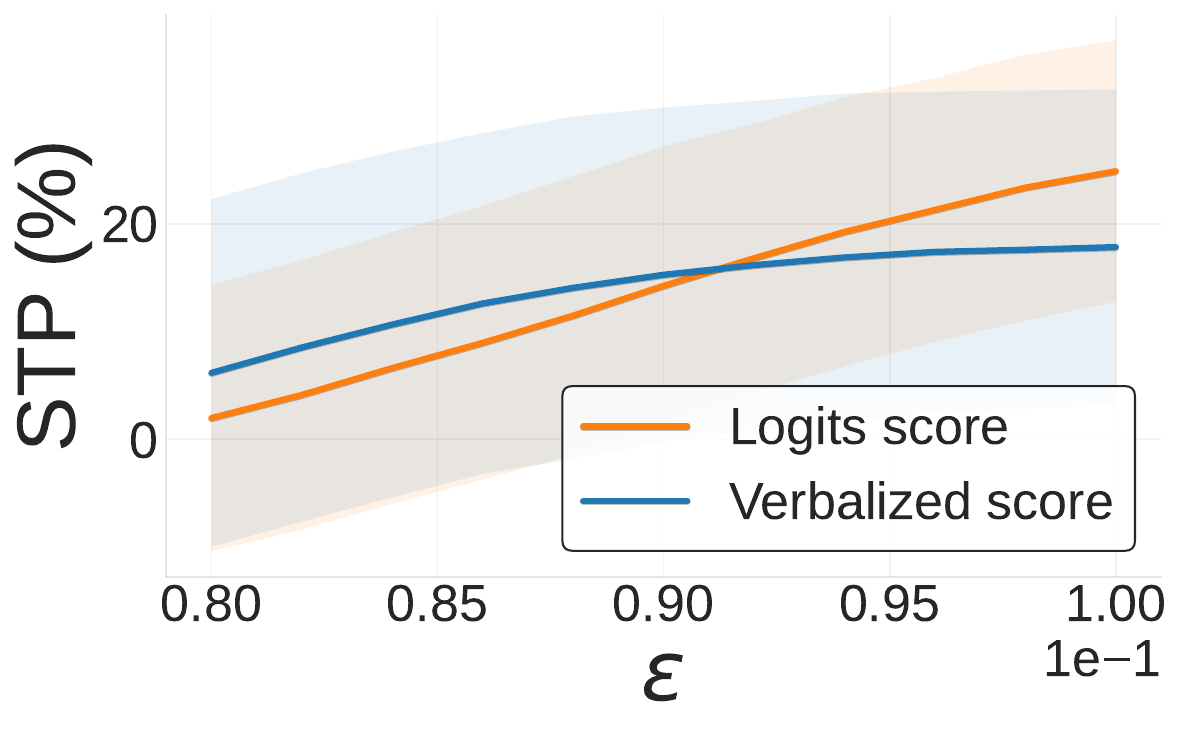}
    }
    \\
    \subcaptionbox{ZebraLogic}{
        \includegraphics[width=0.33\linewidth]{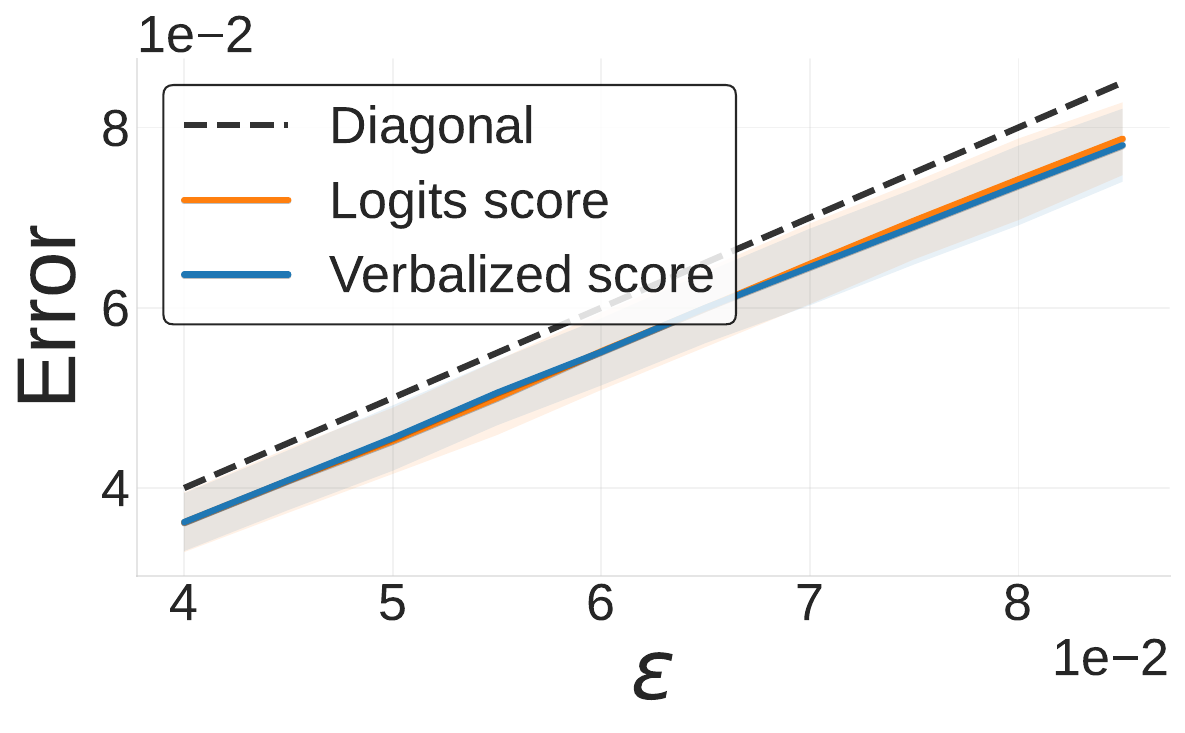}
        \includegraphics[width=0.33\linewidth]{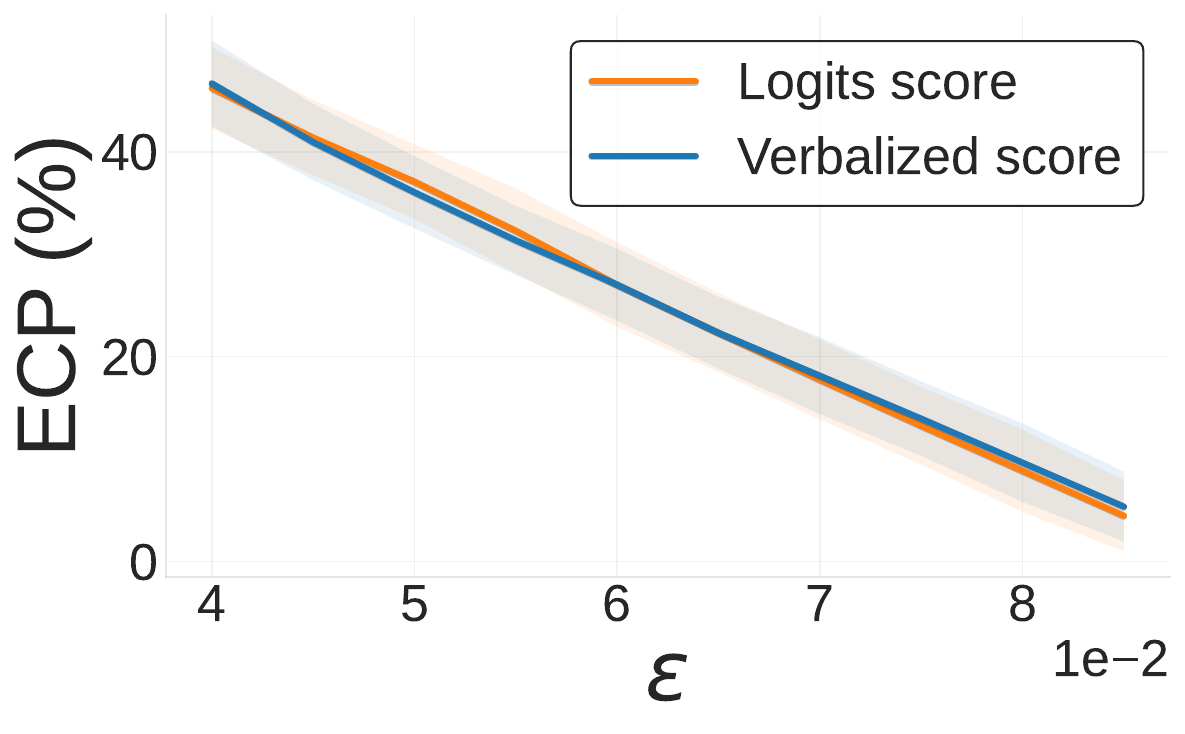}
        \includegraphics[width=0.33\linewidth]{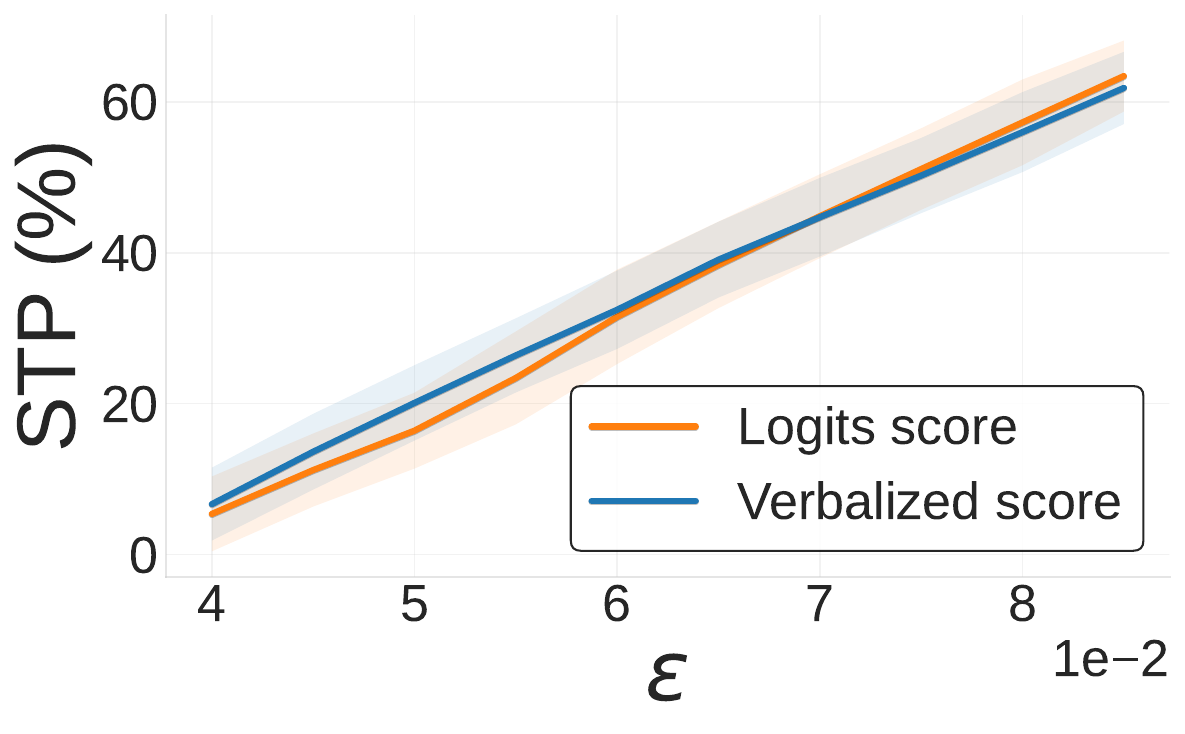}
    }
        \\
    \subcaptionbox{Arena-Hard}{
        \includegraphics[width=0.33\linewidth]{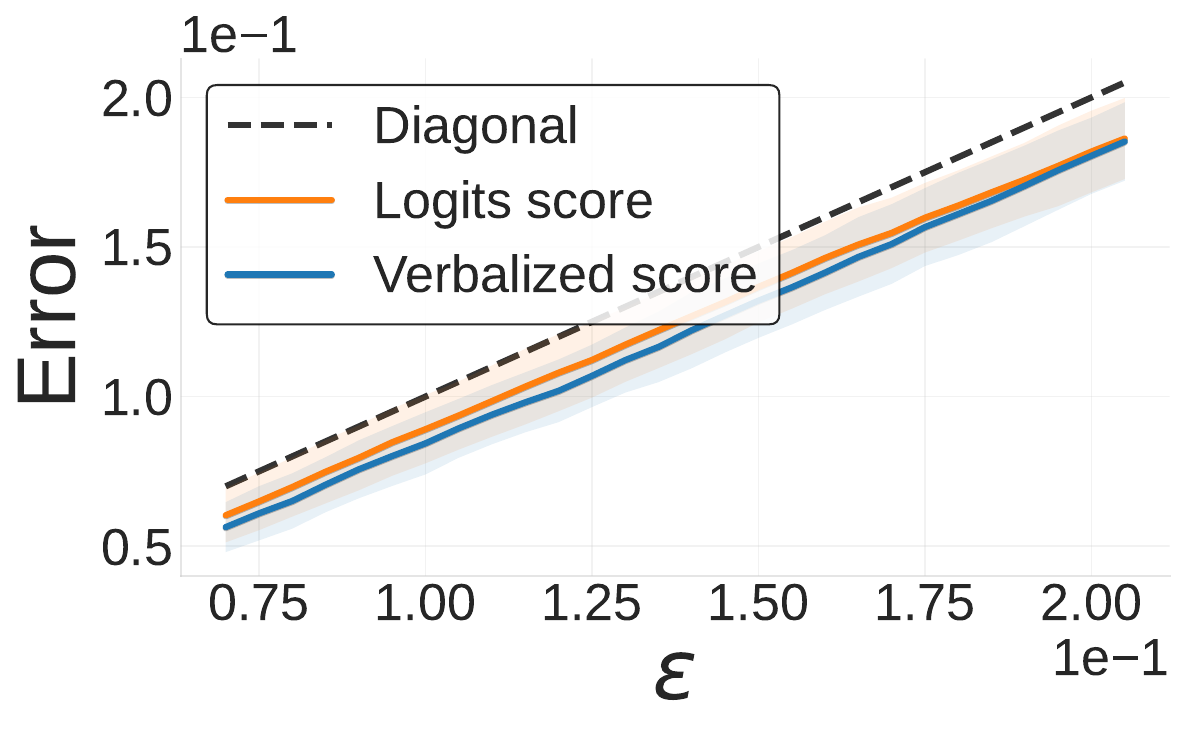}
        \includegraphics[width=0.33\linewidth]{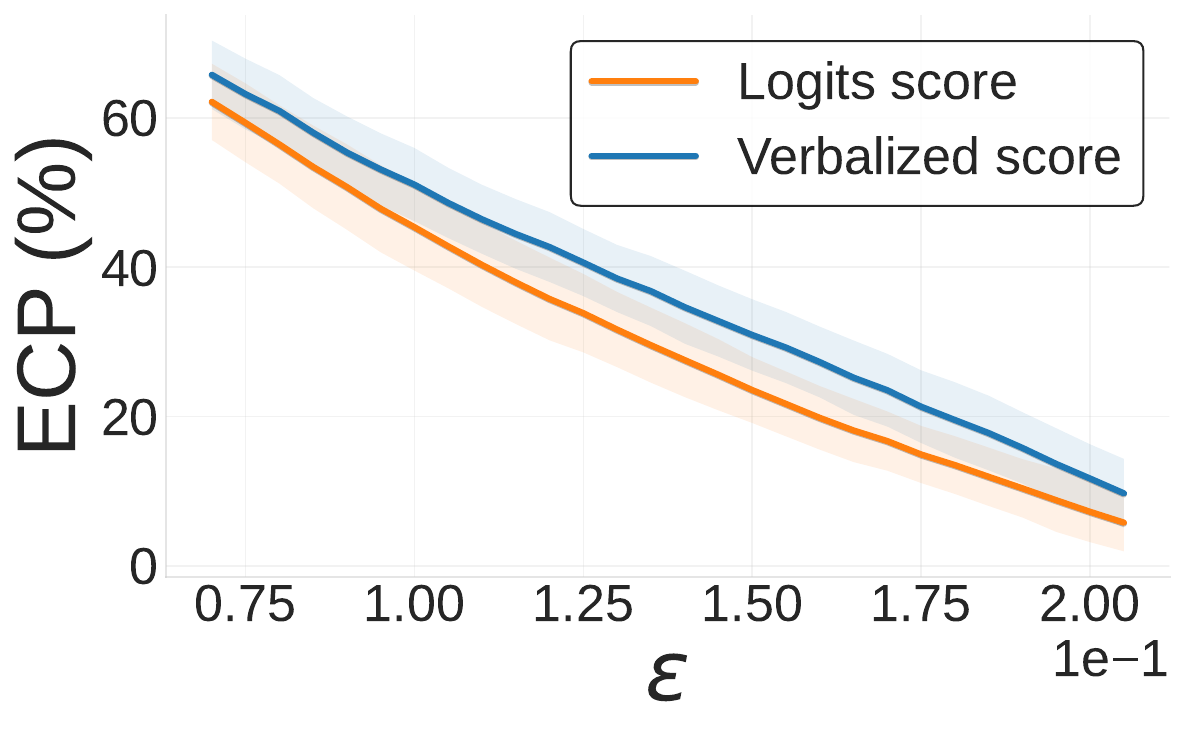}
        \includegraphics[width=0.33\linewidth]{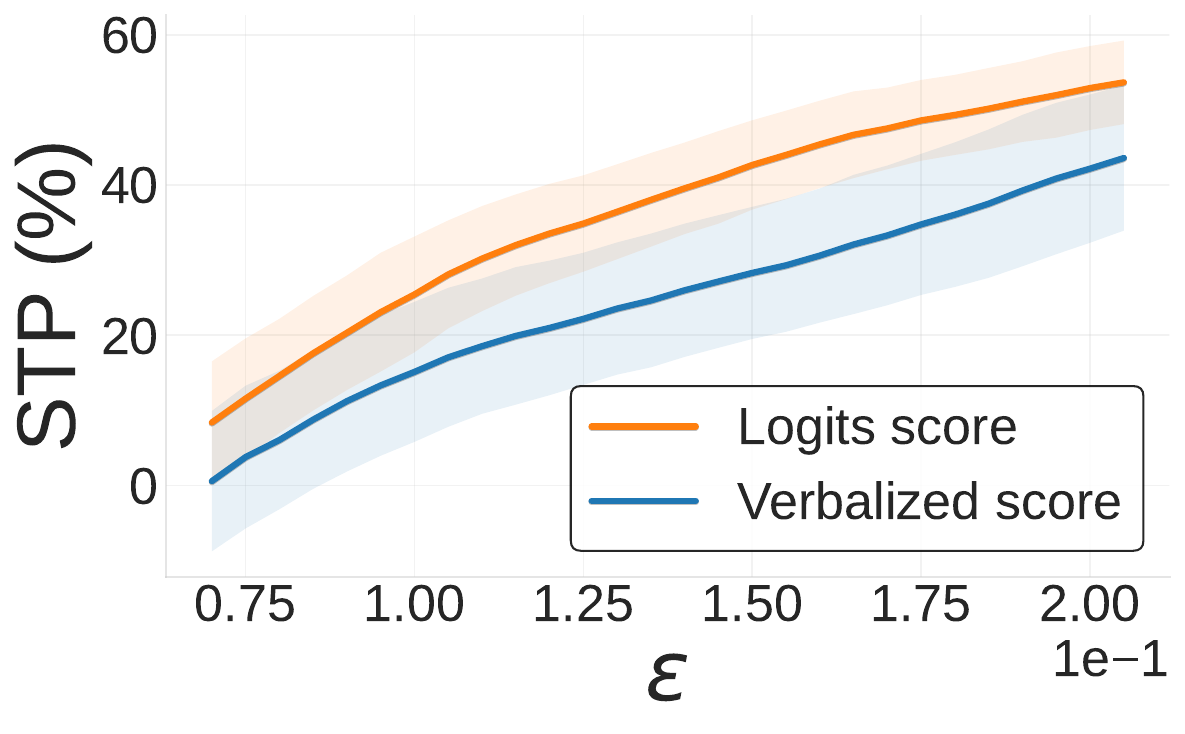}
    }
    \caption{\textbf{Error control, ECP and STP of PAC Reasoning} for semantic loss across three benchmarks under a confidence level of $\alpha = 0.05$, on the Llama-3.1-8B–based models.
    Uncertainty score includes the logits-based score and the verbalized score.
    All experiments are repeated 100 times, and the shaded areas represent standard deviations.}
    \label{fig:llama_results}
\end{figure}

\begin{figure}[htbp]
    \centering
    \subcaptionbox{GPQA}{
        \includegraphics[width=0.33\linewidth]{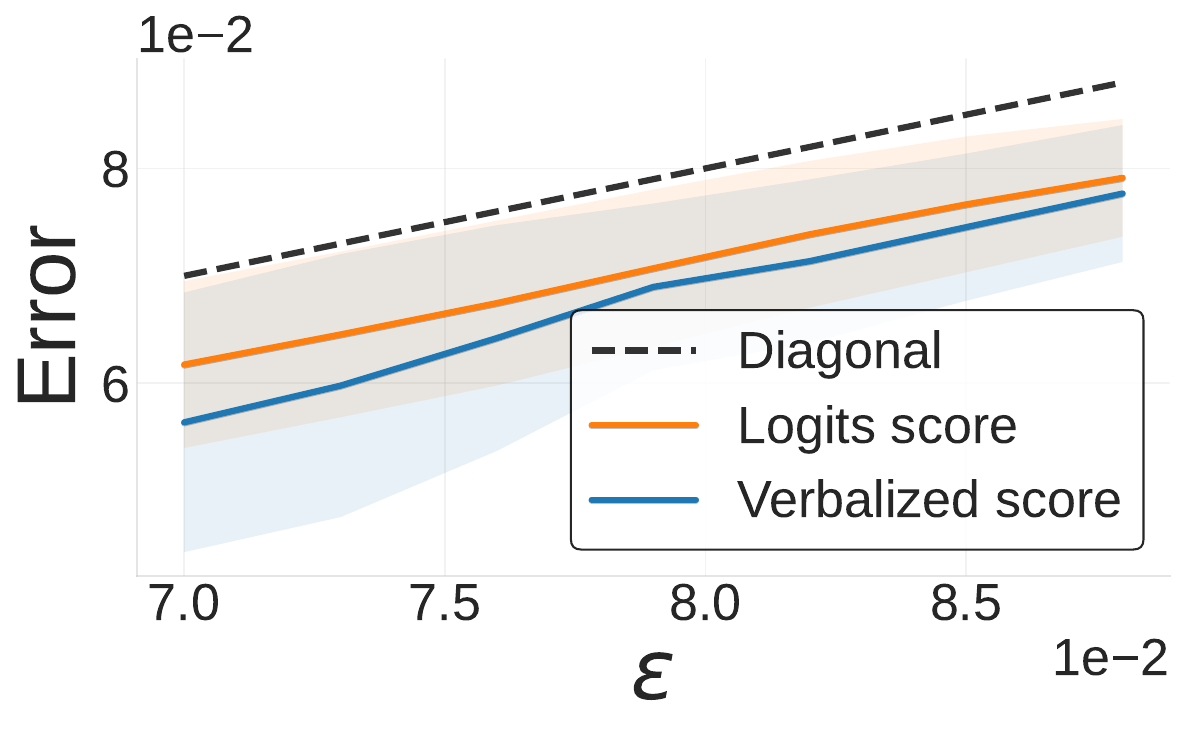}
        \includegraphics[width=0.33\linewidth]{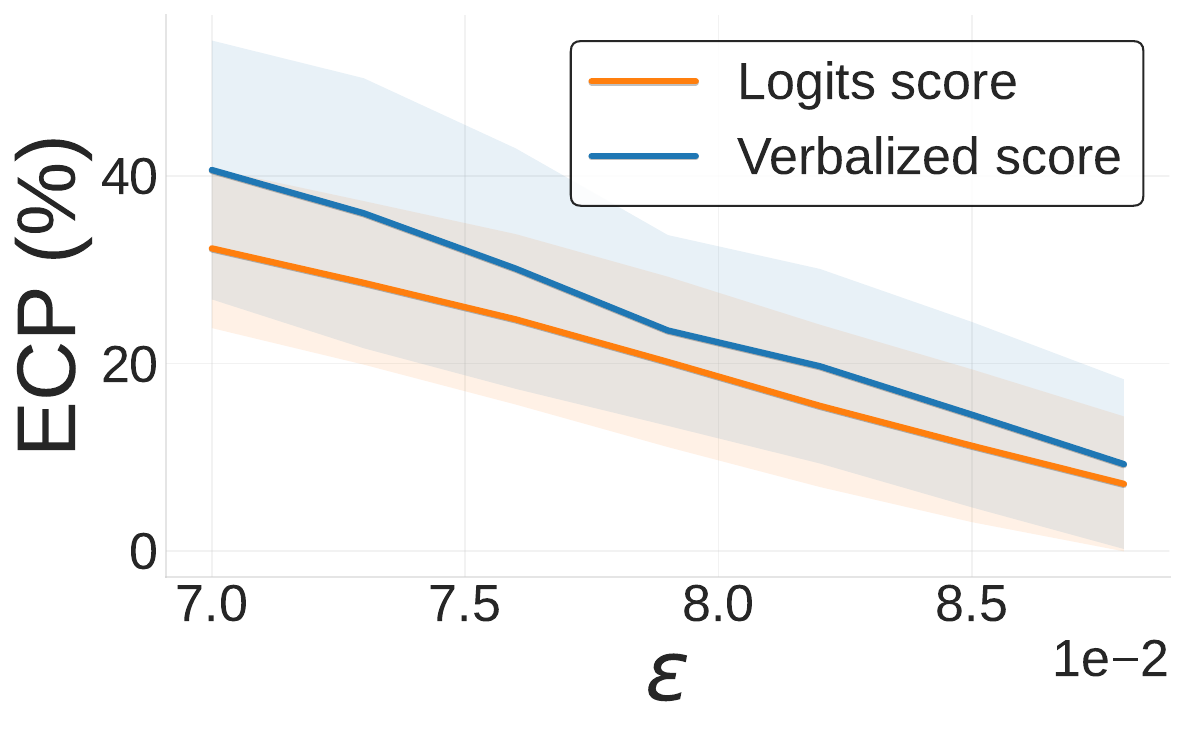}
        \includegraphics[width=0.33\linewidth]{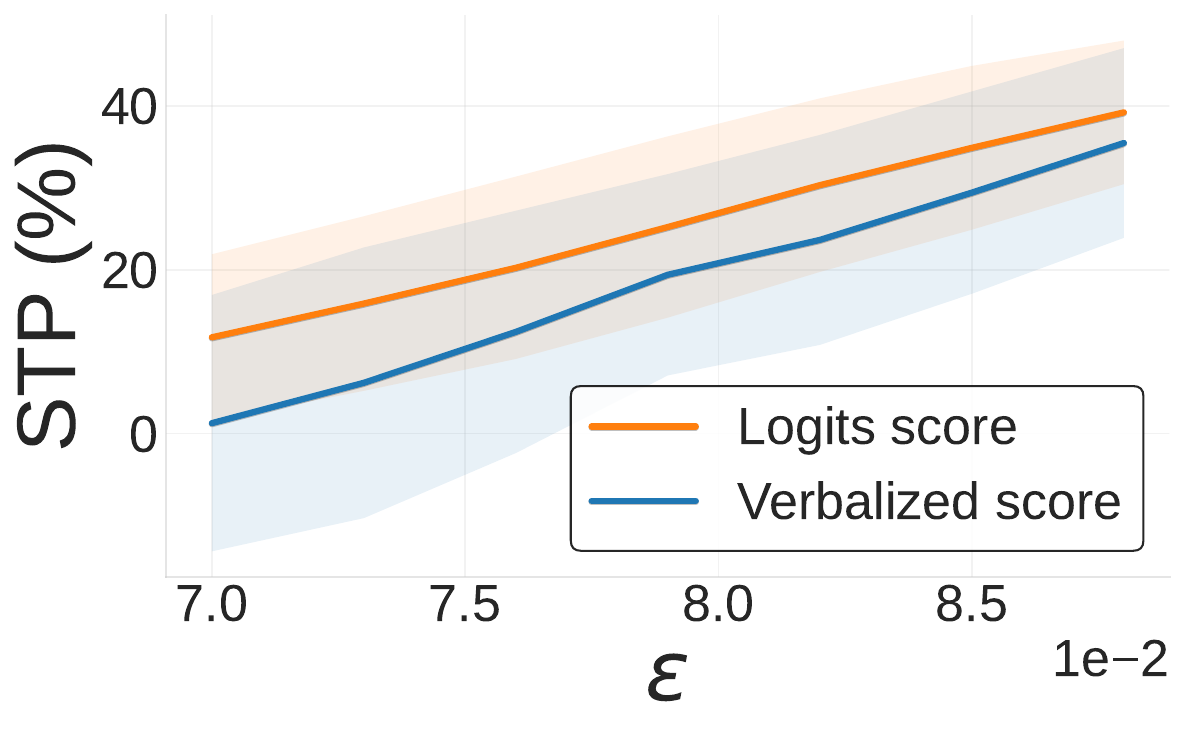}
    }\\
    \subcaptionbox{HumanEval}{
        \includegraphics[width=0.33\linewidth]{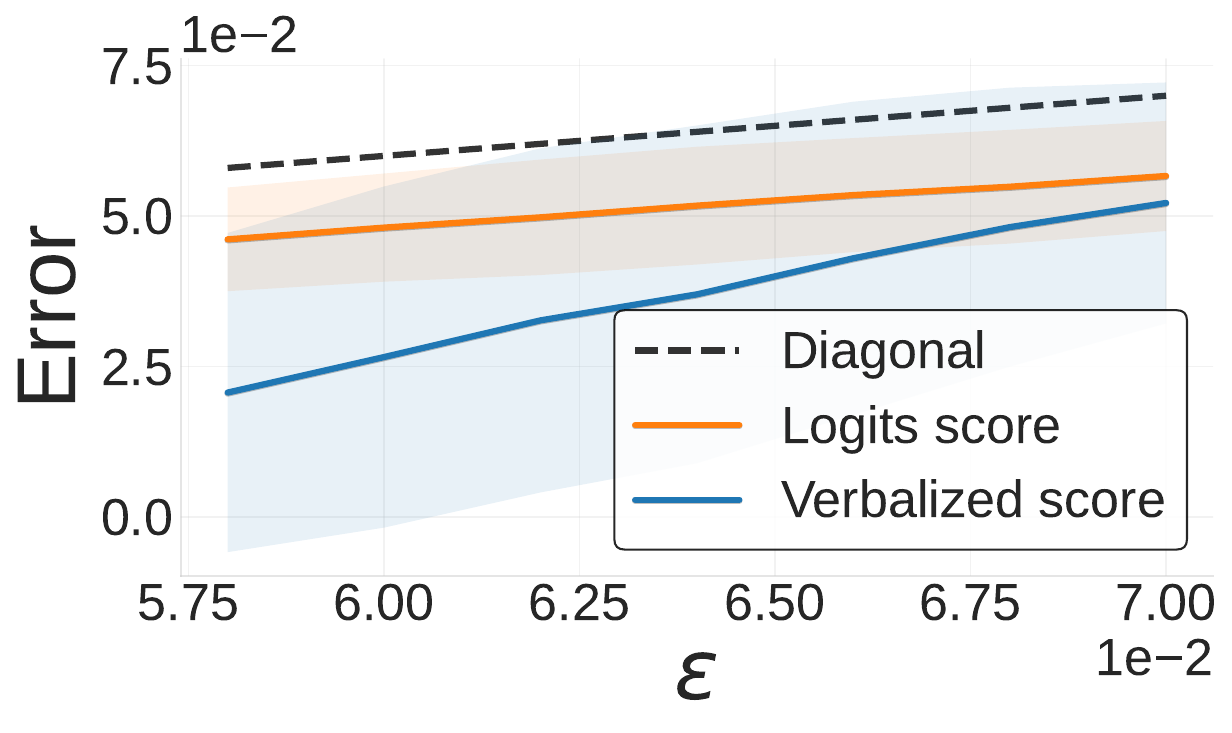}
        \includegraphics[width=0.33\linewidth]{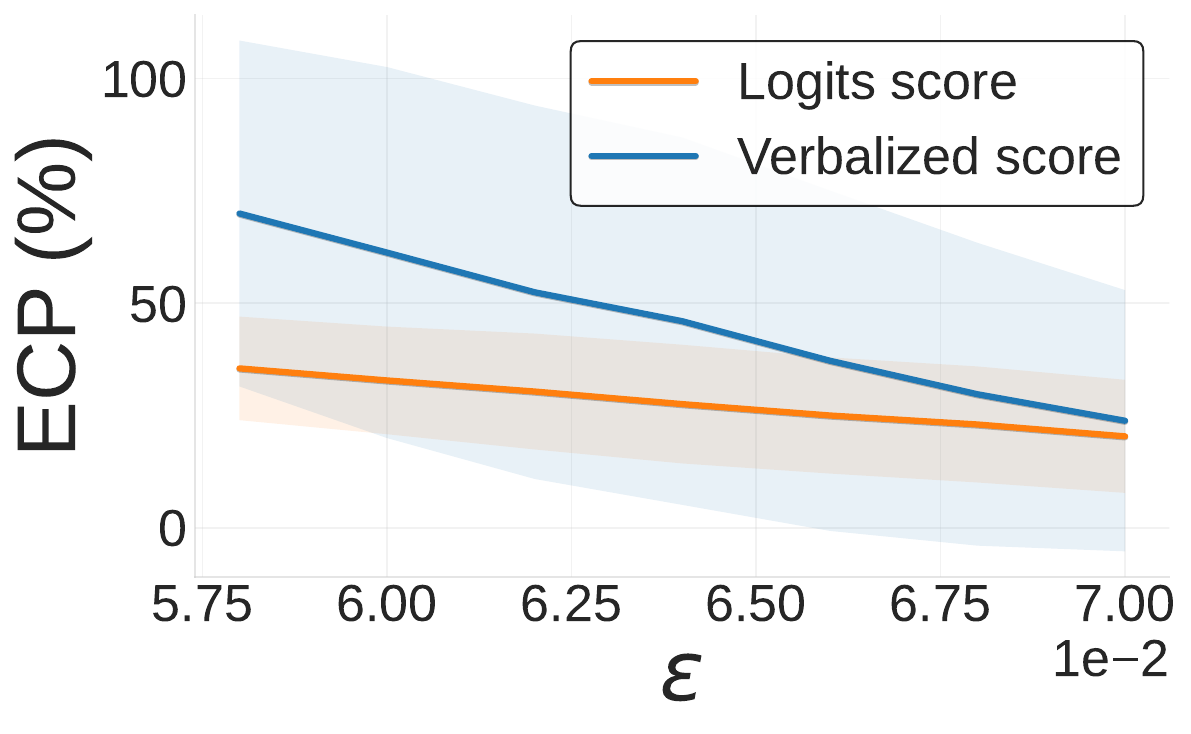}
        \includegraphics[width=0.33\linewidth]{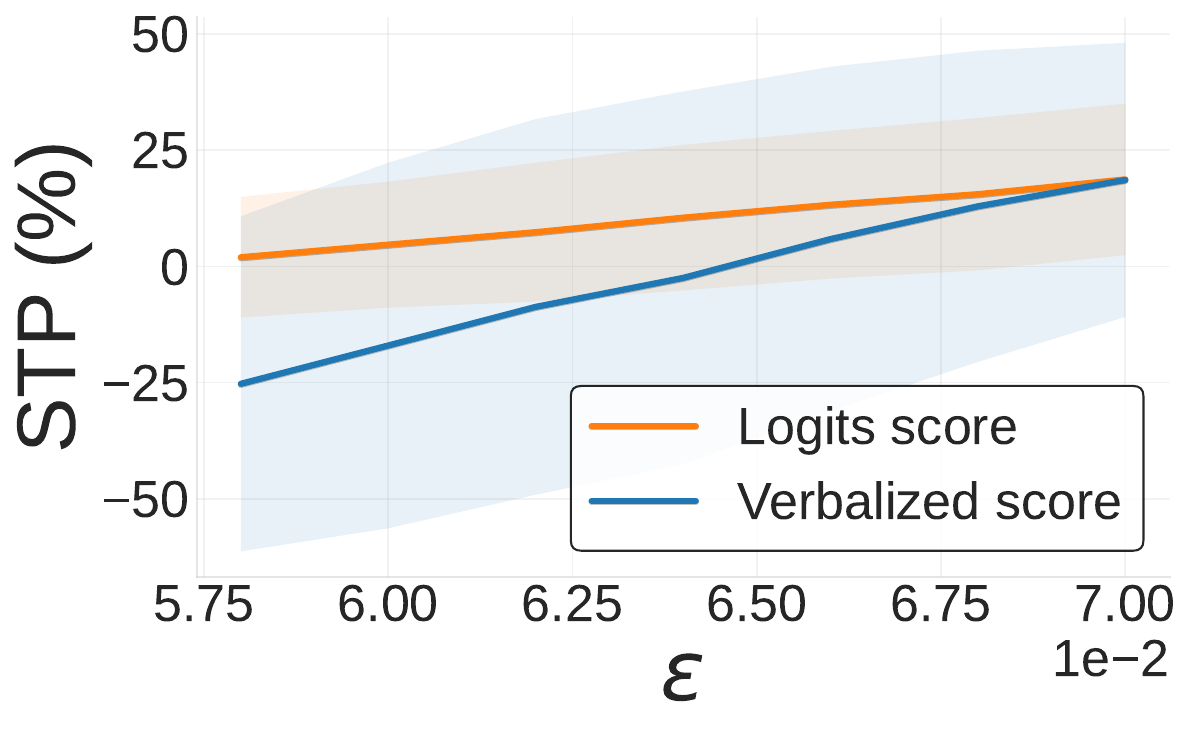}
    }
    \caption{Error control, ECP and STP of PAC Reasoning for semantic loss for GPQA and HumanEval at a confidence level of $\alpha = 0.05$.
    Uncertainty score includes the logits-based score and the verbalized score.
    The red dashed line $\epsilon$ denotes the target tolerance, and the shaded areas represent standard deviations.
    All experiments are repeated 100 times.}
    \label{fig:results_other_datasets}
\end{figure}

\subsection{Additional benchmarks}\label{app:more_data}
In this section, we further validate the effectiveness of our framework on additional datasets, including GPQA~\citep{rein2024gpqa} and HumanEval~\citep{chen2021evaluating}. We choose the semantic loss function to evaluate the performance loss.
The experimental results shown in Figure~\ref{fig:results_other_datasets} evaluate the performance of PAC Reasoning on these benchmarks using two types of uncertainty scores: the logits-based score and the verbalized score.
Across both datasets, the two uncertainty scores demonstrate valid error control.
The ECP decreases steadily as the tolerance level $\epsilon$ grows, with the verbalized score exhibiting slightly worse ECP performance.
For STP, the logits-based score consistently achieves higher performance compared to the verbalized score.
Overall, these observations indicate that PAC Reasoning remains effective and reliable when applied to a broader range of datasets.

\subsection{Accuracy of PAC Reasoning}
\label{sec:acc}
In this section, we investigate the effectiveness of the PAC Reasoning framework when controlled by semantic loss.
As shown in Figure~\ref{fig:semantic_acc}, applying semantic loss to regulate the PAC filtering process leads to consistently improved final accuracies across both MATH-500 and ZebraLogic.
Both logits-based and verbalized uncertainty scores yield higher Pass@1 performance than the baseline non-thinking model, with the verbalized score performing the best.
The results indicate that PAC Reasoning controlled by semantic loss reliably improves output accuracy while remaining robust to different $\epsilon$ settings.

\begin{figure}[htbp]
    \centering
    \subcaptionbox{MATH-500}{
        \includegraphics[width=0.48\linewidth]{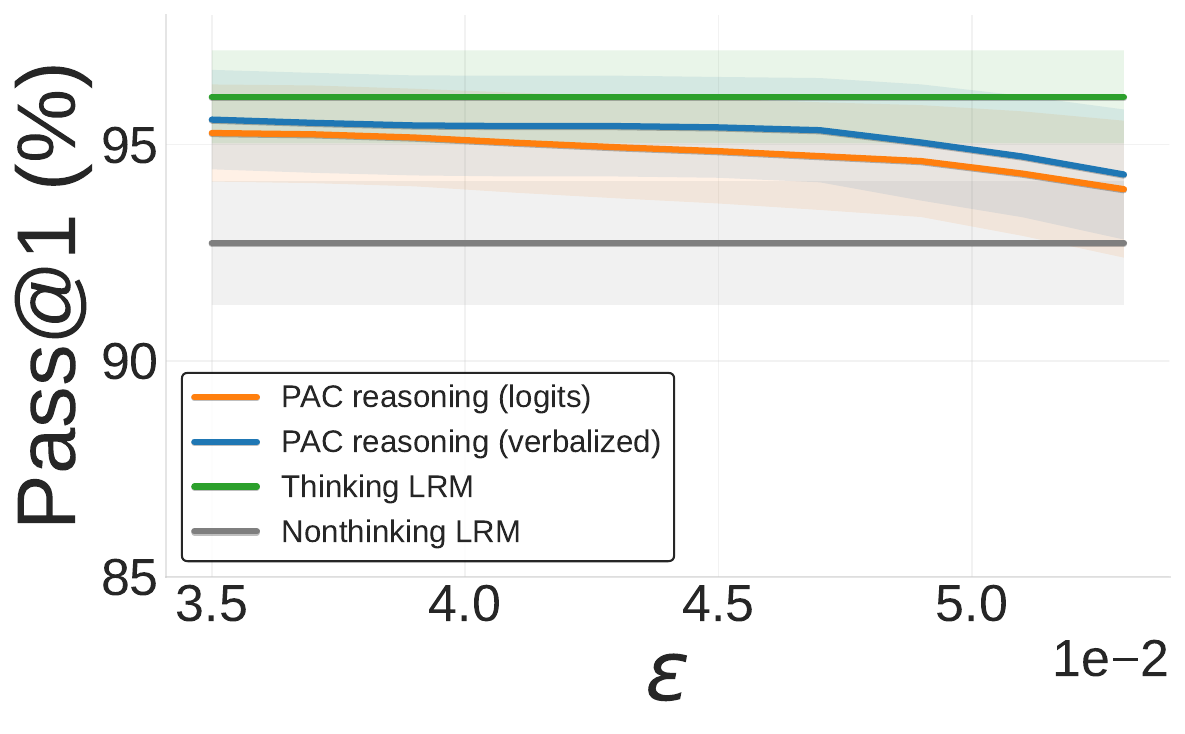}}
     \subcaptionbox{ZebraLogic}{    \includegraphics[width=0.48\linewidth]{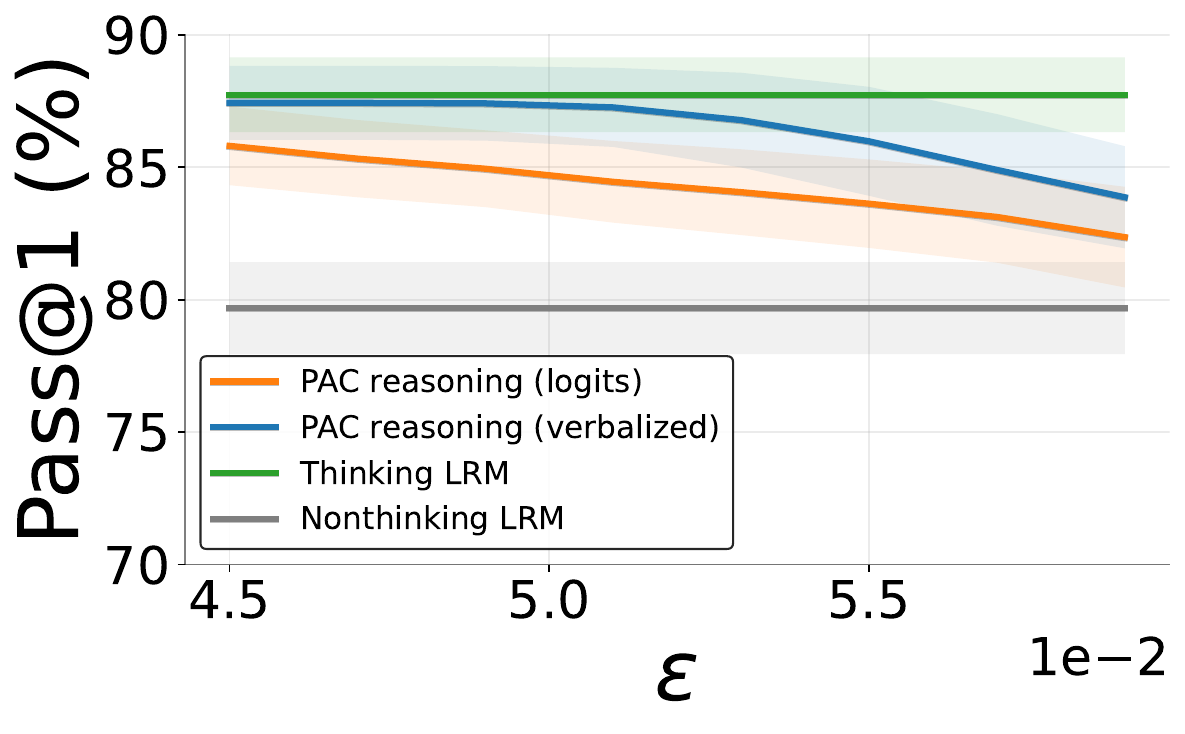}
    }
    \caption{Accuracy of PAC Reasoning based on semantic loss across reasoning benchmarks at a confidence level of $\alpha = 0.05$.
    Uncertainty score includes the logits-based score and the verbalized score.
    All experiments are repeated 100 times, and the shaded areas represent standard deviations.}
    \label{fig:semantic_acc}
\end{figure}

\subsection{Experimental results for the binary loss}
\label{sec:binary_result}
\begingroup
\paragraph{Empirical test-loss success rates.}\label{sec:test_loss_success}
Table~\ref{tab:test_loss_success} reports the success counts from 100 repeated runs on MATH-500 and ZebraLogic.
A run is counted as successful when its held-out empirical test loss is at most $\epsilon$.
We report the number of successful runs out of 100 and a two-sided exact $95\%$ Clopper--Pearson binomial interval for the run-level success probability.
These intervals describe success under the repeated experimental protocol on the fixed benchmark data, rather than population-risk coverage.
\endgroup

\begin{table}[t]
\centering
\small
\caption{{Empirical test-loss success counts and two-sided exact $95\%$ binomial intervals over 100 runs.
Qwen uses $\epsilon=0.03$ on MATH-500 and $\epsilon=0.08$ on ZebraLogic; the Llama pair uses $\epsilon=0.1$ on both datasets.
The intervals are in percent.}}
\label{tab:test_loss_success}
\setlength{\tabcolsep}{5pt}
\begin{tabular}{lrrrr}
\toprule
& \multicolumn{2}{c}{MATH-500} & \multicolumn{2}{c}{ZebraLogic} \\
\cmidrule(lr){2-3}\cmidrule(lr){4-5}
Method / score & Success / 100 & $95\%$ interval & Success / 100 & $95\%$ interval \\
\midrule
\multicolumn{5}{l}{\textbf{Qwen pair (logits score)}} \\
PAC Reasoning & 89 & $[81.17,94.38]$ & 85 & $[76.47,91.35]$ \\
Naive threshold & 15 & $[8.65,23.53]$ & 100 & $[96.38,100.00]$ \\
Naive control & 36 & $[26.64,46.21]$ & 0 & $[0.00,3.62]$ \\
CoD & 0 & $[0.00,3.62]$ & 100 & $[96.38,100.00]$ \\
NoThinking & 0 & $[0.00,3.62]$ & 100 & $[96.38,100.00]$ \\
Router & 16 & $[9.43,24.68]$ & 66 & $[55.85,75.18]$ \\
\midrule
\multicolumn{5}{l}{\textbf{Llama pair (PAC Reasoning)}} \\
Logits & 92 & $[84.84,96.48]$ & 91 & $[83.60,95.80]$ \\
Verbalized & 87 & $[78.80,92.89]$ & 88 & $[79.98,93.64]$ \\
\bottomrule
\end{tabular}
\end{table}

\paragraph{Wall-clock latency and dollar cost.}
We further report total wall-clock latency, token usage, and dollar cost of PAC Reasoning on MATH-500 and ZebraLogic, summed over the test set, following the experimental setup of Table~\ref{table:binary_results_more_methods} in the main manuscript.
Dollar costs are computed based on Alibaba Cloud Model Studio pricing: Qwen3-4B-Thinking-2507 costs \$1.26 per 1M output tokens and Qwen3-4B-Instruct-2507 costs \$0.42 per 1M output tokens, with both models sharing an input price of \$0.11 per 1M tokens.

As shown in Table~\ref{tab:latency_cost}, PAC Reasoning significantly reduces both wall-clock latency and dollar cost.
On MATH-500, latency is reduced from 266.4k\,s to 82.2k\,s and cost from 1.71 to 0.78, corresponding to reductions of 69.23\% and 54.46\%, respectively.
On ZebraLogic, latency decreases from 59.8k\,s to 43.8k\,s and cost from 4.64 to 1.98, yielding reductions of 26.76\% and 57.37\%.

\begin{table}[t]
\centering
\caption{Wall-clock latency, token usage, and dollar cost comparison.}
\label{tab:latency_cost}
\setlength{\tabcolsep}{2mm}{
\resizebox{\textwidth}{!}{
\begin{tabular}{l l r r r r}
\toprule
Dataset & Method & Total Latency (s) & Dollar Cost (\$) & Cheap Model Tokens & Thinking Model Tokens \\
\midrule
\multirow{2}{*}{MATH-500} & Only Thinking Model & 266.4k & 1.71 & -- & 1358.9k \\
 & PAC Reasoning & 82.2k & 0.78 & 724.4k & 125.5k \\
\midrule
\multirow{2}{*}{ZebraLogic} & Only Thinking Model & 59.8k & 4.64 & -- & 3670.1k \\
& PAC Reasoning & 43.8k & 1.98 & 1713.3k & 981.6k \\
\bottomrule
\end{tabular}}}
\end{table}

\subsection{Expected calibration error of two uncertainty scores}\label{sec:ece} 

We evaluate the calibration quality of uncertainty estimates on MATH-500 and ZebraLogic using Qwen3-4B-Instruct-2507.
We consider two uncertainty scores: a logits-based score derived from the model's predictive distribution and a verbalized score elicited from the model's self-reported confidence.
Expected calibration error (ECE)~\citep{guo2017calibration} quantifies the discrepancy between predicted confidence and empirical accuracy via binning and a weighted average of absolute gaps, where smaller values indicate better calibration.
Across both benchmarks, the logits-based score exhibits consistently lower ECE and smoother reliability than the verbalized score, indicating tighter calibration of uncertainty estimates.
The verbalized score shows higher variance and mild overconfidence in high-confidence bins.
These findings support the use of the logits-based score within PAC Reasoning and motivate improved elicitation methods for verbalized confidence.
Figure~\ref{fig:expected_error_uncertainty} summarizes the reliability plots and aggregated ECE\@.
All experiments are repeated 100 times under the decoding configuration described in Section~\ref{sec:exp_setup}.

\begin{figure}[htbp]
    \centering

    \subcaptionbox{MATH-500 (Logits-based)}{
        \includegraphics[width=0.45\linewidth]{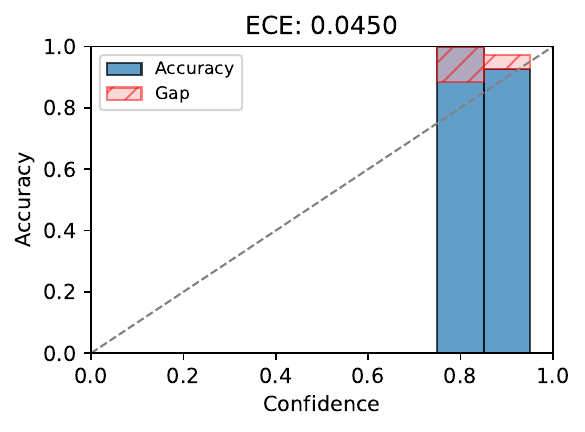}
    }
    \subcaptionbox{MATH-500 (Verbalized)}{
        \includegraphics[width=0.45\linewidth]{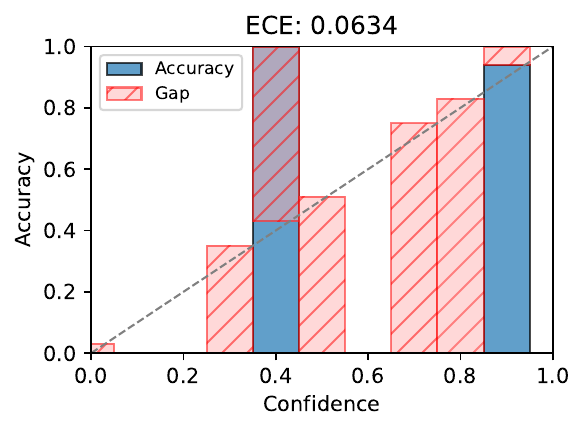}
    }

    \vspace{0.3cm}

    \subcaptionbox{ZebraLogic (Logits-based)}{
        \includegraphics[width=0.45\linewidth]{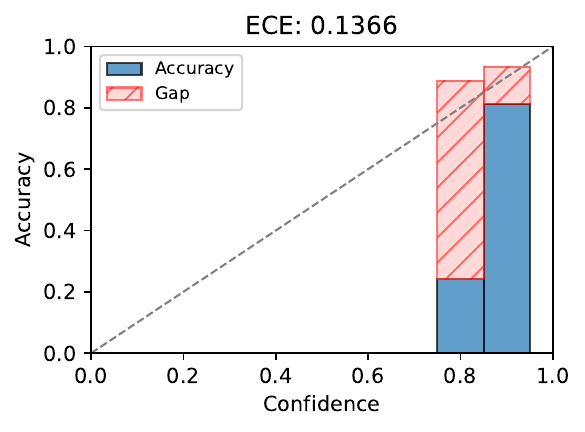}
    }
    \subcaptionbox{ZebraLogic (Verbalized)}{
        \includegraphics[width=0.45\linewidth]{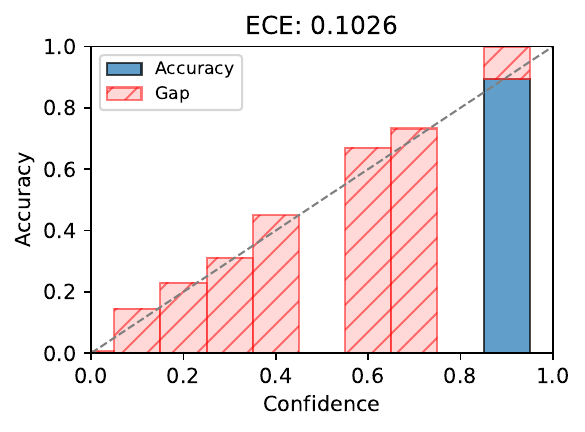}
    }
    \caption{Expected Calibration Error across two reasoning benchmarks.}\label{fig:expected_error_uncertainty}
\end{figure}

\subsection{Ablation study on the size of the calibration set}\label{sec:size_cal}
We conduct experiments to investigate the stability of efficiency gains under different calibration set sizes.
Specifically, we repeat the experiments with varying calibration ratios to examine how the size of the correction set influences performance.
The results are presented in Figure~\ref{fig:results_calibration_ratio}.
PAC Reasoning maintains stable error control and consistent uncertainty calibration across all benchmarks.
Both uncertainty scoring methods control the target performance loss, though the verbalized score exhibits larger variance.
Moreover, the logits-based score consistently outperforms the verbalized uncertainty score, achieving lower ECP and higher STP\@.
These findings demonstrate that PAC Reasoning can maintain valid performance loss control and stable efficiency gains under varying calibration dataset sizes.

\begin{figure}[htbp]
    \centering
    \subcaptionbox{MATH-500}{
        \includegraphics[width=0.33\linewidth]{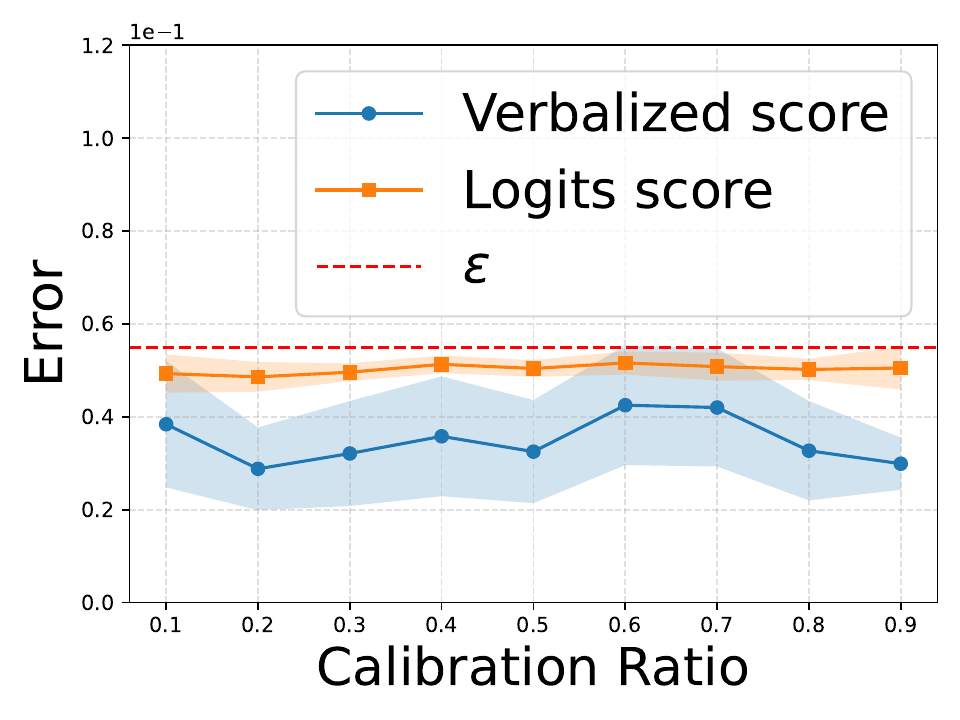}
         \includegraphics[width=0.33\linewidth]{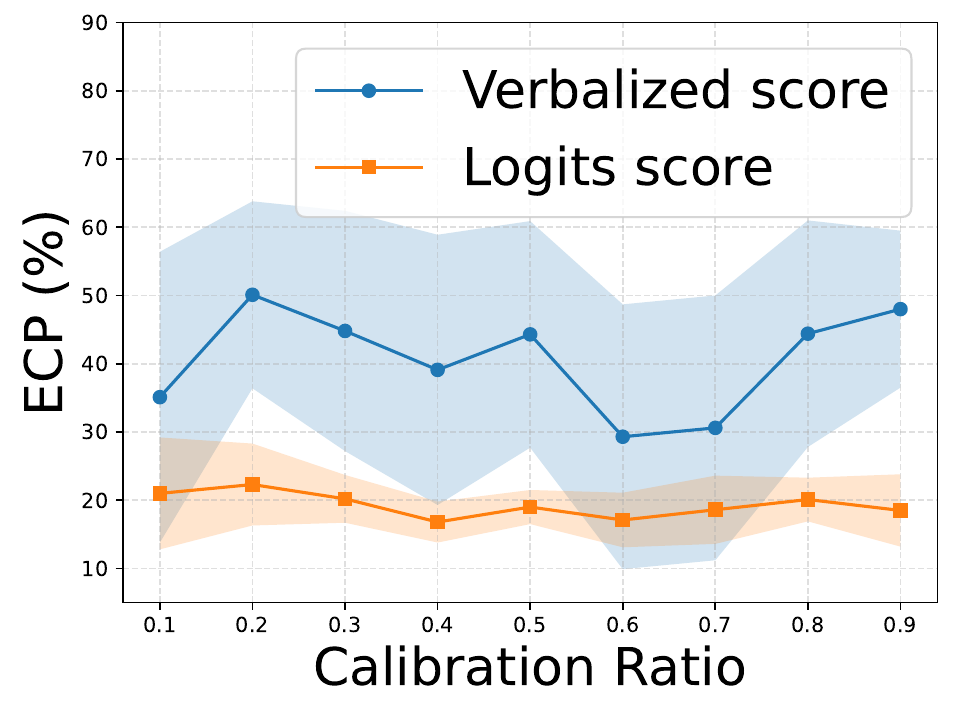}
        \includegraphics[width=0.33\linewidth]{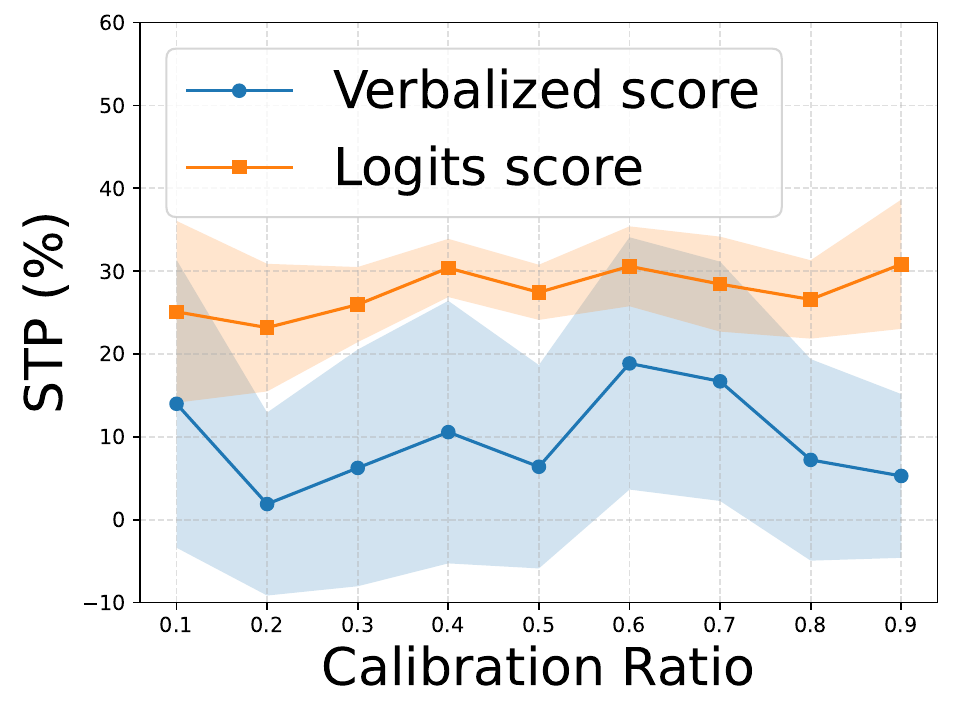}
    }
    \\
    \subcaptionbox{ZebraLogic}{
        \includegraphics[width=0.33\linewidth]{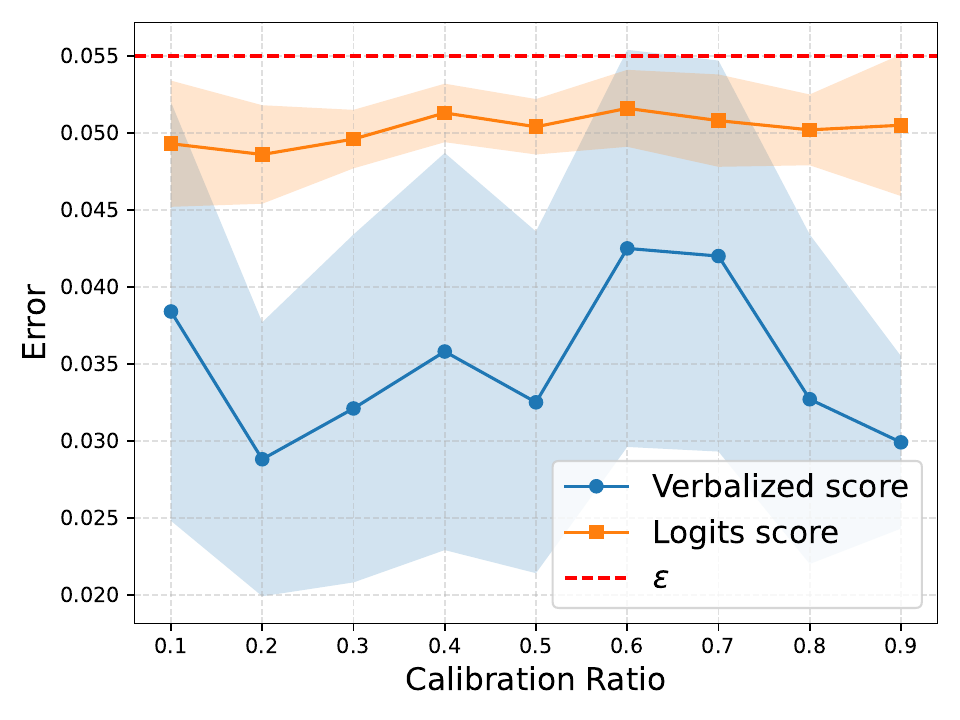}
        \includegraphics[width=0.33\linewidth]{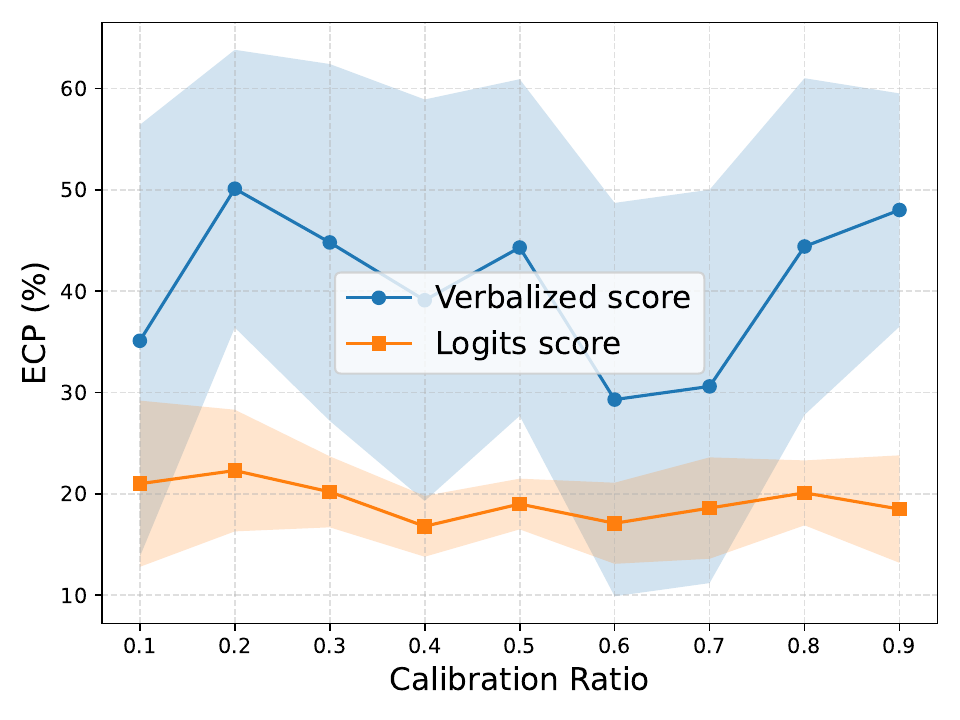}
        \includegraphics[width=0.33\linewidth]{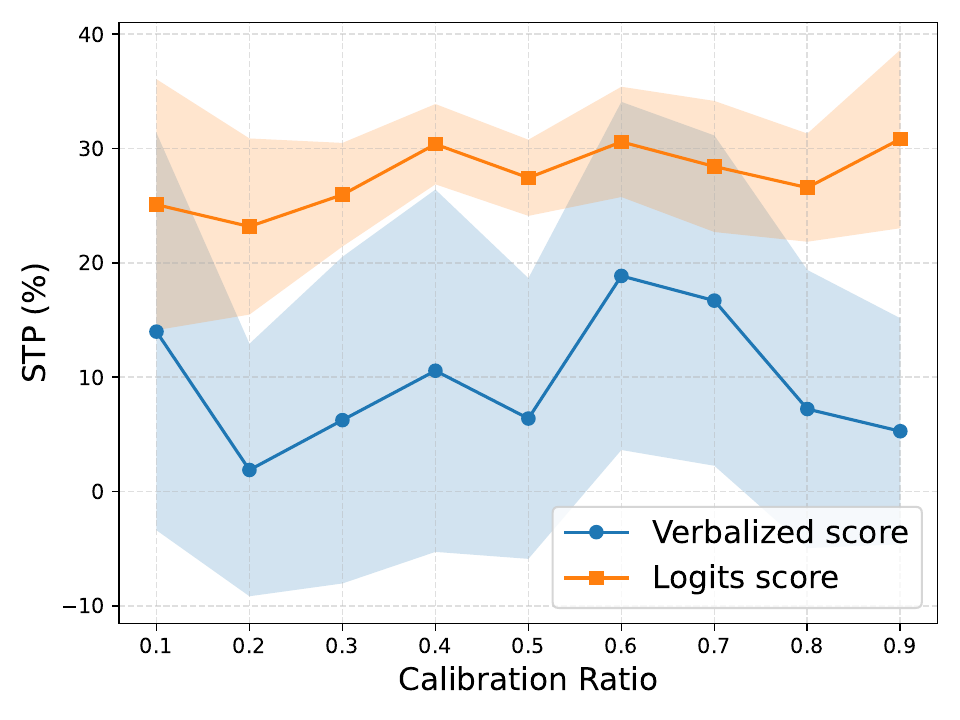}
    }
        \\
    \subcaptionbox{Arena-Hard}{
        \includegraphics[width=0.33\linewidth]{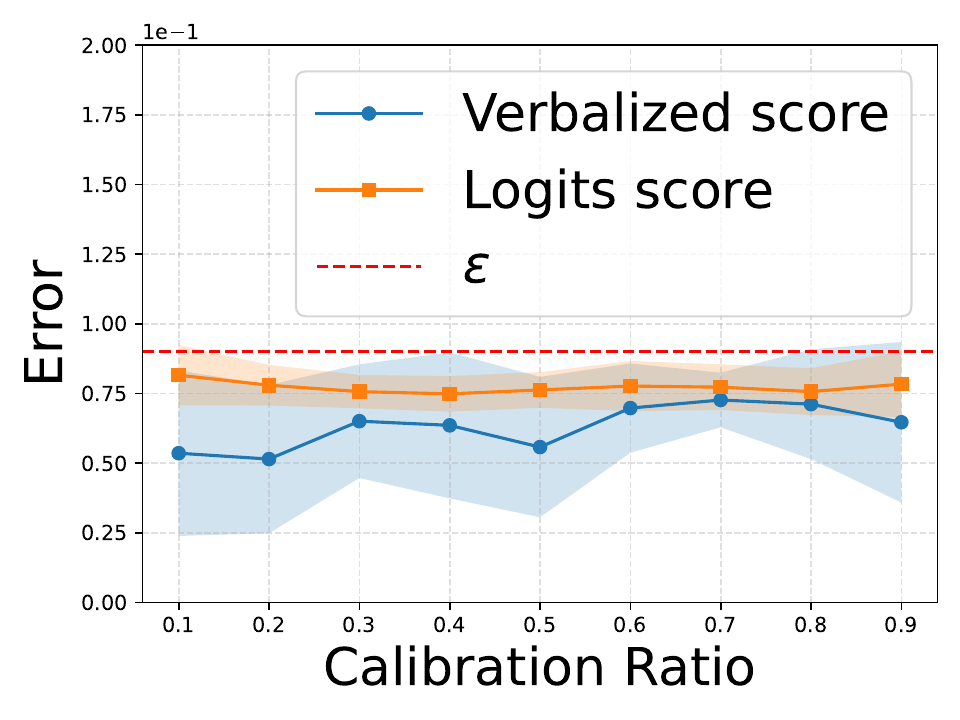}
        \includegraphics[width=0.33\linewidth]{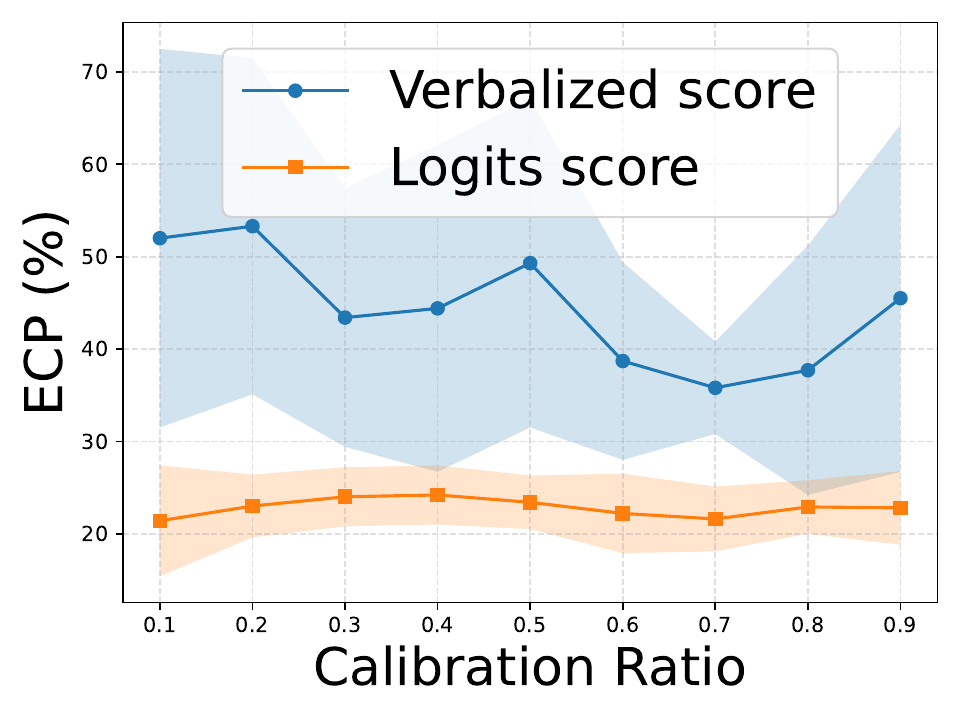}
        \includegraphics[width=0.33\linewidth]{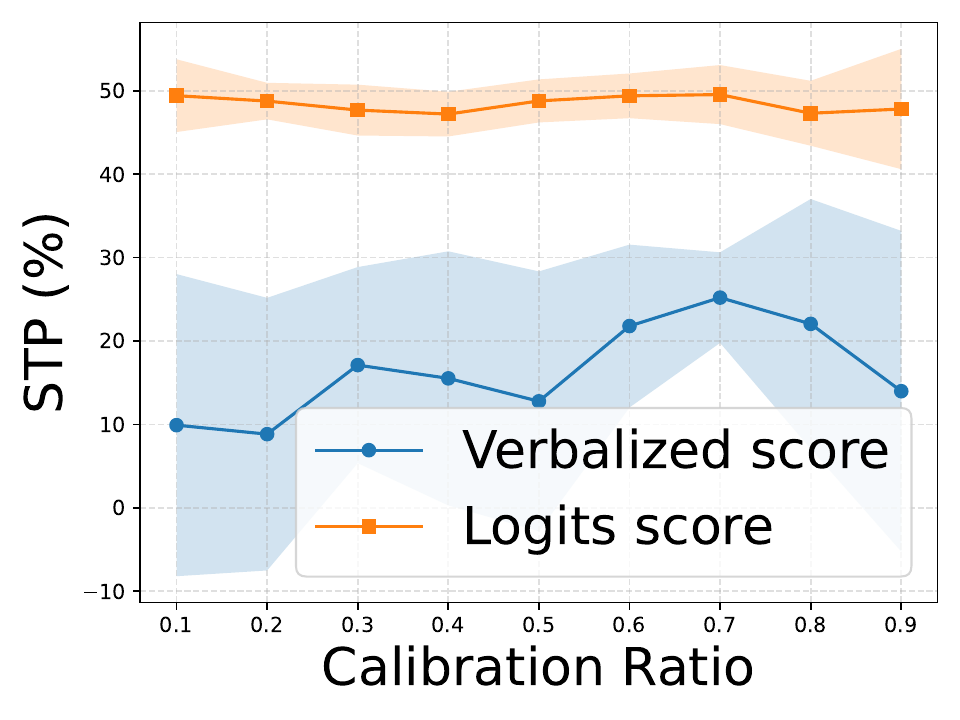}
    }
    \caption{Error control, ECP and STP of PAC Reasoning for semantic loss for different calibration ratios at a confidence level of $\alpha = 0.05$.
    Uncertainty score includes the logits-based score and the verbalized score.
    The red dashed line $\epsilon$ denotes the target tolerance, and the shaded areas represent standard deviations.
    All experiments are repeated 100 times.}
    \label{fig:results_calibration_ratio}
\end{figure}

\subsection{Reward score as an alternative uncertainty score}\label{sec:reward}
In this part, we evaluate whether PAC Reasoning remains valid when replacing the original uncertainty score with the reward score.
Concretely, we apply PAC Reasoning to MATH-500 and ZebraLogic using the reward score as the uncertainty estimate, and we report its error control based on the semantic cosine distance, ECP, and STP under varying $\epsilon$.
We follow the experimental setting described in Section~\ref{sec:exp_setup}, and the reward model is ``Qwen2.5-Math-PRM-7B''~\citep{zhang2025lessons}.

The results are presented in Figure~\ref{fig:semantic_results_reward}.
Across both benchmarks, the observed error curves remain below the diagonal baseline, indicating that PAC Reasoning still satisfies the theoretical error guarantee even with this alternative scoring method.
For efficiency, ECP consistently decreases as $\epsilon$ increases, showing that the method becomes more selective.
These results show that PAC Reasoning is robust to the choice of uncertainty score: using the reward score still ensures valid error control and provides reasonable efficiency gains.

\begin{figure}[htbp]
    \centering
    \subcaptionbox{MATH-500}{
        \includegraphics[width=0.33\linewidth]{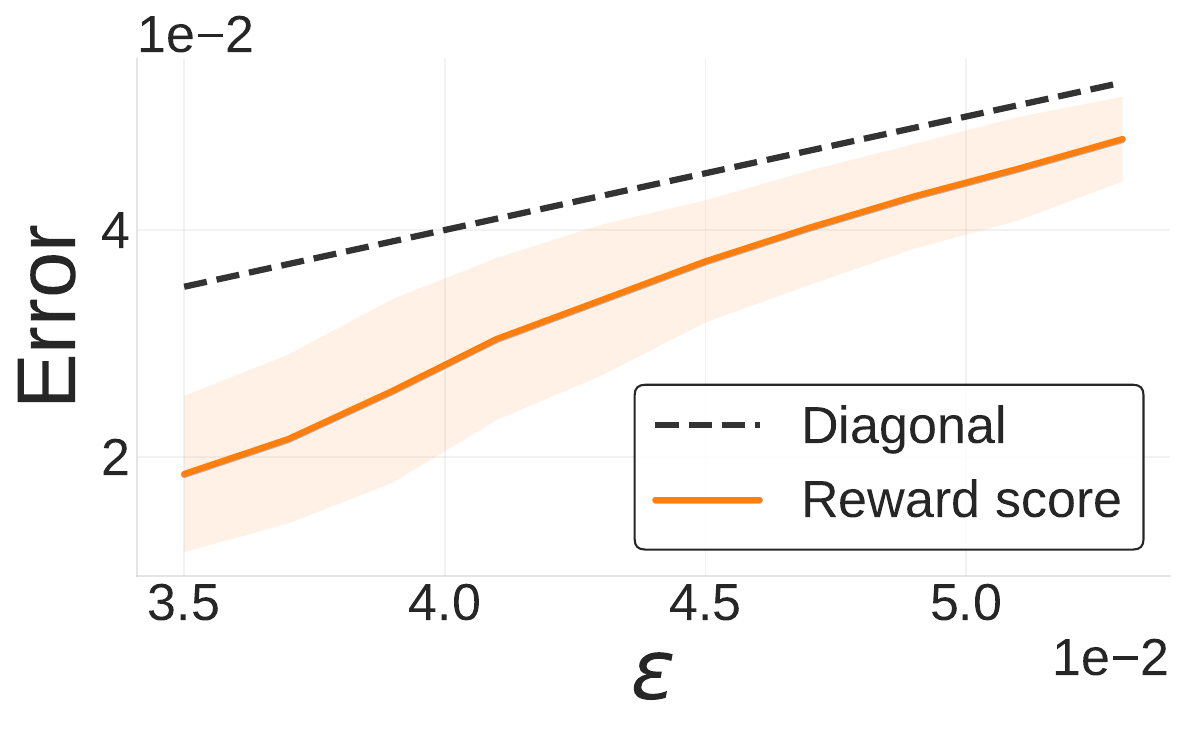}
         \includegraphics[width=0.33\linewidth]{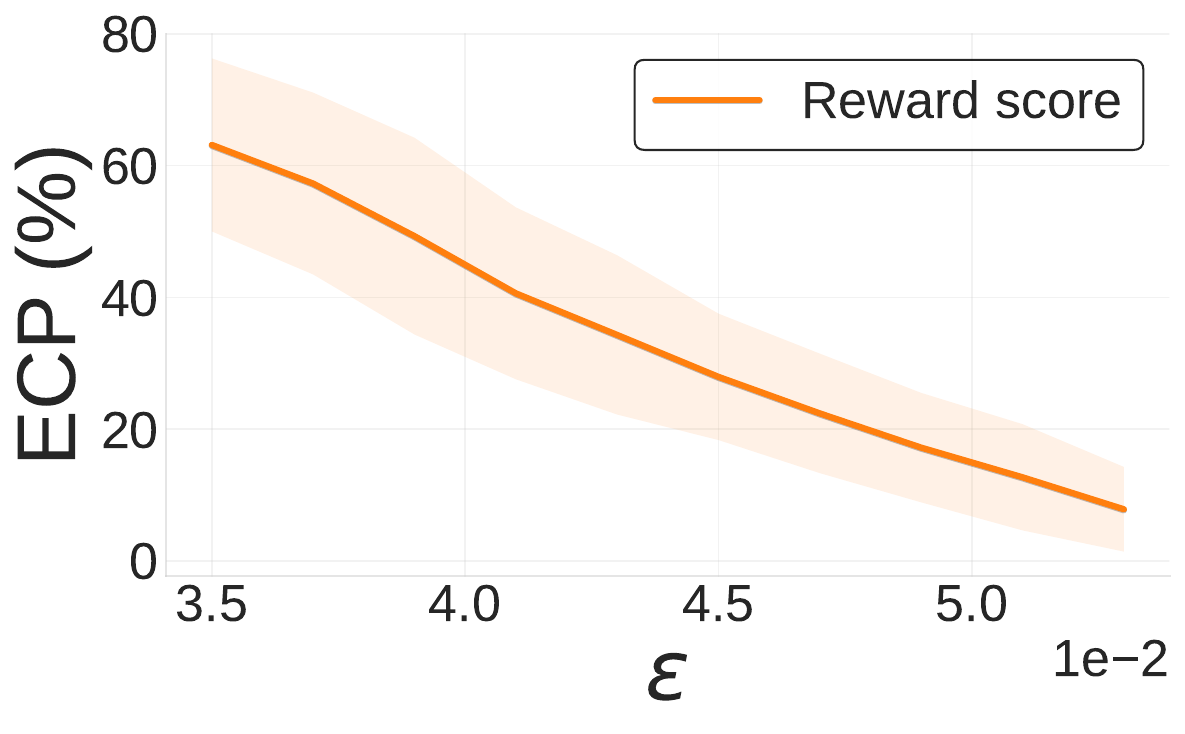}
        \includegraphics[width=0.33\linewidth]{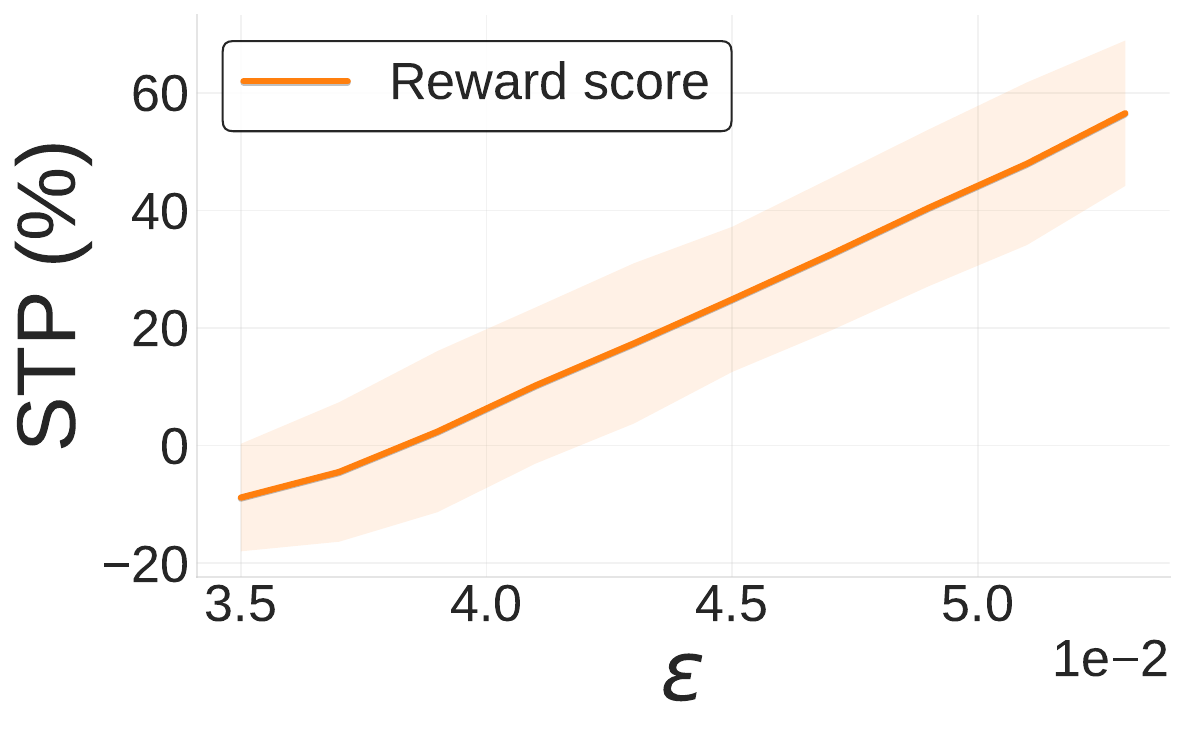}
    }
    \\
    \subcaptionbox{ZebraLogic}{
        \includegraphics[width=0.33\linewidth]{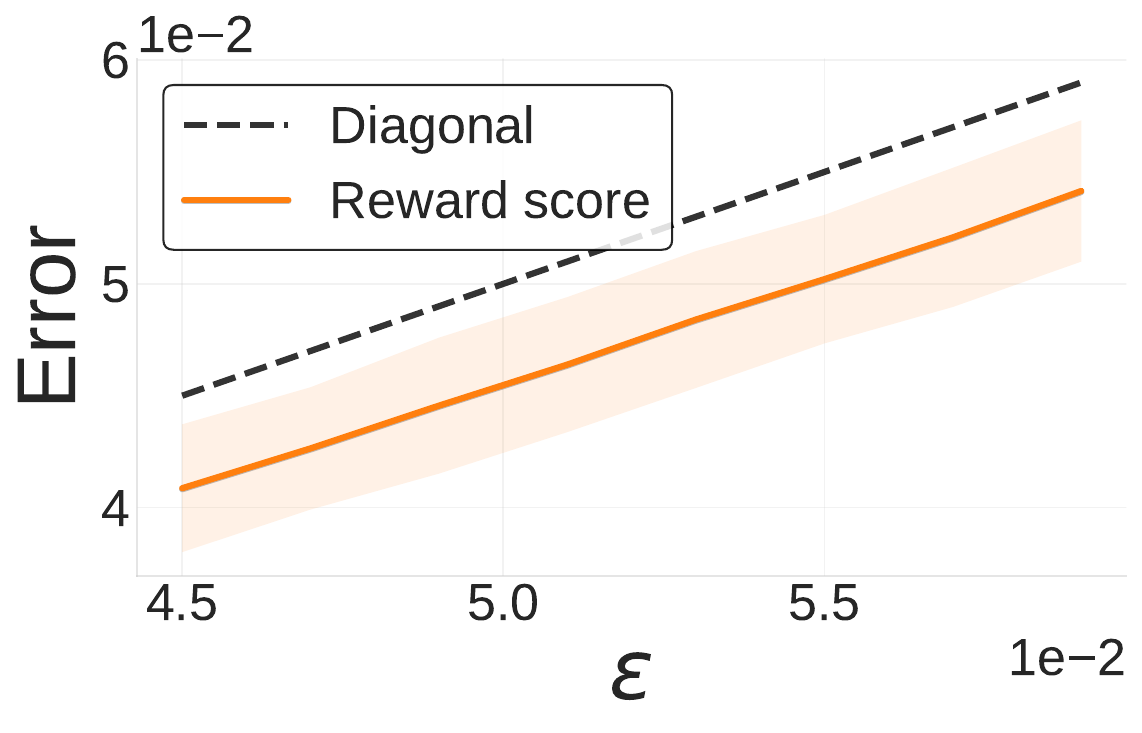}
        \includegraphics[width=0.33\linewidth]{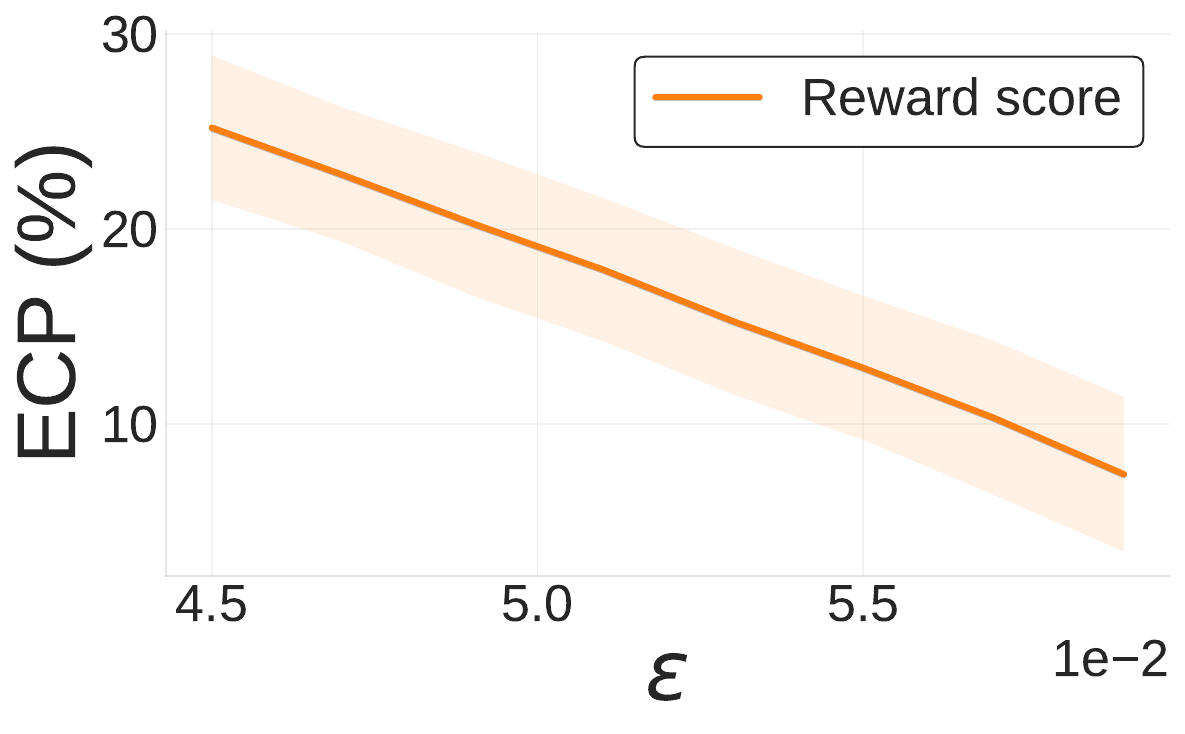}
        \includegraphics[width=0.33\linewidth]{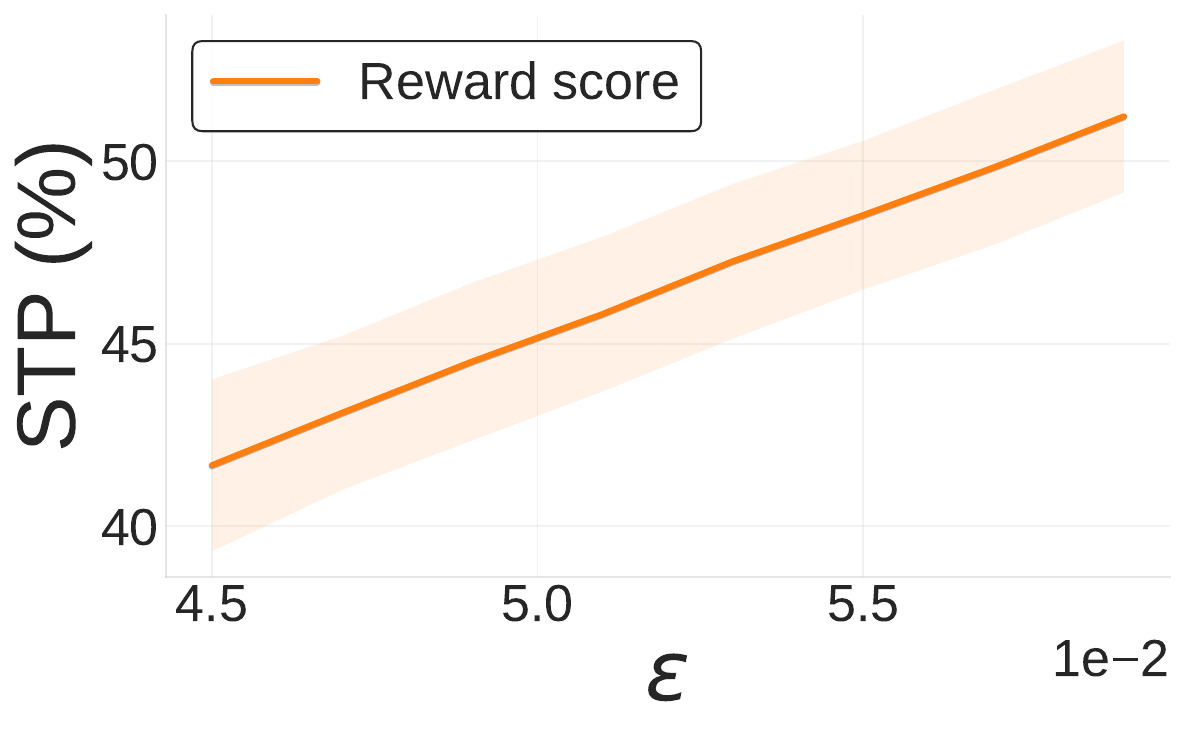}
    }
    \caption{Error control, ECP and STP of PAC Reasoning for semantic loss across two reasoning benchmarks at a confidence level of $\alpha = 0.05$.
    Uncertainty score is the \textbf{reward score}.
    All experiments are repeated 100 times, and the shaded areas represent standard deviations.}
    \label{fig:semantic_results_reward}
\end{figure}

\subsection{Empirical behavior under distribution shift}
\label{sec:distribution_shift}
While the inductive population-risk guarantee assumes calibration and deployment draws from the same target law, we provide empirical observations on how PAC Reasoning behaves under a moderate distribution shift.
Specifically, we calibrate on 300 examples from MATH-500 and evaluate on mixed test sets of 200 examples that include 20, 40, or 60 AIME problems.
Each setting is repeated 100 times under the binary loss.
As shown in Table~\ref{tab:dist_shift}, the \textbf{empirical error of PAC Reasoning remains below the target tolerance across all settings}.

Moreover, a key observation is a conservative escalation effect: AIME problems tend to induce higher uncertainty, causing them to exceed the threshold calibrated on MATH-500 more frequently.
Consequently, the escalation rate (ECP) increases while the token savings (STP) decrease as the proportion of AIME problems grows.
Despite this shift, the method still maintains non-trivial efficiency gains, with time savings ranging from 61.73\% to 69.02\% and monetary savings from 21.17\% to 37.24\%.

\begin{table}[h]
\centering
\caption{Empirical performance under distribution shift with varying proportions of AIME problems. Calibration is performed on MATH-500 ($\epsilon=0.03$).}
\label{tab:dist_shift}
\setlength{\tabcolsep}{2mm}{
\resizebox{\textwidth}{!}{
\begin{tabular}{l r r r r r r}
\toprule
Uncertainty Score & AIME Size & Error & STP (\%) & ECP (\%) & Time Saving (\%) & Money Saving (\%) \\
\midrule
\multirow{3}{*}{Logits-based score}
& 20 & 0.0193 & 19.02 & 26.92 & 69.02 & 36.50 \\
& 40 & 0.0220 & 6.56  & 32.55 & 67.38 & 26.51 \\
& 60 & 0.0211 & 0.51  & 36.97 & 66.00 & 21.17 \\
\midrule
\multirow{3}{*}{Verbalized score}
& 20 & 0.0208 & 19.76 & 30.18 & 64.62 & 37.24 \\
& 40 & 0.0278 & 7.27  & 35.49 & 63.04 & 27.21 \\
& 60 & 0.0289 & 1.19  & 39.58 & 61.73 & 21.86 \\
\bottomrule
\end{tabular}}}
\end{table}

\begingroup

\subsection{Accuracy--cost operating points}\label{app:re_accuracy_cost}
We vary $\epsilon$ to obtain operating points at expert-call rates of about $10\%$, $15\%$, $40\%$, and $65\%$ on MATH-500 and ZebraLogic.
Tables~\ref{tab:re_tradeoff_math} and~\ref{tab:re_tradeoff_zebra} compare PAC Reasoning with Naive control, using $\alpha=0.05$ and 100 random calibration/test splits.
ECP measures the fraction of queries sent to the thinking model.
The methods use different tolerances to reach similar call rates.
Their accuracies are close at these operating points, while the mean loss stays below the stated tolerance for PAC Reasoning and exceeds it for Naive control in each reported row.
Table~\ref{tab:re_fixed_accuracy} gives the accuracy of CoD and NoThinking, which have no escalation stage.

\begin{table}[htbp]
\centering
\small
\caption{{Accuracy--cost operating points on MATH-500 with 300 calibration and 200 test examples.
Entries are mean $\pm$ standard deviation over 100 splits; ECP and accuracy are percentages.}}
\label{tab:re_tradeoff_math}
\setlength{\tabcolsep}{3pt}
\begin{tabular}{rlrrrr}
\toprule
Target ECP & Method & $\epsilon$ & Loss & ECP (\%) & Accuracy (\%) \\
\midrule
$\approx10\%$ & PAC Reasoning & 0.050 & $0.031\pm0.015$ & $8.8\pm8.9$ & $93.7\pm1.7$ \\
& Naive control & 0.030 & $0.031\pm0.014$ & $7.8\pm2.8$ & $93.6\pm1.7$ \\
$\approx15\%$ & PAC Reasoning & 0.040 & $0.026\pm0.013$ & $13.4\pm12.1$ & $94.1\pm1.6$ \\
& Naive control & 0.020 & $0.022\pm0.012$ & $13.6\pm6.8$ & $94.4\pm1.5$ \\
$\approx40\%$ & PAC Reasoning & 0.020 & $0.014\pm0.010$ & $38.5\pm21.9$ & $95.1\pm1.4$ \\
& Naive control & 0.010 & $0.013\pm0.009$ & $40.9\pm17.0$ & $95.1\pm1.3$ \\
$\approx65\%$ & PAC Reasoning & 0.010 & $0.009\pm0.008$ & $63.3\pm22.1$ & $95.5\pm1.2$ \\
& Naive control & 0.005 & $0.007\pm0.007$ & $65.5\pm9.8$ & $95.6\pm1.2$ \\
\bottomrule
\end{tabular}
\end{table}

\begin{table}[htbp]
\centering
\small
\caption{{Accuracy--cost operating points on ZebraLogic with 500 calibration and 500 test examples.
Entries are mean $\pm$ standard deviation over 100 splits; ECP and accuracy are percentages.}}
\label{tab:re_tradeoff_zebra}
\setlength{\tabcolsep}{3pt}
\begin{tabular}{rlrrrr}
\toprule
Target ECP & Method & $\epsilon$ & Loss & ECP (\%) & Accuracy (\%) \\
\midrule
$\approx10\%$ & PAC Reasoning & 0.105 & $0.083\pm0.021$ & $10.3\pm5.8$ & $83.2\pm2.1$ \\
& Naive control & 0.082 & $0.084\pm0.016$ & $9.7\pm3.1$ & $83.1\pm1.6$ \\
$\approx15\%$ & PAC Reasoning & 0.090 & $0.071\pm0.020$ & $15.0\pm6.4$ & $84.2\pm2.0$ \\
& Naive control & 0.068 & $0.070\pm0.015$ & $14.8\pm3.3$ & $84.2\pm1.5$ \\
$\approx40\%$ & PAC Reasoning & 0.043 & $0.032\pm0.012$ & $39.5\pm7.8$ & $87.1\pm1.3$ \\
& Naive control & 0.030 & $0.031\pm0.010$ & $40.6\pm4.8$ & $87.2\pm1.2$ \\
$\approx65\%$ & PAC Reasoning & 0.009 & $0.006\pm0.006$ & $66.2\pm10.4$ & $88.7\pm1.0$ \\
& Naive control & 0.004 & $0.006\pm0.005$ & $64.1\pm4.1$ & $88.7\pm1.0$ \\
\bottomrule
\end{tabular}
\end{table}

\begin{table}[htbp]
\centering
\small
\caption{{Accuracy of the fixed baselines without escalation, in percent.
Entries are mean $\pm$ standard deviation.}}
\label{tab:re_fixed_accuracy}
\begin{tabular}{lrr}
\toprule
Dataset & CoD & NoThinking \\
\midrule
MATH-500 & $54.93\pm2.70$ & $43.95\pm2.38$ \\
ZebraLogic & $88.06\pm1.06$ & $88.12\pm1.11$ \\
\bottomrule
\end{tabular}
\end{table}

\paragraph{Naive-control tolerance sweep}
Figure~\ref{fig:re_naive_control} compares PAC Reasoning and naive control on MATH-500 using logits scores and the semantic loss of Figure~\ref{fig:semantic_results}.
The tolerance ranges from $0.035$ to $0.053$, with 100 random calibration/test splits at each value.
Both mean-loss curves lie below the tolerance line in this experiment.
PAC Reasoning leaves a larger margin below the tolerance, while naive control tracks the line more closely.

\begin{figure}[htbp]
\centering
\includegraphics[width=0.78\linewidth]{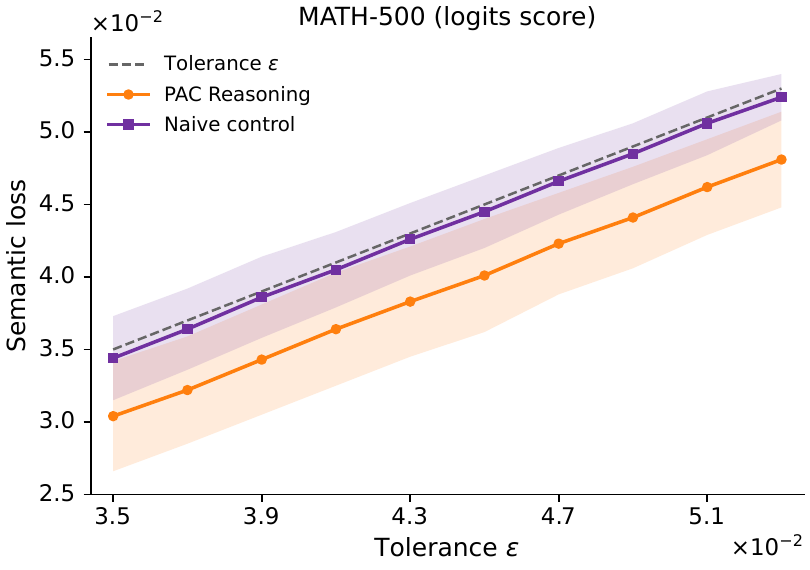}
\caption{{MATH-500 comparison with naive control using logits scores.
The curves show mean semantic loss over 100 random calibration/test splits, and the shaded areas show one standard deviation.
The dashed line marks the target tolerance.}}
\label{fig:re_naive_control}
\end{figure}

\subsection{Additional distribution-shift results}\label{app:re_shift}
\paragraph{Mathematics-to-code transfer}
We calibrate on 300 MATH-500 examples and evaluate on HumanEval with $\epsilon=0.1$ over 100 repetitions.
Table~\ref{tab:re_humaneval_shift} reports the loss and routing costs for both uncertainty scores.
The logits-based score uses the expert more often and gives lower loss, while the verbalized score gives larger token savings in this setting.

\begin{table}[htbp]
\centering
\small
\caption{{MATH-500 calibration followed by HumanEval evaluation at $\epsilon=0.1$.
Loss is mean $\pm$ standard deviation over 100 repetitions; ECP and STP are percentages.}}
\label{tab:re_humaneval_shift}
\begin{tabular}{lrrr}
\toprule
Score & Loss & ECP (\%) & STP (\%) \\
\midrule
Logits & $0.0305\pm0.0021$ & 14.63 & 23.33 \\
Verbalized & $0.0732\pm0.0018$ & 0.61 & 45.82 \\
\bottomrule
\end{tabular}
\end{table}

\paragraph{AIME-only evaluation}
Table~\ref{tab:re_aime_only} compares MATH-500 and AIME-only operating points after MATH-500 calibration.
The AIME-only evaluation uses $\epsilon=0.02$.
On AIME alone, ECP rises to $98.33\%$ for logits scores and $78.33\%$ for verbalized scores, while STP becomes $-33.20\%$ and $-14.29\%$, respectively.
The fast-model output is generated before escalation, so frequent expert calls can make the combined token cost exceed that of the thinking model alone.

\begin{table}[htbp]
\centering
\small
\caption{{MATH-500 reference results and AIME-only evaluation with MATH-500 calibration at $\epsilon=0.02$ for the AIME-only setting.
Values are point summaries; ECP and STP are percentages.}}
\label{tab:re_aime_only}
\begin{tabular}{llrrr}
\toprule
Score & Test set & Loss & ECP (\%) & STP (\%) \\
\midrule
Logits & MATH-500 & 0.0209 & 20.32 & 39.10 \\
& AIME & 0.0000 & 98.33 & $-33.20$ \\
Verbalized & MATH-500 & 0.0198 & 25.32 & 34.44 \\
& AIME & 0.0100 & 78.33 & $-14.29$ \\
\bottomrule
\end{tabular}
\end{table}

\subsection{Additional results for the Llama pair}\label{app:re_llama}
Table~\ref{tab:re_llama_pac} reports PAC Reasoning and the baselines with DeepSeek-R1-Distill-Llama-8B and Llama-3.1-8B-Instruct at $\epsilon=0.1$ on both datasets.
It complements the semantic-loss sweeps in Appendix~\ref{app:exp_other_llm} and the success counts in Table~\ref{tab:test_loss_success}.
For PAC Reasoning, the mean loss is below $0.1$ in all four settings, with positive mean token savings.

\begin{table}[htbp]
\centering
\small
\caption{{Results for the Llama-based model pair at $\epsilon=0.1$.
Loss is mean $\pm$ standard deviation; ECP and STP are mean percentages over 100 runs.
SR is the reported percentage of runs with empirical test loss at most $\epsilon$; a dash marks an entry not reported in the source table.}}
\label{tab:re_llama_pac}
\setlength{\tabcolsep}{4pt}
\begin{tabular}{llrrrr}
\toprule
Method & Score & Loss & ECP (\%) & STP (\%) & SR (\%) \\
\midrule
\multicolumn{6}{l}{\textbf{MATH-500}} \\
Naive threshold & Logits & $0.0000\pm0.0000$ & 100.00 & $-10.06$ & 100 \\
& Verbalized & $0.2575\pm0.1129$ & 4.75 & 81.34 & 3 \\
Naive control & Logits & $0.1713\pm0.1026$ & 14.62 & 72.83 & 14 \\
& Verbalized & $0.2550\pm0.1158$ & 7.00 & 79.62 & 4 \\
PAC Reasoning & Logits & $0.0898\pm0.1137$ & 68.88 & 11.75 & 92 \\
& Verbalized & $0.0788\pm0.1246$ & 54.00 & 36.47 & 87 \\
Router & -- & $0.0000\pm0.0000$ & 74.50 & 14.68 & 56 \\
CoD & -- & $0.1349\pm0.0125$ & -- & 23.44 & 100 \\
NoThinking & -- & $0.0938\pm0.0817$ & -- & 0.76 & 37 \\
\midrule
\multicolumn{6}{l}{\textbf{ZebraLogic}} \\
Naive threshold & Logits & $0.0745\pm0.0086$ & 52.32 & $-30.35$ & 100 \\
& Verbalized & $0.1040\pm0.0095$ & 70.32 & $-76.30$ & 40 \\
Naive control & Logits & $0.2405\pm0.0152$ & 5.23 & 1.77 & 0 \\
& Verbalized & $0.2262\pm0.0144$ & 23.29 & $-27.30$ & 0 \\
PAC Reasoning & Logits & $0.0826\pm0.0217$ & 49.93 & 1.26 & 91 \\
& Verbalized & $0.0825\pm0.0199$ & 76.78 & 3.66 & 88 \\
Router & -- & $0.0108\pm0.0031$ & 95.62 & 0.53 & 48 \\
CoD & -- & $0.1278\pm0.0110$ & -- & 1.90 & 1 \\
NoThinking & -- & $0.1124\pm0.0097$ & -- & $-1.82$ & 22 \\
\bottomrule
\end{tabular}
\end{table}

\subsection{Score resolution and threshold sensitivity}\label{app:re_score_sensitivity}
We use a synthetic experiment to separate score compression, quantization, and sparse-loss effects.
For each observation, draw $S\sim\operatorname{Unif}(0,1)$ and a binary loss $D\mid S=s\sim\operatorname{Bernoulli}(p_\lambda(s))$, where
\[
p_\lambda(s)=\lambda(0.01+0.20s^2),
\qquad
\epsilon_\lambda=0.04\lambda,
\qquad
\lambda\in\{1,0.25\}.
\]
At latent cutoff $t$, the population risk and fast-route rate are
\[
R_\lambda(t)=\lambda\left(0.01t+\frac{0.20}{3}t^3\right),
\qquad A(t)=t.
\]
The fast route incurs loss $D$, while the expert route incurs zero loss.
In the continuous-score conditions, we fix the latent candidate cutoffs $t_k=k/1000$ before calibration.
We use $n\in\{500,2000\}$, $\alpha=0.05$, and 1,000 paired repetitions per condition.
Each repetition uses the losses of all $n$ original observations and applies fixed-sequence selection with CLT, Hoeffding, or exact one-sided binomial UCBs.
We use the known population risk to estimate the violation probability $\mathbb P(R_\lambda(\hat t)>\epsilon_\lambda)$ across repetitions.

\paragraph{Continuous compression}
We define $U^{(c)}=0.05+c(S-1/2)$ for $c\in\{10^{-1},10^{-2},10^{-3},10^{-4}\}$ and transform the pre-specified candidate thresholds by the same map.
This keeps the candidate routing sets fixed while changing the numerical scale of the scores.
Across 12,000 paired comparisons, the selected latent rank is unchanged.

\paragraph{Quantization and finite precision}
We quantize the score into $K\in\{100,20,5\}$ levels and average over four bin phases.
Tied score blocks enter the fast-route set together.
Table~\ref{tab:re_score_levels} reports the fast-route rate and the two-sided one-level routing jump, measured between the candidates one level below and one level above the selected threshold.
Coarser levels produce larger routing jumps and lower fast-route rates in this experiment.
With fixed decimal precision of $10^{-3}$, compressing the score width from $10^{-1}$ to $10^{-2}$ lowers the mean fast-route rate from $0.6716$ to $0.6246$; at width $10^{-3}$, selection returns the expert-only action.

\begin{table}[htbp]
\centering
\small
\caption{{Score quantization for the CLT rule with $n=500$ and $\lambda=1$.
Values average four bin phases; rates and jumps are fractions.}}
\label{tab:re_score_levels}
\begin{tabular}{rrrr}
\toprule
Levels $K$ & Tie mass & Two-sided routing jump & Mean fast-route rate \\
\midrule
100 & 0.01 & 0.02 & 0.6716 \\
20 & 0.05 & 0.10 & 0.6516 \\
5 & 0.20 & 0.40 & 0.5760 \\
\bottomrule
\end{tabular}
\end{table}

\paragraph{Sparse losses and UCB choice}
For regular losses ($\lambda=1$), the CLT violation rate is $0.067$ at $n=500$ and $0.050$ at $n=2000$.
For sparse losses ($\lambda=0.25$) at $n=500$, it rises to $0.117$ with an exact $95\%$ interval of $[0.0977,0.1386]$, while the exact-binomial UCB gives $0.043$ with interval $[0.0313,0.0575]$.
The Hoeffding rule has no observed violations in these continuous-score conditions; at $n=500$, it selects the expert-only action in both loss regimes.
Table~\ref{tab:re_score_ucb} also reports the $n=2000$ results.
These full-observation experiments separate the effect of score resolution from the small-sample behavior of the UCB approximation.
Figure~\ref{fig:re_score_sensitivity} summarizes the routing and violation results.

\begin{table}[htbp]
\centering
\small
\caption{{Population-risk violation rates and two-sided exact $95\%$ binomial intervals for continuous scores.
All entries are percentages, with 1,000 repetitions per setting and $\alpha=0.05$.}}
\label{tab:re_score_ucb}
\setlength{\tabcolsep}{4pt}
\begin{tabular}{rrlll}
\toprule
$n$ & $\lambda$ & CLT & Exact binomial & Hoeffding \\
\midrule
500 & 1 & $6.70\ [5.23,8.43]$ & $4.10\ [2.96,5.52]$ & $0.00\ [0.00,0.37]$ \\
500 & 0.25 & $11.70\ [9.77,13.86]$ & $4.30\ [3.13,5.75]$ & $0.00\ [0.00,0.37]$ \\
2000 & 1 & $5.00\ [3.73,6.54]$ & $3.80\ [2.70,5.18]$ & $0.00\ [0.00,0.37]$ \\
2000 & 0.25 & $5.80\ [4.43,7.43]$ & $3.40\ [2.37,4.72]$ & $0.00\ [0.00,0.37]$ \\
\bottomrule
\end{tabular}
\end{table}

\begin{figure}[htbp]
\centering
\includegraphics[width=\linewidth]{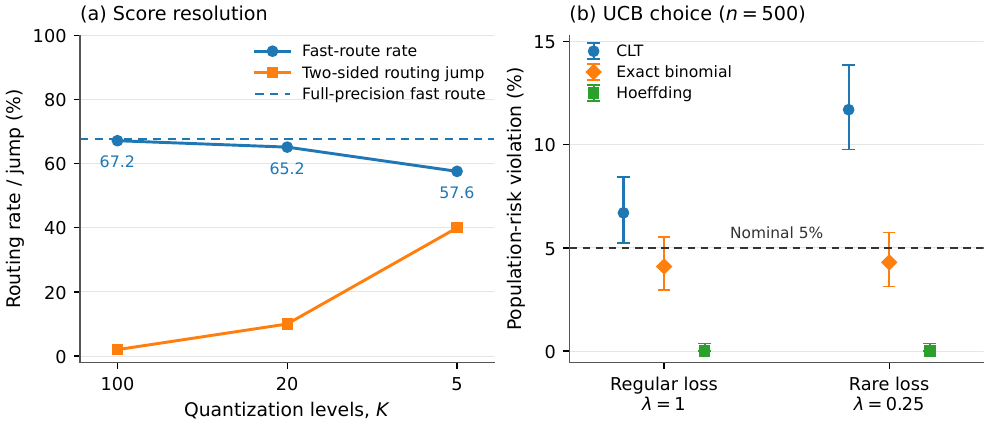}
\caption{{Score resolution and UCB choice in the full-observation synthetic experiment, with $\alpha=0.05$ and 1,000 repetitions.
(a) Mean fast-route rate and two-sided one-level routing jump for $n=500$, $\lambda=1$, and the CLT UCB, averaged over four bin phases.
The dashed line is the full-precision fast-route rate.
(b) Population-risk violation rates for continuous scores at $n=500$, with two-sided exact $95\%$ binomial intervals.
The dashed line marks the nominal $5\%$ level.}}
\label{fig:re_score_sensitivity}
\end{figure}

\subsection{Calibration choices and amortized cost}\label{app:re_calibration_cost}
\paragraph{Choosing the tolerance and model pair}
The tolerance $\epsilon$ expresses the application's acceptable loss relative to the chosen thinking model.
Its numerical scale depends on the loss, task, model pair, and their disagreement.
The values $0.03$ and $0.08$ in Table~\ref{table:binary_results_more_methods} are representative binary-loss operating points for MATH-500 and ZebraLogic.
The tolerance sweeps in Figures~\ref{fig:semantic_results}, \ref{fig:llama_results}, \ref{fig:results_other_datasets}, and~\ref{fig:semantic_results_reward} show how routing and savings change within each setting.
The framework uses two callable models, an expert-relative loss, and an uncertainty score, with no requirement that the models share a family.
The current experiments cover the Qwen pair and the Llama-based pair described in Section~\ref{sec:exp_setup}.
Useful pair choices balance the thinking model's quality, the fast model's cost, and the observed disagreement between their outputs.
If all queries are escalated after fast-model generation, ECP is $100\%$ and the extra fast-model tokens can make STP negative.

\paragraph{Query events and cached expert calls}
Importance weighting supports partial or unequal-probability expert queries while preserving the conditional pool mean.
For uniform sampling of indices with replacement and a fresh query decision of probability $\pi$ on each draw, the expected number of query events is $m\pi$.
With one cached expert answer per calibration item, the expected number of distinct expert calls is
\[
\mathbb E[N_{\mathrm{unique}}]
=n\left[1-\left(1-\frac{\pi}{n}\right)^m\right].
\]
Indeed, a fixed item is queried on one draw with probability $\pi/n$, so the chance that it is never queried in $m$ draws is $(1-\pi/n)^m$.
At the experimental choice $m=n/\pi$, the expected number of query events is $n$, while the expected number of distinct calls is approximately $(1-e^{-1})n\approx0.632n$ for large $n$.
The number of distinct cached expert calls is at most $n$.
The choice $\pi=0.5$ is an experimental setting; changing $\pi$ changes the sampling variance, weighted range, and query cost.
With $\pi=1$, every sampled index is queried, whereas full-observation calibration evaluates each of the $n$ original observations once.

\paragraph{Deployment amortization}
Table~\ref{tab:re_mean_token_cost} reports mean token use, including calibration for PAC Reasoning.
On MATH-500, the reported means are $6.06\mathrm{k}$ for PAC Reasoning, $6.63\mathrm{k}$ for CoD, and $9.63\mathrm{k}$ for NoThinking.
On ZebraLogic, the corresponding means are $7.53\mathrm{k}$, $7.34\mathrm{k}$, and $7.38\mathrm{k}$.

\begin{table}[htbp]
\centering
\small
\caption{{Reported mean token use, in thousands of tokens.
PAC Reasoning includes calibration-stage token use.}}
\label{tab:re_mean_token_cost}
\begin{tabular}{lrr}
\toprule
Method & MATH-500 & ZebraLogic \\
\midrule
NoThinking & 9.63 & 7.38 \\
CoD & 6.63 & 7.34 \\
PAC Reasoning & 6.06 & 7.53 \\
\bottomrule
\end{tabular}
\end{table}

Let $C_{\mathrm{cal}}$ be the one-time calibration cost, and let $C_{\mathrm{deploy}}$ and $C_{\mathrm{think}}$ be the mean cost per deployment query for PAC Reasoning and the thinking model alone, measured in the same unit.
For $N$ deployment queries,
\[
C_{\mathrm{total}}(N)=C_{\mathrm{cal}}+NC_{\mathrm{deploy}},
\qquad
C_{\mathrm{avg}}(N)=\frac{C_{\mathrm{cal}}}{N}+C_{\mathrm{deploy}}.
\]
When $C_{\mathrm{deploy}}<C_{\mathrm{think}}$, the costs match at
\[
N^*=\frac{C_{\mathrm{cal}}}{C_{\mathrm{think}}-C_{\mathrm{deploy}}},
\]
and PAC Reasoning has lower total cost for $N>N^*$.
If $C_{\mathrm{deploy}}\ge C_{\mathrm{think}}$, amortization does not produce a finite break-even point.
The calibration cost includes fast-model generation, uncertainty scoring, distinct expert queries, and loss evaluation.
The test-stage measurements in Table~\ref{tab:latency_cost} provide the deployment component of this accounting.

\endgroup

\clearpage
\section*{NeurIPS Paper Checklist}

\begin{enumerate}

\item {\bf Claims}
    \item[] Question: Do the main claims made in the abstract and introduction accurately reflect the paper's contributions and scope?
    \item[] Answer: \answerYes{}
    \item[] Justification: The abstract and introduction state the main contribution as PAC-style control of performance loss for efficient reasoning, and this claim is matched by the theoretical results in Sections 3--4 and the experiments in Sections 5--6.

\item {\bf Limitations}
    \item[] Question: Does the paper discuss the limitations of the work performed by the authors?
    \item[] Answer: \answerYes{}
    \item[] Justification: {Section~\ref{sec:limitations} discusses the sampling conditions, calibration cost, distribution shift, and the expert-relative loss target.}

\item {\bf Theory assumptions and proofs}
    \item[] Question: For each theoretical result, does the paper provide the full set of assumptions and a complete (and correct) proof?
    \item[] Answer: \answerYes{}
    \item[] Justification: Section~\ref{sec:theoretical_analysis} states the assumptions for the fixed-sequence selection theorem, including pointwise population-risk validity, and gives proof sketches.
    Appendix~\ref{sec:proof_pac_guarantee} gives the selection proof and its asymptotic version under the stated pointwise validity condition.
    The CLT result in Appendix~\ref{sec:clt_ucb_validity} is conditional on a fixed calibration pool; it does not by itself establish population-risk coverage.

    \item {\bf Experimental result reproducibility}
    \item[] Question: Does the paper fully disclose all the information needed to reproduce the main experimental results of the paper to the extent that it affects the main claims and/or conclusions of the paper (regardless of whether the code and data are provided or not)?
    \item[] Answer: \answerYes{}
    \item[] Justification: Section~\ref{sec:exp_setup} and Appendix~\ref{app:experimental_details} describe the model pairs, uncertainty scores, data splits, loss functions, metrics, prompts, and decoding settings.
    The public code repository is \url{https://github.com/Jianguo99/PAC-Reasoning}.
    Its README provides installation instructions, data-access requirements, and commands for generation, scoring, calibration, and evaluation.

\item {\bf Open access to data and code}
    \item[] Question: Does the paper provide open access to the data and code, with sufficient instructions to faithfully reproduce the main experimental results, as described in supplemental material?
    \item[] Answer: \answerYes{}
    \item[] Justification: The code is publicly available at \url{https://github.com/Jianguo99/PAC-Reasoning} under the Apache-2.0 license.
    The repository README describes the software environment, data preparation, and scripts for the main experiments.
    It also identifies the access requirements for gated datasets, including ZebraLogic and GPQA.

\item {\bf Experimental setting/details}
    \item[] Question: Does the paper specify all the training and test details (e.g., data splits, hyperparameters, how they were chosen, type of optimizer) necessary to understand the results?
    \item[] Answer: \answerYes{}
    \item[] Justification: Section~\ref{sec:exp_setup} and Appendix~\ref{app:experimental_details} describe the calibration/test splits, model choices, decoding settings, prompts, repeated-run protocol, and baselines.
    The public repository linked above gives commands for running generation, scoring, calibration, and evaluation with these settings.

\item {\bf Experiment statistical significance}
    \item[] Question: Does the paper report error bars suitably and correctly defined or other appropriate information about the statistical significance of the experiments?
    \item[] Answer: \answerYes{}
    \item[] Justification: The figures and tables report means with standard deviations over 100 repeated experiments, and the captions explicitly state that shaded areas represent standard deviations.

\item {\bf Experiments compute resources}
    \item[] Question: For each experiment, does the paper provide sufficient information on the computer resources (type of compute workers, memory, time of execution) needed to reproduce the experiments?
    \item[] Answer: \answerYes{}
    \item[] Justification: The appendix specifies the compute worker as one NVIDIA RTX A6000 graphics card, and the paper further reports wall-clock latency and dollar cost for representative experiments.
    
\item {\bf Code of ethics}
    \item[] Question: Does the research conducted in the paper conform, in every respect, with the NeurIPS Code of Ethics \url{https://neurips.cc/public/EthicsGuidelines}?
    \item[] Answer: \answerYes{}
    \item[] Justification: To the best of our knowledge, the research follows the NeurIPS Code of Ethics and uses existing benchmark datasets and models.

\item {\bf Broader impacts}
    \item[] Question: Does the paper discuss both potential positive societal impacts and negative societal impacts of the work performed?
    \item[] Answer: \answerNo{}
    \item[] Justification: The current draft does not include a discussion of broader societal impacts, either positive or negative.
    
\item {\bf Safeguards}
    \item[] Question: Does the paper describe safeguards that have been put in place for responsible release of data or models that have a high risk for misuse (e.g., pre-trained language models, image generators, or scraped datasets)?
    \item[] Answer: \answerNA{}
    \item[] Justification: The paper studies a routing method on existing models and benchmarks, and it does not release a new high-risk model or scraped dataset that would require special release safeguards.

\item {\bf Licenses for existing assets}
    \item[] Question: Are the creators or original owners of assets (e.g., code, data, models), used in the paper, properly credited and are the license and terms of use explicitly mentioned and properly respected?
    \item[] Answer: \answerNo{}
    \item[] Justification: The public code repository uses the Apache-2.0 license and credits the evaluation tools and prompts on which it builds.
    The paper cites the benchmark datasets and model families, but it does not provide a full list of licenses and terms of use for the third-party datasets, models, and software used in the experiments.

\item {\bf New assets}
    \item[] Question: Are new assets introduced in the paper well documented and is the documentation provided alongside the assets?
    \item[] Answer: \answerNA{}
    \item[] Justification: The paper does not present a new released dataset, model, or benchmark asset as a primary contribution in the current submission text.

\item {\bf Crowdsourcing and research with human subjects}
    \item[] Question: For crowdsourcing experiments and research with human subjects, does the paper include the full text of instructions given to participants and screenshots, if applicable, as well as details about compensation (if any)? 
    \item[] Answer: \answerNA{}
    \item[] Justification: The paper does not report crowdsourcing experiments or research with human subjects conducted by the authors.

\item {\bf Institutional review board (IRB) approvals or equivalent for research with human subjects}
    \item[] Question: Does the paper describe potential risks incurred by study participants, whether such risks were disclosed to the subjects, and whether Institutional Review Board (IRB) approvals (or an equivalent approval/review based on the requirements of your country or institution) were obtained?
    \item[] Answer: \answerNA{}
    \item[] Justification: The paper does not involve human subjects research by the authors, so IRB approval is not applicable.

\item {\bf Declaration of LLM usage}
    \item[] Question: Does the paper describe the usage of LLMs if it is an important, original, or non-standard component of the core methods in this research? Note that if the LLM is used only for writing, editing, or formatting purposes and does \emph{not} impact the core methodology, scientific rigor, or originality of the research, declaration is not required.
    \item[] Answer: \answerYes{}
    \item[] Justification: The method uses existing LLMs as the thinking and non-thinking models in the proposed routing framework. In addition, LLM-generated uncertainty estimates, such as verbalized confidence scores, are used as one possible scoring function for calibration and routing. The paper describes these components, including the model choices, prompts, and decoding settings, in the experimental setup and appendix.

\end{enumerate}

\end{document}